\documentclass[10pt,reqno]{amsart}

\usepackage{companypaper}
\microtypesetup{expansion=false}
\usepackage{natbib}
\usepackage{hyphenat}
\usepackage{enumitem}
\usepackage{verbatim}
\usepackage{textcomp}
\usepackage{xurl}
\makeatletter
\pretocmd{\@settitle}{\let\uppercasenonmath\@gobble}{}{}
\patchcmd{\@settitle}{\bfseries}{\bfseries\LARGE}{}{}
\pretocmd{\@setauthors}{\let\MakeUppercase\@firstofone}{}{}
\patchcmd{\@setauthors}{\centering\footnotesize}{\centering\normalsize}{}{}
\renewcommand{\@setaddresses}{}
\makeatother
\usepackage{pifont}
\usepackage{booktabs}
\usepackage{multirow}
\usepackage{threeparttable}
\usepackage{array}
\newcolumntype{L}[1]{>{\raggedright\arraybackslash}p{#1}}
\usepackage{nicefrac}

\DeclareMathOperator*{\argmax}{argmax}

\DeclareMathOperator{\softmax}{softmax}
\DeclareMathOperator{\simop}{sim}

\DeclareMathOperator{\TopK}{TopK}

\newcommand{\logos}{\textnormal{\textsc{Logos}}\xspace}
\newcommand{\maas}{\textnormal{\textsc{MaAS}}\xspace}
\newcommand{\aflow}{\textnormal{\textsc{AFlow}}\xspace}
\newcommand{\genome}{Agent Genome\xspace}
\newcommand{\pack}{Agent Pack\xspace}
\newcommand{\askf}{\mathrm{Ask\text{-}F_1}}
\DeclareMathOperator{\UCB}{UCB}
\DeclareMathOperator{\Beta}{Beta}

\newcommand{\reals}{\mathbb{R}}
\newcommand{\indic}[1]{\mathbf{1}\left\{#1\right\}}
\CompanyLogo{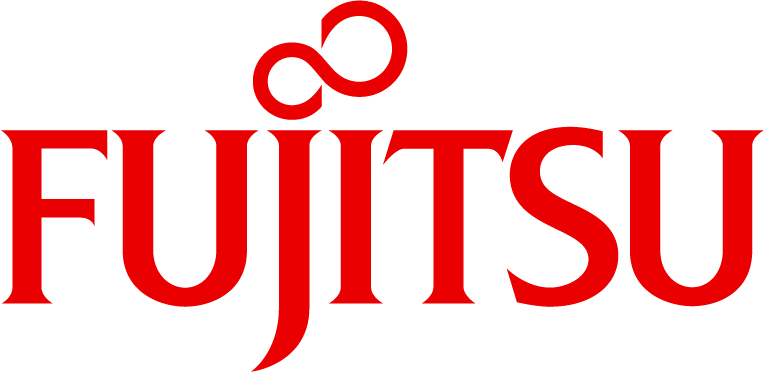}
\CompanyLogoHeight{11mm}
\CompanyHeaderText{Fujitsu Limited}
\CompanyDocType{}
\CompanyClassification{}
\CompanySubtitle{}
\CompanyRunningTitle{LOGOS: A Living Logic for AI Agent Teams That Evolve With Humans}

\title{LOGOS: A Living Logic for AI Agent Teams That \\ Evolve With Humans}

\author{Yuma Ichikawa\textsuperscript{1,2},
Yamato Arai\textsuperscript{1,3},
Kosaku Kimura\textsuperscript{1},
Akira Sakai\textsuperscript{1,4},
Hiromichi Kobashi\textsuperscript{1}}
\thanks{\textsuperscript{1}Fujitsu Limited,
\textsuperscript{2}RIKEN Center for AIP,
\textsuperscript{3}The University of Tokyo,
\textsuperscript{4}Tokai University\\
\textbf{Correspondence}: Yuma Ichikawa at
\textcolor{blue}{\texttt{ichikawa.yuma@fujitsu.com}}}
\date{}

\begin{document}
\raggedbottom

\maketitle
\enlargethispage{4pt}

\begingroup
\makeatletter\let\maketitle\@empty\makeatother
\begin{abstract}
    AI agents are evolving from answer engines into persistent teams that use tools, delegate work, learn from experience, and modify the artifacts that shape their future behavior. The defining question for deployment is no longer merely what agents can do, but who controls what they are allowed to become. We introduce logos, a pluggable layer for self-evolution and governance that strengthens existing multiagent frameworks rather than replacing them. logos compiles heterogeneous multimodal inputs, including documents, images, audio, tables, databases, APIs, and human instructions into versioned agent packs containing agents, tools, knowledge, tests, permissions, and policies. During operation, it transforms agent activity into portable, auditable event traces and applies fail-closed verification across frameworks and backends. Every learned prompt, memory, skill, tool, role, or workflow remains an untrusted release candidate until held-out execution evidence, human-controlled policy, and explicit authorization permit its promotion. This architecture enables ``verifiable human-agent loop engineering'': agents can act, ask, learn, and propose improvements, while humans can steer objectives, permissions, approvals, and irreversible actions without interrupting continuous operation. logos provides a living logic for accountable automation. Agents may evolve at machine speed, but only evidence and human authority can close the loop.
\end{abstract}
\endgroup
\section{Introduction}
\label{sec:intro}

\begin{figure}[tb]
\centering
\includegraphics[width=1.0\linewidth]{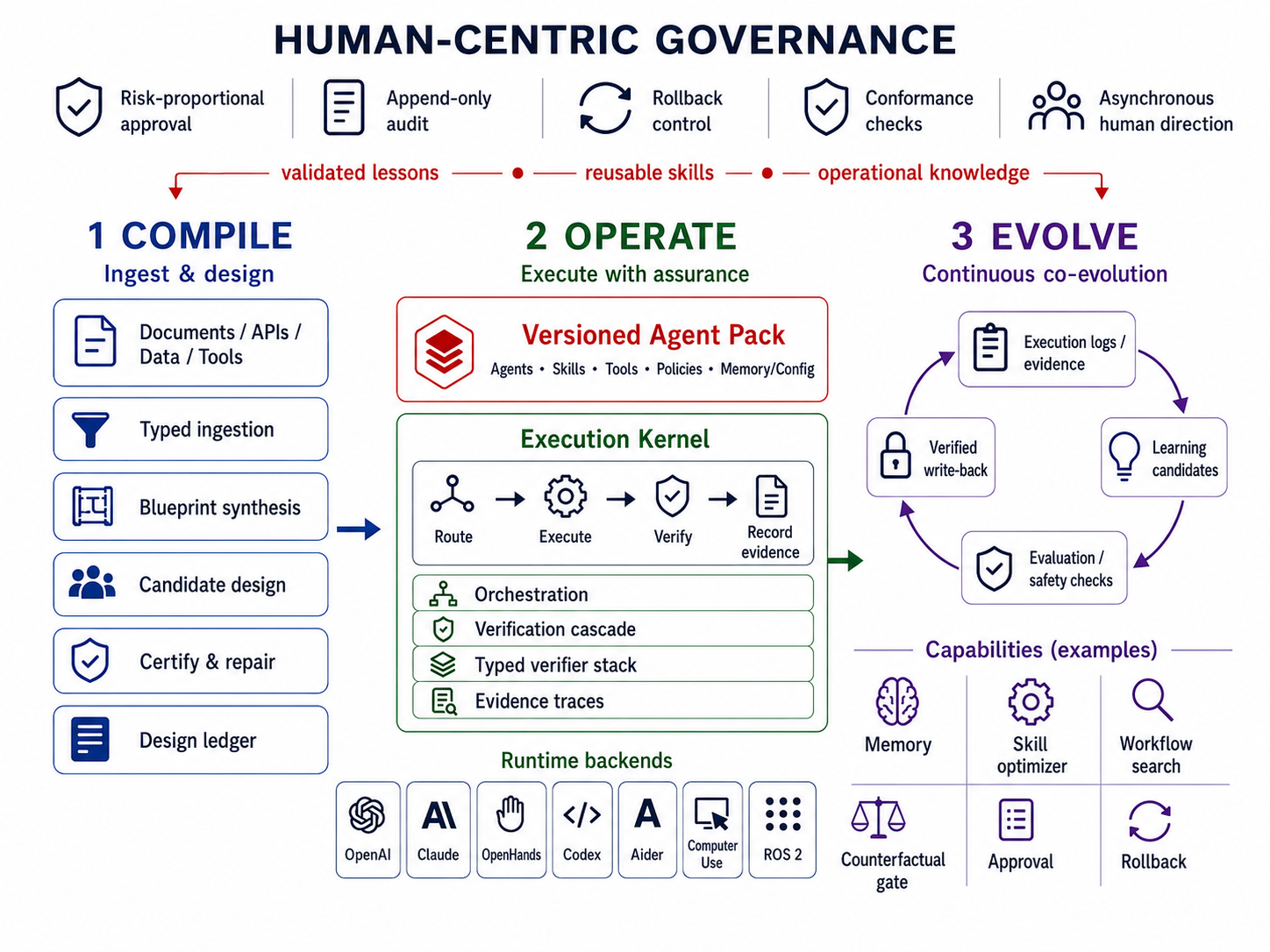}
\caption{The \logos compile--operate--evolve lifecycle. Heterogeneous
multimodal inputs and a human objective are compiled into a probe-validated or
fail-closed \pack. A pluggable execution layer runs the pack across compatible
multi-agent frameworks, while candidate changes remain isolated until a
proposer-independent promotion gate combining evidence and authority checks permits
promotion.}
\label{fig:overview}
\end{figure}

AI systems are becoming persistent actors rather than passive answer engines. Modern agents use tools over long horizons, retain state, coordinate across agents and organizations, and adapt through evaluation and experience~\citep{openai2025agents,google2025a2a}. Once they can revise prompts, memories, skills, tools, roles, and workflows, adaptation becomes a software release problem. The key deployment question is no longer only \emph{what can the model do?} It is also \emph{who decides what the deployed system may become?}

This question matters because persistence, permissions, credentials, memory, delegation, and access to production state turn model capabilities into operational consequences~\citep{nist2023airmf,owasp2025agenticThreats,owasp2026agenticTop10}. A model should not both propose and approve its own updates. A better prompt, useful memory, or a higher-scoring workflow does not, by itself, establish deployment safety or utility. Our governing principle is, therefore, simple: \emph{proposal is not promotion}. Agents may propose changes, but release must depend on independently defined evidence, policy, and authority.

We introduce \logos, a pluggable self-evolution and governance layer for multi-agent systems. Rather than replacing existing frameworks, \logos adds the compile--operate--evolve lifecycle shown in Figure~\ref{fig:overview}.\footnote{The name derives from the ancient Greek \emph{l'ogos}, meaning word, reason, or account. A self-modifying system should be able to account for why a change warrants release.} Compilation converts bounded, heterogeneous inputs, including documents, source code, images, audio, tables, databases, API specifications, and human instructions, into a versioned \pack containing agents, tools, knowledge, memory, routing policies, permissions, verifiers, validation probes, and provenance. Operation normalizes framework-specific behavior into typed, auditable events. Evolution keeps every learned prompt, memory, skill, tool, role, routing policy, or workflow isolated until held-out execution evidence, declared safety and regression constraints, human-owned policy, and required authorization permit promotion.

A root policy remains outside of ordinary self-evolution. It fixes the human-defined objectives, evaluators, and gate versions, protected holdouts, credential boundaries, permission-expansion rules, approval requirements, audit destinations, and allowed external effects. Agents may propose changes within these boundaries, but they cannot silently redefine their objectives, weaken their evaluators, inspect final holdouts, expand credentials, or authorize irreversible actions. Verification establishes whether a change appears to work; policy and approval determine whether it may be released. Rollback restores packs and configurations, while external effects require separate safeguards such as previews, transactions, idempotency, compensation, or human confirmation.

This design transforms human--agent collaboration into an observable, testable, and governable feedback loop. Agents can expose uncertainty, request missing information or authority, learn from outcomes, and propose improvements. Humans can steer asynchronously, restrict permissions, approve sensitive actions, pause execution, and restore the structural state. Oversight is concentrated at explicit control points, each with a defined scope, authority, evidentiary basis, and audit consequences.

Existing work typically separates four tightly coupled concerns. \emph{Construction} creates agents, roles, skills, tools, permissions, tests, and orchestration structures from source material and objectives~\citep{hu2024adas,zhang2024aflow,zhang2025maas,khattab2024dspy}. \emph{Routing} selects which model, agent, tool, or workflow should act under resource constraints~\citep{chen2023frugalgpt,ong2024routellm}. \emph{Verification} defines sufficient evidence for acceptance~\citep{zheng2023llmjudge}. \emph{Adaptation} modifies prompts, memories, skills, tools, roles, workflows, or code based on experience~\citep{gao2025survey,agrawal2025gepa,zhang2025dgm}. Self-evolving search agents raise the same release question for generated curricula~\citep{arai2026eveagent}.

These concerns cannot be governed independently. Routing depends on verifier quality; memory updates change future behavior; tool modifications alter permissions and external effects; and workflow optimization may improve a narrow proxy while reducing deployment utility. \logos therefore treats construction, routing, verification, and adaptation as one release-governance problem: \emph{which version should run next, under whose authority, and on what evidence?}

This paper formalizes and evaluates the governance contract required to answer that question: Agent Packs, portable execution traces, fail-closed verification, evidence-gated promotion, human-owned root policy, agent-initiated questions, and auditable control events. Our contribution is a unified contract for artifacts, events, verification, authority, promotion, and audit that makes pluggable self-evolution governable across multi-agent systems. We evaluate this contract through external benchmarks, mechanism studies, generated fault injection, and contract-conformance checks, using each method only for the claims it can support.

\paragraph{Contributions.}
The central claim is precise: \logos turns agent self-evolution into release engineering and human--agent collaboration into an auditable control loop. Specifically, this paper contributes:
\begin{enumerate}
\item a \textbf{multimodal agent-team compiler} that transforms heterogeneous inputs and human objectives into a probe-validated or fail-closed \pack;
\item a \textbf{pluggable execution and evidence layer} that strengthens compatible multi-agent frameworks through adapters, normalized traces, routing, verification, and target-side re-verification;
\item an \textbf{evidence-governed self-evolution protocol} that keeps learned prompts, memories, skills, tools, roles, and workflows isolated until held-out evidence and human-owned policy authorize promotion;
\item an \textbf{auditable human--agent control contract} for asynchronous direction, agent-initiated questions, scoped authority, approval, rollback, and audit; and
\item a \textbf{bounded evaluation account} that distinguishes task performance, mechanism evidence, fault injection, and contract conformance.
\end{enumerate}

\paragraph{Organization.}
Section~\ref{sec:background} introduces the concepts and notation for governed agent systems. Section~\ref{sec:overview} develops the \logos lifecycle. Section~\ref{sec:experiments} evaluates its mechanisms and evidence boundaries. Section~\ref{sec:related} positions \logos relative to multi-agent frameworks, automated agent design, routing, verification, self-evolution, memory, and human oversight.

\section{A Practical Primer on Governed Agent Systems}
\label{sec:background}

This section is a practical primer for readers who are new to agent systems and
a notation map for readers who already work on them. The basic distinction
is that a model produces tokens, while an agent is a policy that can choose
actions, call tools, observe results, and continue. A team of agents adds
delegation; persistence adds memory; self-evolution adds the ability to change
the artifacts that shape future behavior. Each step is useful, but each also
creates a new control problem. \logos addresses the last step: how a system can
learn from work without silently redefining the objective, evaluator,
permissions, credentials, or production state.

As a running example, consider a team that reviews vendor security-exception
requests. The compiler reads a policy excerpt, intake form, risk matrix, ticket
CSV, and exception-management API schema, then emits a pack with an intake
analyst, policy mapper, risk assessor, response writer, verifier, and a human
approval boundary. If an API schema later changes, the team may propose a tool
schema repair, but that repair is evaluated on held-out requests before it can
replace the live pack. If the team drafts a ticket update, rollback can restore
the pack version, but it cannot unsend an already submitted external update.
That is the practical distinction this paper formalizes.

Three symbols organize the paper. $\Pi$ denotes the agent team inside a pack,
$\mathfrak{D}$ denotes the full governed deployment object that is actually run,
and $\zeta$ denotes the normalized execution trace. The lifecycle is
\begin{equation}
 \mathfrak{D} \xrightarrow{\mathrm{run}} (\zeta,\chi_{\mathrm{eff}}) \xrightarrow{\mathrm{propose}} \mathfrak{D}' \xrightarrow{\mathrm{gate}} \mathfrak{D}_{+},
 \label{eq:bg_lifecycle}
\end{equation}
where $\chi_{\mathrm{eff}}$ records external effects and $\mathfrak{D}_{+}$ is
either the unchanged baseline $\mathfrak{D}$ or an accepted next version. This
notation separates three facts that are often conflated: a trace can justify a
proposal, a proposal is not a release, and rolling back a pack is not the same
as undoing an email, database write, payment, or other external effect.

The rest of the paper uses five terms consistently. A \emph{tool} is a typed capability that can read, compute, or affect outside state. A \emph{verifier} is a typed check whose verdict can accept, reject, or escalate an output. A \emph{trace} is the normalized event record used for audit and learning. A \emph{gate} is the release rule for candidate artifact changes. A \emph{root policy} is the human-owned envelope that fixes objectives, credentials, evaluators, approval rules, effect boundaries, and audit sinks outside ordinary self-evolution.

\begin{table}[tb]
\centering\small
\setlength{\tabcolsep}{4.2pt}\renewcommand{\arraystretch}{1.12}
\caption{From a model to a self-evolving system.}
\label{tab:agent_ladder}
\begin{tabular}{@{}L{0.18\linewidth}L{0.30\linewidth}L{0.42\linewidth}@{}}
\toprule
Stage & Added capability & New control problem \\
\midrule
Model & generates an answer & correctness and calibration \\
Agent & calls tools and acts & permissions and side effects \\
Agent team & delegates and coordinates & coordination and credit assignment \\
Persistent system & keeps memory and state & poisoning, staleness, and privacy \\
Self-evolving system & changes future behavior & evaluator gaming and authority drift \\
\bottomrule
\end{tabular}
\end{table}

For example, a support agent may fail after an API schema changes. An ungated evolution loop may rewrite the prompt because prompt editing is available, even though the defect is in the tool schema. The new prompt may improve a narrow proxy score while still failing against the updated API. The system has changed, yet deployment utility has fallen. This is \emph{misevolution}: a release artifact changed without evidence that the change improves the deployment.

\begin{table}[tb]
\centering\small
\setlength{\tabcolsep}{4.0pt}\renewcommand{\arraystretch}{1.15}
\caption{Deployment questions and release invariants used in later sections.}
\label{tab:governance_map}
\begin{tabular}{@{}L{0.16\linewidth}L{0.32\linewidth}L{0.42\linewidth}@{}}
\toprule
Question & Failure mode & Release invariant \\
\midrule
Build & manual wiring, missing tools, oversized permissions & the emitted pack passes declared probes or is returned as a typed diagnostic \\
Operate & avoidable cost, weak routing, silent verifier gaps & routes, verifier verdicts, costs, risks, and external effects are trace events \\
Verify & false accepts, judge overreach, schema-only approval & hard checks are necessary but not sufficient; semantic evidence and authority remain distinct \\
Adapt & regression, memory poisoning, proxy gaming, authority drift & a proposal is not a promotion, and root policy remains outside ordinary self-evolution \\
Reuse and effects & harmful transfer, rollback confusion, irreversible writes & imported knowledge is re-gated on the target, and pack rollback does not undo external effects \\
\bottomrule
\end{tabular}
\end{table}

Recent work increasingly treats agent design, prompts, workflows, and even
agent code as optimization objects~\citep{gao2025survey,hu2024adas,
zhang2024aflow,agrawal2025gepa,zhang2025dgm}; evidence-verification approaches
for self-evolving search agents and documented misevolution risks make the
release boundary correspondingly more important~\citep{arai2026eveagent,
shao2025misevolve}. Section~\ref{sec:related} collects the broader industrial
and research context so that this primer can stay focused on the core control
concepts.

The formal notation follows the operational order: a team of agents, a compiler that builds the team, a runtime that routes and verifies work, a memory that carries experience forward, and a gate that decides whether self-evolution has enough evidence to accept a candidate for promotion. Longer derivations are collected in Appendix~\ref{sec:bg_ext}.

\subsection{Multi-Agent Systems and Application Frameworks}
\label{sec:bg_mas}

\textbf{What is the deployed object?} A large language model, abbreviated LLM, maps a prompt to text. An \emph{agent} wraps that model in a loop that can choose an action, observe the result, and continue. If $h_{i,t}$ is agent $i$'s task history at step $t$, then a compact single-agent view is
\begin{equation}
 a_{i,t}\sim\psi_i(\cdot\mid h_{i,t},C_i,W_i;\xi),\quad a_{i,t}\in\mathcal{A}_{\mathrm{text}}\cup W_i\cup\{\textsc{stop}\}.
 \label{eq:single_agent}
\end{equation}
where $\mathcal{A}_{\mathrm{text}}$ is the set of text-emission actions, $C_i$ is the agent context, $W_i$ is its typed tool set, and $\xi$ represents randomness from decoding, scheduling, tools, or external services. A \emph{multi-agent system}, abbreviated MAS, is a team of such agents connected by delegation and shared state:
\begin{equation}
 \Pi = \bigl(\Gamma, \{\psi_i\}_{i=1}^{m}, \{C_i\}_{i=1}^{m}, \{W_i\}_{i=1}^{m}, E, \mathcal{B}_{\mathrm{exec}}\bigr).
 \label{eq:mas}
\end{equation}
Here $\Gamma$ is the communication and delegation graph; $\psi_i$ is agent $i$'s possibly stochastic, history-dependent model-and-decoding policy; $C_i$ is its context, including prompts, memory, and procedural skills; $W_i$ is its set of typed tools, such as a REST client, SQL reader, or file operation; $E$ is the shared environment or state interface; and $\mathcal{B}_{\mathrm{exec}}$ is the execution budget for calls, latency, or cost.

The MAS is not the whole safety boundary. The governed deployment object is
\begin{equation}
 \mathfrak{D} =
 \bigl(\mathsf{Pack},R_{\mathrm{root}},b,\mathsf{Cred},S_{\mathrm{env}},\Lambda_{\mathrm{audit}}\bigr),
 \label{eq:deployment_object}
\end{equation}
where $R_{\mathrm{root}}$ is the human-owned root policy outside the ordinary self-evolution loop, $b$ is a backend plus adapter, $\mathsf{Cred}$ is the credential capability envelope, $S_{\mathrm{env}}$ is live external state, and $\Lambda_{\mathrm{audit}}$ is the audit destination and policy. Candidate evaluation may copy pack artifacts and sandboxed state, but it may not mutate $R_{\mathrm{root}}$, credentials, live external state, or the audit sink.
Strictly speaking, an ordinary self-evolution candidate is a pack candidate,
$\mathsf{Pack}'$, evaluated under the fixed envelope
$(R_{\mathrm{root}},\mathsf{Cred},\Lambda_{\mathrm{audit}})$ and a sandboxed
state copy $\widetilde S_{\mathrm{env}}$. When we write a candidate deployment
$\mathfrak{D}'$, it is shorthand for this pack candidate placed under the same
human-owned envelope, not permission for the candidate to rewrite that
envelope.

We distinguish a deployment \emph{objective} or brief $q\in\mathcal{Q}$, such as ``triage support tickets,'' from an executable task space $\mathsf{T}_q$ and a deployment distribution $P_q\in\Delta(\mathsf{T}_q)$. A concrete task is $\tau\sim P_q$, and a held-out gate set is a finite sample $\mathcal{H}=(\tau_1,\ldots,\tau_n)\sim P_q^n$ unless a table explicitly states a fixed benchmark split. Running $\mathfrak{D}$ on $\tau$ with execution randomness $\xi$ produces an output, trace, updated external-state handle, and external-effect record
\begin{equation}
 (y,\zeta,S'_{\mathrm{env}},\chi_{\mathrm{eff}})=\mathrm{run}(\mathfrak{D},\tau;\xi),\quad \zeta=(e_1,\ldots,e_K)\in\mathcal{Z},\quad e_k\in \mathcal{E}_{\mathrm{tagged}}.
 \label{eq:trace}
\end{equation}
The event space $\mathcal{E}_{\mathrm{tagged}}$ is a tagged union over event kinds such as task, plan, action, observation, route decision, verifier verdict, cost, risk, human question, effect, and outcome; Appendix~\ref{sec:bg_ext} expands it. The trace is the common evidence stream used for debugging, routing, memory updates, and self-evolution. The effect record $\chi_{\mathrm{eff}}$ is what prevents a pack rollback from being confused with undoing emails, payments, database writes, or other outside-world effects. Task quality is measured by a scorer $s(\mathfrak{D},\tau;\xi)\in[0,1]$: tests for code, exact or token-overlap scoring for QA, or a rubric for free-form work. We write the deployment utility as
\begin{equation}
 J(\mathfrak{D};P_q)=\mathbb{E}_{\tau\sim P_q,\xi}\bigl[s(\mathfrak{D},\tau;\xi)-\lambda_c c(\mathfrak{D},\tau;\xi)-\lambda_\ell \ell(\mathfrak{D},\tau;\xi)-\lambda_r r_{\mathrm{risk}}(\mathfrak{D},\tau;\xi)\bigr],
 \label{eq:utility}
\end{equation}
where $c$, $\ell$, and $r_{\mathrm{risk}}$ denote monetary cost, latency, and risk, and the $\lambda$'s set their prices. Later objectives are simpler versions of Eq.~\eqref{eq:utility}. For task-level gates we distinguish the raw score $s$ from the utility integrand
\begin{equation}
 u(\mathfrak{D},\tau;\xi)=s(\mathfrak{D},\tau;\xi)-\lambda_c c(\mathfrak{D},\tau;\xi)-\lambda_\ell \ell(\mathfrak{D},\tau;\xi)-\lambda_r r_{\mathrm{risk}}(\mathfrak{D},\tau;\xi).
 \label{eq:task_utility}
\end{equation}
and write the gate metric as $m(\mathfrak{D},\tau;\xi)\in\{s(\mathfrak{D},\tau;\xi),u(\mathfrak{D},\tau;\xi)\}$. Most benchmark tables instantiate $m=s$ to isolate accuracy, while deployment gates can instantiate $m=u$ when cost, latency, or risk budgets are part of the acceptance contract.

In practice, MAS instances run inside agent frameworks such as AutoGen, MetaGPT, OpenHands, and OpenAI Agents~\citep{wu2023autogen,hong2024metagpt,wang2024openhands,openai2025agents}. Each framework has its own configuration, capability surface, and trace format. \logos makes learning portable with two contracts: a normalizer translates every backend trace $\mathcal{Z}_b$ into the common trace space $\mathcal{Z}$, and an emitter projects portable lessons back into backend-native artifacts $\mathcal{A}_b$:
\begin{equation}
 \phi_b:\mathcal{Z}_b\to\mathcal{Z},\quad \epsilon_b:\mathcal{A}_{\mathrm{portable}}\times G\times\mathcal{C}_b\to\mathcal{A}_b\cup\mathcal{D}_b.
 \label{eq:tracemap}
\end{equation}
Here $\mathcal{A}_{\mathrm{portable}}$ is a portable candidate artifact, $\mathcal{C}_b$ is the backend capability declaration, $\mathcal{A}_b$ is a backend-native artifact, and $\mathcal{D}_b$ is a compatibility diagnostic. The emitter is not a mathematical inverse or a bit-exact replay of backend-internal framework state; it is the contract that lets a skill learned on one backend be re-expressed and re-verified on another.

\subsection{The MAS Construction Problem}
\label{sec:bg_construct}

\textbf{What team should exist?} The first problem is to obtain a team whose roles, tools, permissions, and tests match the intended work. Construction is a typed map
\begin{equation}
 \mathrm{build}:(\mathcal{S}, q)\longmapsto\mathsf{BuildResult}\in\{\textsc{ProbeValidatedPack},\textsc{DiagnosticPack}\},
 \label{eq:bg_build}
\end{equation}
where $\mathcal{S}$ may contain bounded inputs from documents, tables, databases, images, or audio, and $q$ is the brief. A successful result is probe-checked only against the stated finite probes; throughout the paper, ``validated'' in \textsc{ProbeValidatedPack} has only this finite-probe meaning. A diagnostic result carries repair information and no live authority. In \logos the emitted pack is richer than $\Pi$ alone: an \pack contains an initial team $\Pi_0$, a versioned genome $G$, a verifier stack $\mathcal{V}$, finite validation probes $\mathcal{C}$, and manifest metadata $\mathcal{M}_{\mathrm{man}}$. This turns construction into a software build problem: the artifact should be runnable, inspectable, and testable, or fail-closed with a typed diagnostic.

Prior systems usually fix one side of this map. Hand-built frameworks leave the design to a developer~\citep{wu2023autogen,hong2024metagpt,qian2024chatdev,li2023camel}; pipeline compilers such as DSPy optimize within a declared program~\citep{khattab2024dspy}; automated designers search workflows or agent graphs on curated benchmarks~\citep{hu2024adas,zhuge2024gptswarm,shang2024agentsquare}. The missing piece for deployment is a fail-closed build interface from heterogeneous bounded operational inputs to either a team that has passed declared probes or a diagnostic artifact that carries no live authority. Section~\ref{sec:compiler} gives the compiler, whose fallback contract returns a typed success or diagnostic object and whose validation loop checks that the emitted tools are actually agent-callable under stated probes.

The tuple notation in this paper is an interface schema: it states what a pack,
trace, verifier, or gate must record and where authority changes hands. It is
not meant to turn every data structure into a theorem. Formal claims are stated
only where the assumptions and denominators are explicit.

\subsection{Routing in Agentic Systems}
\label{sec:bg_routing}

\textbf{Who should act on this task?} Once a team exists, each incoming task still needs a runtime decision: which model, agent, or workflow should handle it? We use one notation for all of these choices. A routing policy maps a task context $x$ to a probability distribution, denoted by $\Delta(\cdot)$, over finite routing actions:
\begin{equation}
 \rho:\mathcal{X}\longrightarrow\Delta(\mathcal{A}_{\mathrm{route}}),\quad \mathcal{A}_{\mathrm{route}}=\{a_1,\ldots,a_K\}.
 \label{eq:bg_router}
\end{equation}
Write $x_\tau\in\mathcal{X}$ for the routing context extracted from task $\tau$ and the current deployment state. If action $a$ on task $\tau$ yields reward $r(\tau,a)$ at cost $c(\tau,a)$, the basic routing objective is
\begin{equation}
 \rho^\star=\argmax_{\rho}\mathbb{E}_{\tau\sim P_q}\mathbb{E}_{a\sim\rho(x_\tau)}\bigl[U(\tau,a)\bigr],\quad U(\tau,a)=r(\tau,a)-\lambda_{\mathrm{cost}}c(\tau,a).
 \label{eq:bg_route_obj}
\end{equation}
Because only the chosen action's result is observed, learned routing is a contextual bandit~\citep{auer2002ucb,lattimore2020bandit}.

There are two useful regimes. A \emph{pre-answer} router commits before seeing any answer: for example a static cost router that sends simple tasks to a cheap tier and hard tasks to a stronger tier. An \emph{evidence-aware} router observes intermediate evidence and can escalate. The main evidence-aware pattern is a verification-gated cascade: try cheaper stages first, accept only when a verifier passes, and otherwise escalate. Under the explicit cascade assumptions in Section~\ref{sec:cascade}, a sound verifier and strongest fallback make cascade accuracy no worse than always using the fallback; Appendix~\ref{sec:bg_ext} gives the cost condition and false-accept accounting. Section~\ref{sec:routing_verif} implements static routing, learned routing, and such a cascade in one family.

\subsection{Verification of Agent Outputs}
\label{sec:bg_verification}

\textbf{What evidence is sufficient for acceptance?} Verification is the trust primitive used by routing, compilation, and self-evolution. A verifier returns a typed verdict
\begin{equation}
 \mathbb{V}=\{\textsc{pass},\textsc{fail},\textsc{not-applicable},\textsc{error}\},\quad \nu_k(\tau,y;\omega_k)\in \mathbb{V},\quad b(v)=\indic{v=\textsc{pass}}.
 \label{eq:bg_verifier}
\end{equation}
Here $\omega_k$ collects verifier configuration or randomness. The Boolean map $b$ is the only value used for acceptance. The important safety quantities must state their denominator: acceptance among incorrect outputs, $\Pr[b(\nu)=1\mid\mathrm{incorrect}]$, is different from error among accepted outputs, $\Pr[\mathrm{incorrect}\mid b(\nu)=1]$. Cascade analysis additionally conditions on reaching a stage.

\logos separates output verification from candidate promotion. Output verification uses hard constraints, such as schema validation, syntax, lint, sandbox limits, and safety allow-lists, together with acceptance requirements that declare which task evidence is sufficient for a named scope. Hard constraints are necessary but not sufficient: at least one applicable semantic acceptance requirement must pass, and missing, crashed, inapplicable, or unparsable checks fail closed. Human approval governs authority-sensitive actions; it is not a verifier verdict and cannot turn missing semantic evidence into a pass. Section~\ref{sec:verification} gives the full composite rule in Eq.~\eqref{eq:verifier_composite}.

\subsection{Self-Evolution as Optimization over Agent Systems}
\label{sec:bg_selfevolve}

\textbf{What may change, and who authorizes the change?} Self-evolution proposes changes to a deployed pack from its own experience; proposal is intentionally typed so it cannot mutate the live deployment. We write a generic proposal operator as
\begin{equation}
 f_{\mathrm{prop}}:(\mathsf{Pack},\zeta,r_{\mathrm{fb}};R_{\mathrm{root}})\longmapsto \mathsf{CandidatePack}\cup\{\textsc{abstain}\},
 \label{eq:evolve_op}
\end{equation}
where $\zeta$ is the trace and $r_{\mathrm{fb}}$ is feedback, such as a score, human rating, or encoded directive. The operator may propose edits to prompts, skills, memory, tools, workflows, verifier declarations, or code-facing artifacts, but its output is an isolated candidate pack or an abstention. Following the survey taxonomy of \citet{gao2025survey}, we ask four questions for each operator: \emph{what} changes, \emph{when} it changes during or between tasks, \emph{how} it is proposed by reward, imitation, or population search, and \emph{where} the change is stored as a deployment-pack or shared-asset update.

This taxonomy helps position the literature. Reflexion and Self-Refine edit context during a task~\citep{shinn2023reflexion,madaan2023selfrefine}; Voyager accumulates skills across a world~\citep{wang2024voyager}; GEPA evolves prompts~\citep{agrawal2025gepa}; AFlow and MaAS search workflows~\citep{zhang2024aflow,zhang2025maas}; EVE-Agent makes self-generated search curricula auditable by requiring evidence spans and an evidence verifier~\citep{arai2026eveagent}; and the Darwin--G\"odel Machine changes agent code~\citep{zhang2025dgm}. Section~\ref{sec:echo} shows how \logos implements memory, skill, workflow, and code-facing variants under one promotion gate.

\subsection{Memory in Agent Systems}
\label{sec:bg_memory}

\textbf{What survives from one task to the next?} Between-task evolution needs memory. A memory has a write operator that stores an episode and a read operator that recalls relevant items:
\begin{equation}
 \omega:(\mathcal{M}_{\mathrm{mem}},\eta)\longmapsto\mathcal{M}_{\mathrm{mem}}',\quad \mathrm{read}:(\mathcal{M}_{\mathrm{mem}},\mathbf{q})\longmapsto\mathcal{R}\subseteq\mathcal{M}_{\mathrm{mem}}.
 \label{eq:bg_memory}
\end{equation}
Here $\eta$ is an episode, $\mathbf{q}$ is an embedding of the current task, and $\mathcal{R}$ is injected into an agent's context. Similarity alone is not enough: a memory can be semantically close but harmful. \logos therefore treats retrieval as a value-aware selection problem with bandit-style updates, combining semantic relevance with learned outcome value. Section~\ref{sec:memory} realizes this as a three-tier memory whose raw episodes are promoted to validated lessons and durable skills only after passing gates. LongMemEval and LongMemEval-V2 make clear that long-lived agent memory should be evaluated not only as chat-history recall, but also as environment-specific workflow, state, and failure-mode knowledge~\citep{wu2025longmemeval,wu2026longmemevalv2}.

\subsection{Cross-Deployment Knowledge Transfer}
\label{sec:bg_transfer}

\textbf{When is prior experience safe to reuse?} A platform should support controlled reuse across deployments, not only isolated per-pack learning. Let $\mathcal{K}_u$ be the platform knowledge state before processing deployment $u$:
\begin{equation}
 \mathcal{K}_u=(\mathcal{D}_u,\Sigma_u,\mathcal{L}_u),\quad \mathcal{K}_{u+1}=\omega_{\mathrm{plat}}\bigl(\mathcal{K}_u,\Pi^{(u)},\{\zeta^{(u)}_t\}_t\bigr).
 \label{eq:bg_transfer}
\end{equation}
The three components are design precedents $\mathcal{D}$, validated skills $\Sigma$, and shared workplace lessons $\mathcal{L}$, and $\omega_{\mathrm{plat}}$ is the platform-level knowledge update operator. A new build can condition on $\mathcal{K}_{u}$, but recalled knowledge can also hurt when the new brief differs from the old one. Because $\mathrm{build}$ returns a \textsc{ProbeValidatedPack} or a \textsc{DiagnosticPack}, transfer utility is defined only after a validated pack is unwrapped and deployed under the target root envelope:
\begin{align}
 \mathsf{Pack}^{(u)}_{\mathcal{K}}
 &= \operatorname{unwrapValidated}\bigl(\mathrm{build}(\mathcal{S}^{(u)},q^{(u)};\mathcal{K}_{u})\bigr), \nonumber\\
 \mathsf{Pack}^{(u)}_{\varnothing}
 &= \operatorname{unwrapValidated}\bigl(\mathrm{build}(\mathcal{S}^{(u)},q^{(u)};\varnothing)\bigr), \nonumber\\
 \mathfrak{D}^{(u)}_{\mathcal{K}}
 &= \operatorname{deploy}\bigl(\mathsf{Pack}^{(u)}_{\mathcal{K}},
 R_{\mathrm{root}}^{(u)},b^{(u)},\mathsf{Cred}^{(u)},S_{\mathrm{env},0}^{(u)},\Lambda_{\mathrm{audit}}^{(u)}\bigr), \nonumber\\
 \mathfrak{D}^{(u)}_{\varnothing}
 &= \operatorname{deploy}\bigl(\mathsf{Pack}^{(u)}_{\varnothing},
 R_{\mathrm{root}}^{(u)},b^{(u)},\mathsf{Cred}^{(u)},S_{\mathrm{env},0}^{(u)},\Lambda_{\mathrm{audit}}^{(u)}\bigr), \nonumber\\
 \Delta_u
 &= J\bigl(\mathfrak{D}^{(u)}_{\mathcal{K}};P_{q^{(u)}}\bigr)
 -
 J\bigl(\mathfrak{D}^{(u)}_{\varnothing};P_{q^{(u)}}\bigr).
 \label{eq:bg_transfer_gain}
\end{align}
The expression is evaluated only when both calls to $\operatorname{unwrapValidated}$ succeed. If either build returns a diagnostic pack, $\Delta_u$ is not a utility gain; the event is recorded as a fail-closed build result with no live authority. \logos treats transfer as a candidate, never as an unconditional prior: design precedents, ported skills, and World lessons must re-enter the same non-regression lifecycle in their target deployment. The present evidence is therefore authority-boundary evidence rather than a transfer-gain estimate.

\subsection{Human Oversight and the Value of Information}
\label{sec:bg_hitl}

\textbf{When should the system ask a human?} A deployed agent team should ask a human when the expected value of the answer exceeds the interruption cost. In the ideal decision rule, let $\mathcal{A}_{\mathrm{dec}}$ be the set of possible decisions, let $U(a;x)$ be the expected utility of choosing decision $a$ in context $x$, and let $U(a;x,b)$ be the same utility after observing answer $b$ to question $\beta$. The question should be asked when
\begin{equation}
 \mathrm{VOI}(\beta;x)=\mathbb{E}_{b\sim P(\cdot\mid x,\beta)}\Bigl[\max_{a\in\mathcal{A}_{\mathrm{dec}}} U(a;x,b)\Bigr]-\max_{a\in\mathcal{A}_{\mathrm{dec}}} U(a;x)\ge c_{\mathrm{ask}}(\beta).
 \label{eq:bg_voi}
\end{equation}
This quantity is rarely observable, so \logos uses a deployable surrogate: estimated uncertainty, deliverable impact, and cost of being wrong, under a human-attention budget. When labeled calibration data are available, the uncertainty estimate can be calibrated; otherwise it remains a heuristic risk signal. Section~\ref{sec:hitl} also handles the reverse direction: free-form human directives are fused with validated memory and clarified when ambiguous, rather than appended blindly.

\subsection{Misevolution and the Safety Requirement}
\label{sec:bg_misevolution}

\textbf{How do we distinguish adaptation from degradation?} The danger of self-evolution is \emph{misevolution}: the deployment changes itself and becomes worse than the version it started from~\citep{shao2025misevolve}. Because a deployment may value task success, cost, latency, and risk, we define the event on the utility $J$ of Eq.~\eqref{eq:utility}. A trajectory exhibits misevolution if, for some round $t$,
\begin{equation}
 J\bigl(\mathfrak{D}_t;P_q\bigr) < J\bigl(\mathfrak{D}_0;P_q\bigr).
 \label{eq:misevolution}
\end{equation}
\logos's population target is adjacent-version non-regression in mean utility,
\begin{equation}
 \Delta J = J(\mathfrak{D}';P_q) - J(\mathfrak{D};P_q) \ge 0.
 \label{eq:nonregression}
\end{equation}
The deployed gate observes only a held-out sample $\mathcal{H}\sim P_q^n$ or a declared benchmark split and checks the empirical surrogate
\begin{equation}
 \widehat{\Delta}_{\mathcal{H}}=\frac{1}{|\mathcal{H}|}\sum_{\tau\in\mathcal{H}}\bigl(m(\mathfrak{D}',\tau;\xi_\tau')-m(\mathfrak{D},\tau;\xi_\tau)\bigr),
 \label{eq:bg_empirical_delta}
\end{equation}
under the metric choice $m\in\{s,u\}$ stated by the table or deployment, with $\xi_\tau'=\xi_\tau$ only when common random numbers or deterministic replay are available. The paired gate reports this held-out evidence; by itself it does not prove Eq.~\eqref{eq:nonregression} on the deployment population.

The deployed gate has three empirical regimes. \emph{Positive observed mean gain} requires the gate-sample mean to clear a margin. \emph{Regression-budgeted evidence} additionally caps the number or rate of observed task drops, as in Eq.~\eqref{eq:echo_gate}. \emph{Zero observed task-level regressions} sets that budget to zero on the gate sample. An optional anytime-valid gate adds error control under the sampling and holdout-firewall assumptions of Section~\ref{sec:echo_gate}. Separately, a cumulative baseline audit periodically compares the current deployment against $\mathfrak{D}_0$ or a human-designated stable baseline, because adjacent-version updates can still accumulate into global degradation under finite-sample error or distribution shift.
Two failure modes make this difficult. \emph{Credit dilution} means a trace does not reveal which component caused a failure, so the wrong asset may be edited. \emph{Proxy evaluation} means a cheap surrogate score can approve changes that hurt the real task. Section~\ref{sec:echo} shows how the paired gate addresses both operationally: it assigns a primary typed component or abstains, then evaluates candidates by task-paired execution of baseline and candidate in the target execution harness on held-out tasks.

\subsection{Anytime-Valid Inference}
\label{sec:bg_avi}

\textbf{How are repeated release decisions monitored?} \logos makes repeated accept/reject decisions while evidence arrives over time. Fixed-sample tests are not valid if we inspect results after every task and stop at the first favorable time. Anytime-valid inference addresses this problem with non-negative evidence processes whose crossing probability is controlled under the null hypothesis~\citep{ramdas2023gametheoretic,shafer2021testing,vovk2021evalues}. Operationally, an e-process accumulates evidence against a null while remaining valid under optional stopping. Ville's inequality gives the key guarantee: threshold crossing remains valid at any stopping time~\citep{ville1939}.

In \logos this machinery appears in three places. Compilation can attach anytime-valid evidence to sampled tool callability; the optional adoption gate can attach anytime-valid evidence to candidate gains and regression rates; and the LORD-style ledger accounts for repeated candidate decisions. Formal online error control applies only when the candidate-level p-values and dependence assumptions stated in Appendix~\ref{sec:sage_lord_formal} hold~\citep{javanmard2018online,ramdas2018saffron}. Appendix~\ref{sec:bg_ext} gives the formal processes, confidence sequences, and spending rule.

\subsection{Lifecycle Contract}
\label{sec:bg_map}

\logos's contract is deliberately narrow: each safety-critical lifecycle decision must be represented as an artifact decision and checked by execution evidence before it can affect the live deployment. Construction must return a pack that passes declared probes or a diagnostic artifact. Operation must emit normalized events and fail-closed verifier outcomes. Adaptation must remain an isolated candidate until paired evidence, root-policy compatibility, and required authorization permit promotion. Human oversight is not a side channel: questions, directives, approvals, rejections, pause/resume, rollback, and effect controls are scoped events inside the same lifecycle.

The corresponding mechanisms are developed together in Section~\ref{sec:overview}. Interface presentation, storage layout, and infrastructure choices are outside the scientific claim; the claim concerns artifact semantics, event semantics, gate semantics, and authority semantics.

\section{\logos: Release Governance for Agent Teams}
\label{sec:overview}

Section~\ref{sec:background} identified four recurring questions: what to
build, who should act, what evidence is sufficient, and what may change. \logos
answers them with the release lifecycle introduced in Section~\ref{sec:intro}:
\emph{compile} a governed Agent Pack, \emph{operate} it through a
trace-normalizing kernel, \emph{propose} changes from evidence, and
\emph{promote} only candidates that pass an isolated gate. The architecture does
not require every mechanism to be novel in isolation. What matters is that all
mechanisms write to the same release record and obey the same
authority boundary. Four concepts are sufficient for a first reading; named
modules implement them:

\begin{table}[tb]
\centering\small
\setlength{\tabcolsep}{4.5pt}\renewcommand{\arraystretch}{1.12}
\caption{Four core concepts used to read the \logos platform.}
\label{tab:core_concepts}
\begin{tabular}{@{}>{\raggedright\arraybackslash}p{0.22\linewidth}>{\raggedright\arraybackslash}p{0.34\linewidth}>{\raggedright\arraybackslash}p{0.34\linewidth}@{}}
\toprule
Concept & Role in \logos & Operational implication \\
\midrule
Agent Pack & versioned release artifact containing agents, tools, memories, verifiers, policy, probes, and provenance & the deployable unit is inspectable and testable \\
Execution Kernel & adapter layer that runs packs on compatible backends and normalizes traces & backend portability does not bypass verification \\
Evolution Gate & external promotion rule for memory, skill, tool, role, and workflow changes & candidates are evaluated before they can mutate the live pack \\
Human Authority & root policy, approval, credentials, permission expansion, and final-holdout control & release authority remains human-owned \\
\bottomrule
\end{tabular}
\end{table}

\logos is a release-governance layer above foundation models and orchestration
frameworks. Its central invariant is that no candidate modifies the live
deployment before an isolated gate authorizes promotion. This invariant is
stronger than an isolated prompt-edit test: it also
fixes who owns the evaluator, which credentials the candidate may use, what
side effects can be previewed, which holdout feedback is hidden from future
proposals, and how rollback and audit are scoped.

Equivalently, for live state $(\mathfrak{D}_t,S_{\mathrm{env},t})$, proposal and evaluation must run on an isolated pack candidate $(\mathsf{Pack}',\widetilde S_{\mathrm{env}})$ with scoped candidate credentials and an effect-preview envelope. The promotion predicate decomposes into evidence, policy, and authorization:
\begin{equation}
 \mathrm{Promote}(\mathsf{Pack}';\mathfrak{D}_t)\iff B_{\mathrm{gate}}(\mathfrak{D}_t,\mathsf{Pack}';\mathcal{H}_{\mathrm{gate}})\wedge B_{\mathrm{root}}(\mathsf{Pack}',\widetilde{\chi}_{\mathrm{preview}};R_{\mathrm{root}})\wedge B_{\mathrm{auth}}(\mathsf{Pack}';R_{\mathrm{root}}),
 \label{eq:promotion_contract}
\end{equation}
where $B_{\mathrm{gate}}$ is the candidate-level held-out gate decision, $B_{\mathrm{root}}$ checks the human-owned policy against the previewed effect record $\widetilde{\chi}_{\mathrm{preview}}$, and $B_{\mathrm{auth}}$ records any role-based or human approval required by that policy. This separates output acceptance $B_{\mathrm{output}}(\tau,y)$ from candidate promotion: evidence decides whether the candidate appears to work, while policy and authority decide whether it may become real. Under the trusted-computing-base assumptions in Section~\ref{sec:security_artifacts}, until this predicate passes, the transition is required to be the identity on live deployment and live external state:
\begin{equation}
 \neg\mathrm{Promote}(\mathsf{Pack}';\mathfrak{D}_t)\Longrightarrow(\mathfrak{D}_{t+1},S_{\mathrm{env},t+1})=(\mathfrak{D}_t,S_{\mathrm{env},t}).
 \label{eq:adoption_invariant}
\end{equation}
This section is organized around Agent Pack, Execution Kernel,
Evolution Gate, and Human Authority. Named modules are mechanisms and are
detailed only where they are needed to explain the central
release-governance contract. It defines the pack and its evolvable
genome, the kernel that executes it, the lifecycle interface, compilation,
routing, verification, execution-gated self-evolution, human governance, and
security. Table~\ref{tab:lifecycle_obligations} summarizes the evidence
obligation each lifecycle stage must discharge before it can affect the live
deployment.

\begin{table}[tb]
\centering\small
\setlength{\tabcolsep}{4.0pt}\renewcommand{\arraystretch}{1.12}
\caption{Lifecycle stages and the evidence each stage must leave behind.}
\label{tab:lifecycle_obligations}
\begin{tabular}{@{}L{0.17\linewidth}L{0.34\linewidth}L{0.39\linewidth}@{}}
\toprule
Stage & What changes & Evidence obligation \\
\midrule
Compile & source material becomes a candidate pack & typed manifest, tool and skill specs, validation probes, provenance, diagnostic fallback \\
Operate & tasks are executed through a backend & normalized trace, route decisions, tool/effect records, verifier verdicts, cost and risk events \\
Propose & experience becomes a candidate edit & component attribution or abstention, patch record, sandbox result, root-policy compatibility check \\
Promote & accepted candidate replaces live artifact & paired gate result, safety and regression record, approval when required, rollback pointer and audit hash \\
\bottomrule
\end{tabular}
\end{table}

\subsection{The Agent Pack and the Agent Genome}
\label{sec:pack}

An \pack is a versioned release representation of a MAS instance. It pairs an initial system $\Pi_0$ of Eq.~\eqref{eq:mas}, materialized as agent definitions, procedural \emph{skills}, and typed \emph{tools}, with a declarative \emph{\genome} $G$, verifier declarations $\mathcal{V}$, validation probes $\mathcal{C}$, and manifest metadata $\mathcal{M}_{\mathrm{man}}$:
\begin{equation}
\begin{aligned}
 \mathsf{Pack}
 &= \bigl(\Pi_0,G,\mathcal{V},\mathcal{C},\mathcal{M}_{\mathrm{man}}\bigr),\\
 G
 &= \bigl(G_{\text{route}},G_{\text{mem}},G_{\text{verify}},G_{\text{promo}},
           G_{\text{safety}},G_{\text{sand}},G_{\text{tools}},G_{\text{att}}\bigr).
\end{aligned}
 \label{eq:genome}
\end{equation}
The eight axes govern routing and fallback $G_{\text{route}}$, memory $G_{\text{mem}}$, verification declarations used during ordinary operation $G_{\text{verify}}$, promotion staging $G_{\text{promo}}$, execution safety $G_{\text{safety}}$, sandbox limits $G_{\text{sand}}$, tool exposure including human questions $G_{\text{tools}}$, and the human-attention budget $G_{\text{att}}$. The manifest $\mathcal{M}_{\mathrm{man}}$ records schema versions, adapter requirements, provenance, hashes, and artifact/replay metadata. The genome is assembled by a field-wise deep overlay $G = G_{\text{pack}} \sqcup G_{\text{task}}$, where $G_{\text{pack}}$ contains pack defaults and $G_{\text{task}}$ contains task- or brief-specific overrides. An unset override never clobbers a default; conflicting concrete values are resolved by declared precedence or reported as diagnostics rather than silently merged. A candidate $G'$ can be staged, promoted only under the non-regression evidence of Eq.~\eqref{eq:nonregression}, or rolled back, as Figure~\ref{fig:genome} illustrates and Section~\ref{sec:evolution} formalizes.

\paragraph{Human-owned root policy.}
The evolvable genome is not the whole deployment contract. A separate root policy, owned outside the ordinary self-evolution loop, fixes the human objective, minimum safety constraints, final evaluator and gate versions, final-holdout sampler, credential boundaries, permission-expansion rules, approval policy, external-effect envelope, and audit sink. Candidate packs may propose changes within this envelope, but they cannot silently rewrite the evaluator or gate that decides their own adoption, inspect final-gate examples, expand write permissions, change effect policy, suppress audit, or mutate the live pack before atomic promotion. External side effects are governed by the root policy and the effect envelope of Section~\ref{sec:security_artifacts}, not by pack rollback alone.
We write this root envelope as
\begin{equation}
 R_{\mathrm{root}}=\bigl(O_{\mathrm{human}},C_{\mathrm{safety}},V_{\mathrm{final}},H_{\mathrm{sampler}},P_{\mathrm{perm}},P_{\mathrm{approval}},P_{\mathrm{effect}},C_{\mathrm{cred}},\Lambda_{\mathrm{audit}}\bigr),
 \label{eq:root_policy}
\end{equation}
where the fields denote the human objective, mandatory safety constraints, final evaluator/gate versions, final-holdout sampler, permission policy, approval policy, effect policy, credential envelope, and audit sink. The overlay operator is a deterministic partial operator $\sqcup:G\times G\to G\cup\mathsf{Conflict}$: unset fields inherit defaults, type mismatches and unresolved precedence conflicts produce diagnostics, and a candidate may be promoted only when the resulting genome and root envelope are jointly admissible.

\begin{figure}[tb]
 \centering
 \includegraphics[width=0.9\linewidth]{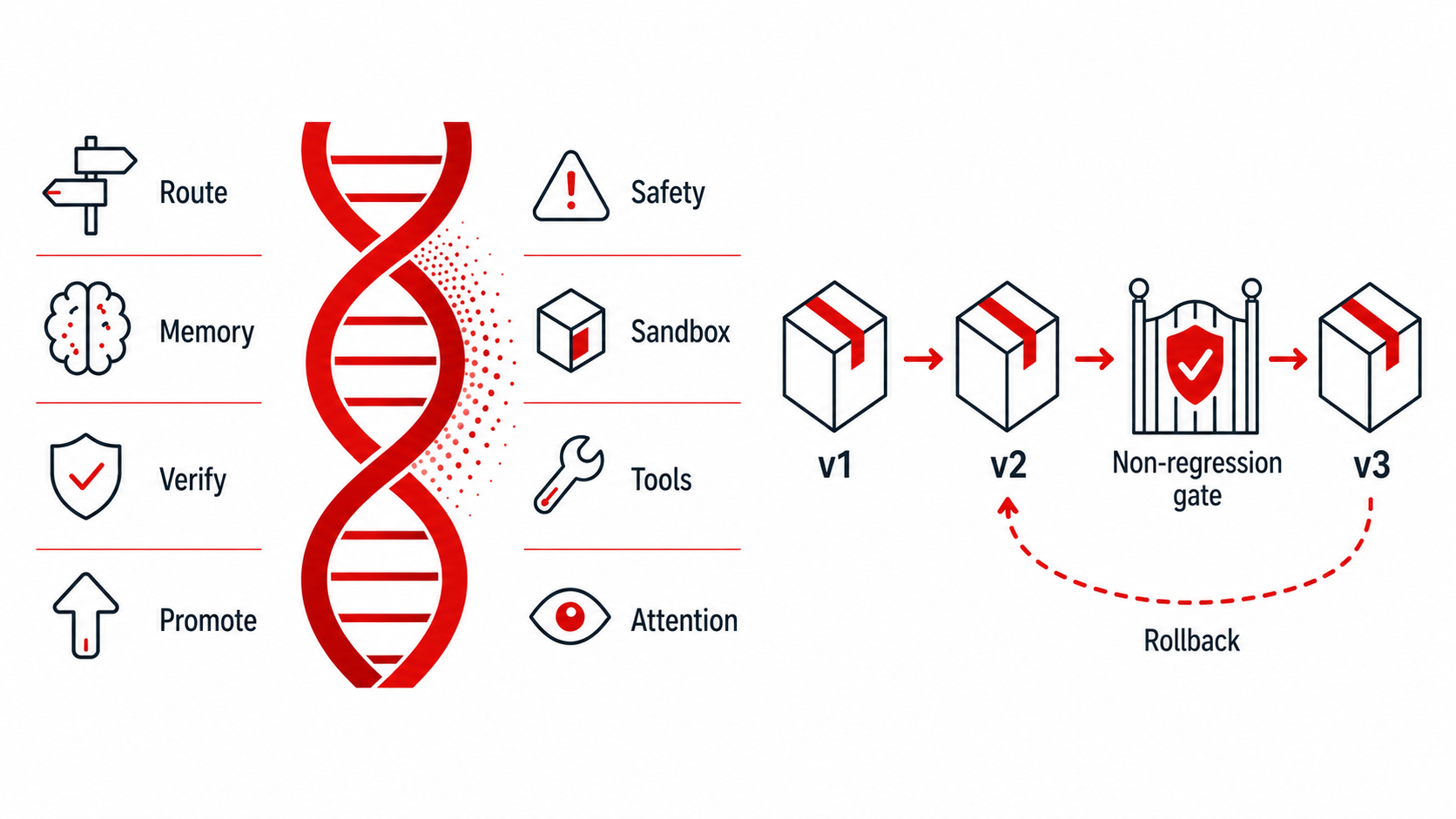}
 \caption{Root policy sits outside the evolvable genome: candidates can be copied, gated, rejected, or atomically promoted, but cannot rewrite the evaluator, permissions, or final holdout.}
 \label{fig:genome}
\end{figure}

\subsection{Adapter-Based Execution Kernel}
\label{sec:kernel}

The kernel executes a pack only on a backend whose declared capabilities satisfy the pack requirements. An adapter maps backend-native events into the normalized trace $\mathcal{Z}$ of Eq.~\eqref{eq:trace} and maps portable artifacts back into backend-native configuration through Eq.~\eqref{eq:tracemap}. This separates the release contract from any one orchestration library. Existing frameworks such as OpenHands and OpenAI Agents illustrate the diversity of execution surfaces that an adapter layer must accommodate~\citep{wang2024openhands,openai2025agents}.

An adapter is accepted only if it preserves the governance contract. Its plugin
contract has six obligations: \emph{capability discovery} declares supported
models, tools, event kinds, effect classes, credentials, and verifier bindings;
\emph{event normalization} maps backend events into $\mathcal{Z}$ without
dropping required task, route, tool, verdict, human, effect, cost, risk, and
audit records; \emph{artifact emission} translates portable packs, skills, and
workflow candidates into backend-native artifacts or emits a diagnostic;
\emph{effect and credential mapping} enforces least-privilege scopes,
redaction, idempotency, and preview or transactional handles; \emph{verifier
binding} connects backend results to the declared verifier stack; and
\emph{diagnostic fallback} refuses activation when a requirement cannot be
represented. The goal is contract preservation, not bit-identical behavior
across frameworks.

Portability does not imply automatic trust. A prompt, skill, workflow, or memory learned on backend $b$ may be emitted to backend $b'$, but it becomes a new target-side candidate and must pass the target verifier and gate. The kernel also records routing decisions, tool calls, verifier outcomes, cost, risk, and human-control events in a common trace. These records are the evidence consumed by routing, debugging, memory, and the paired gate; backend-internal state is neither assumed portable nor treated as gate evidence.

\subsection{The Lifecycle Interface}
\label{sec:interface}

The lifecycle is exposed as four artifact-level transitions. \emph{Compile}
turns bounded material and a brief into a probe-checked or fail-closed pack.
\emph{Operate} executes that pack while emitting normalized evidence and
accepting governed human interventions. \emph{Propose} converts experience into
isolated candidate edits and is opt-in. \emph{Promote} is the only transition
that can replace a live artifact, and it runs under the root policy, gate, and
approval rules rather than under the proposing model. The same semantics can be
exposed through ordinary service interfaces or interoperable agent protocols
such as Agent2Agent~\citep{google2025a2a}; transport does not change the
authority or promotion rules.

This interface also determines what an operator can safely automate. Compile
and operate may be fully automated when their probes, tools, credentials, and
effect envelope allow it. Propose may be frequent and model-driven because
candidates are isolated. Promotion is intentionally constrained: it is where held-out
evidence, safety non-regression, compatibility, and human-owned authority meet.
That separation is the practical reason \logos can support learning without
letting the learning loop rewrite its own objective, evaluator, or permission
envelope.

\subsection{Compiler: From Multimodal Operational Inputs to a Probe-Checked or Fail-Closed Agent Pack}
\label{sec:compiler}

\begin{figure}[tb]
 \centering
 \includegraphics[width=0.9\linewidth]{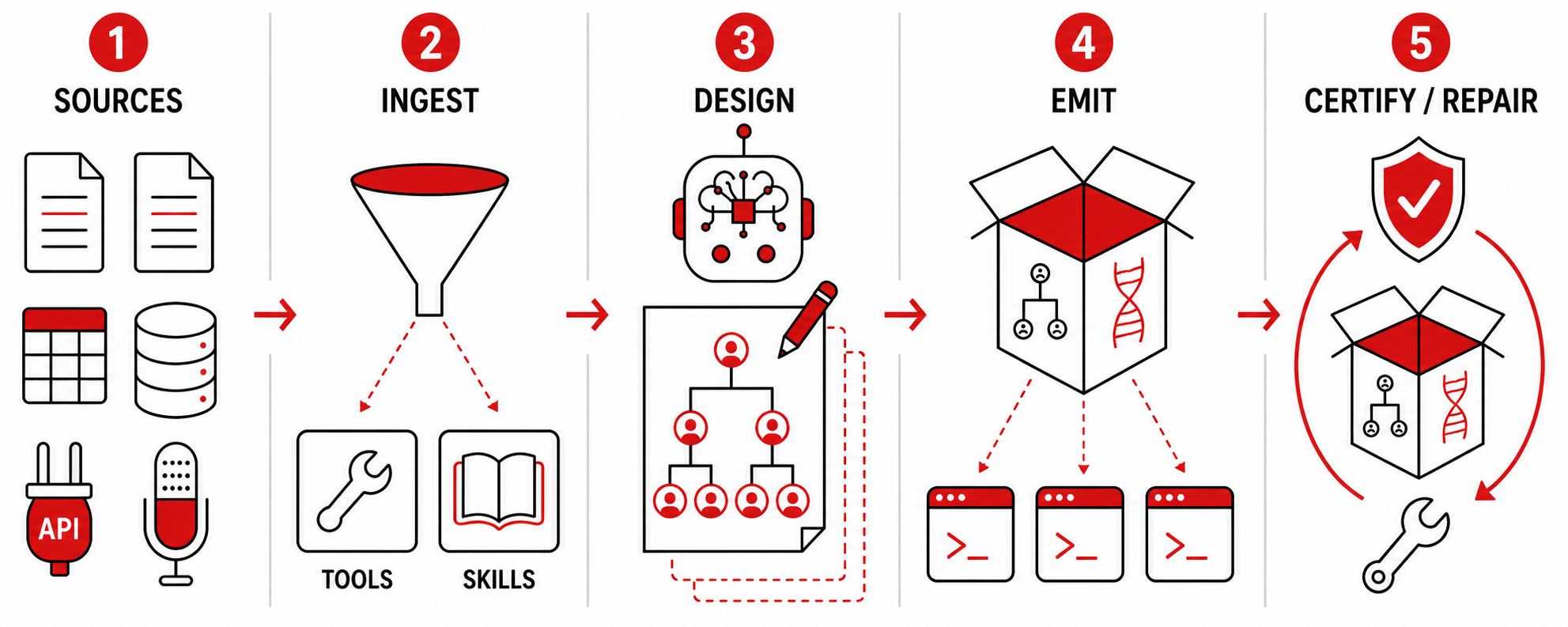}
\caption{Compiler pipeline from heterogeneous multimodal sources to a probe-checked or fail-closed \pack, with fallback synthesis and finite validate-and-repair checks.}
 \label{fig:compiler}
\end{figure}

The first stage of the \logos lifecycle turns bounded heterogeneous operational inputs into an executable team. Inputs may include documents, images, audio transcripts, tables, databases, APIs, logs, and human instructions; all are compiled into typed source specifications with provenance rather than appended as undifferentiated context. We formalize this as a \emph{compilation map}
\begin{equation}
 \mathrm{build}:(\mathcal{S}, q)\longmapsto\mathsf{BuildResult}\in\{\textsc{ProbeValidatedPack},\textsc{DiagnosticPack}\},
 \label{eq:compile}
\end{equation}
from a bounded source set $\mathcal{S}$, which may be empty, and a natural-language or spoken objective $q$, the brief of Section~\ref{sec:bg_mas}. A \textsc{ProbeValidatedPack} result contains the deployable pack $(\Pi_0,G,\mathcal{V},\mathcal{C},\mathcal{M}_{\mathrm{man}})$ of Eq.~\eqref{eq:genome}; the type name means only that the declared finite validation probes passed. A diagnostic result contains repair instructions and no live-tool permission. Figure~\ref{fig:compiler} shows the pipeline. We require Eq.~\eqref{eq:compile} to be fail-closed as a typed interface, not universally successful as a compiler: it returns either \textsc{ProbeValidatedPack} with stated finite-probe evidence or \textsc{DiagnosticPack} with no execution authority. The stages below define the operational contract.

\subsubsection{Ingestion: Sources to Typed Specifications}
\label{sec:ingest}

Each source $\sigma \in \mathcal{S}$ is routed by file type to an adapter that emits typed specifications and provenance-bearing data artifacts. We distinguish \emph{tools}, typed, single-call capabilities, from \emph{skills}, procedural guidance documents, following a deliberate two-layer separation; source content that is not itself a tool or skill remains available as a bounded knowledge or data artifact with provenance. Formally, an adapter is a map
\begin{equation}
 a_{\text{type}(\sigma)}:\sigma\longmapsto\bigl(\mathcal{W}(\sigma), \mathcal{K}(\sigma), \mathcal{X}_{\mathrm{src}}(\sigma), \mathcal{P}_{\mathrm{src}}(\sigma)\bigr),
 \label{eq:ingest}
\end{equation}
where $\mathcal{W}(\sigma)$ is a set of tool specifications, $\mathcal{K}(\sigma)$ a set of skill specifications, $\mathcal{X}_{\mathrm{src}}(\sigma)$ typed knowledge or data artifacts, and $\mathcal{P}_{\mathrm{src}}(\sigma)$ provenance records. Each tool specification carries an error contract
\begin{equation}
 w = \bigl(\text{name}, \text{desc}, \mathcal{I}_w, \mathcal{O}_w, \mathcal{E}_w, \rho_w,
 \text{effect}, \text{auth}, \text{network}, \text{rate}, \text{idempotency}, \text{binding}\bigr),
 \label{eq:toolspec}
\end{equation}
with input/output/error JSON schemas $\mathcal{I}_w,\mathcal{O}_w,\mathcal{E}_w$,
recovery guidance $\rho_w$, effect class such as read-only, transactional write,
destructive, or external side effect, authentication scope, network policy, rate
limit, and idempotency key policy. Structured sources compile to typed callable
bindings rather than undocumented stubs: OpenAPI becomes a schema-constrained
REST binding, tables become typed readers, databases become read-only SQL tools
guarded by an allow-list, and protocol descriptors become typed remote tools.
The adapter family covers ordinary operational material, text and office
documents, PDFs and images, tabular data, media, APIs, databases, and MCP, so
heterogeneous inputs remain admissible without changing Eq.~\eqref{eq:ingest}. Skills
use \emph{progressive disclosure}: a short playbook plus on-demand detail files,
keeping context cost bounded. Spoken briefs are transcribed and merged into the
same objective $q$.

\subsubsection{LLM-First Synthesis under a Fallback Contract}
\label{sec:synthesis}

Given $(\mathcal{W},\mathcal{K},\mathcal{X}_{\mathrm{src}},\mathcal{P}_{\mathrm{src}},q)$, a designer synthesizes a \emph{blueprint}, an intermediate representation of the team,
\begin{equation}
 B = \bigl(\{r_j\}_{j=1}^{m}, \mathcal{K}_B, O, \{(\mu_\ell,\kappa_\ell)\}_\ell,
 \mathcal{A}_B,\mathcal{V}_B,\mathcal{C}_B,\mathcal{I}_B,\mathcal{E}_B,h\bigr),
 \label{eq:blueprint}
\end{equation}
where each role $r_j$ records a name, instruction, skill ids, tool ids, permission scope, and model policy; $\mathcal{K}_B$ are skill drafts; $O$ is an orchestration plan with a pattern, maximum rounds, and parallelism drawn from a fixed pattern set $\mathcal{O}$; $(\mu_\ell,\kappa_\ell)$ are cost-aware model tiers pairing a model $\mu_\ell$ with a complexity bucket $\kappa_\ell$; $\mathcal{A}_B$ assigns tools and source artifacts to roles; $\mathcal{V}_B$ and $\mathcal{C}_B$ are verifier and validation-probe declarations; $\mathcal{I}_B$ records input/output contracts; $\mathcal{E}_B$ records the effect boundary; and $h\in\{0,1\}$ is a \emph{human-input} flag indicating whether the running team must be able to question the human. The designer is LLM-first but has a deterministic fail-closed fallback:
\begin{equation}
 \mathrm{design}(\mathcal{W},\mathcal{K},\mathcal{X}_{\mathrm{src}},\mathcal{P}_{\mathrm{src}},q) =
 \begin{cases}
 \text{team}(\cdot), & \text{if the LLM returns schema-valid JSON,}\\
 \text{playbook}(\cdot), & \text{else if the LLM returns prose,}\\
 \text{deterministic}(\cdot), & \text{otherwise: no LLM, or failure.}
 \end{cases}
 \label{eq:fallback}
\end{equation}
The modes produce, respectively, a structured team, a single-playbook agent, or a deterministic template. Normalization then enforces two invariants, $|\{r_j\}|\ge 1$ and exactly one canonical task-playbook skill, repairing absent agents/playbooks and discarding malformed tiers. When a model emits invalid JSON, the designer spends at most $r_{\max}$ repair prompts before degrading to playbook mode. Since the deterministic branch needs no external service and normalization always returns a valid blueprint, the composition
\begin{equation}
 \mathrm{build} = \mathrm{validate\text{-}repair}^{(N)}\circ\mathrm{emit}\circ\mathrm{normalize}\circ\mathrm{design}\circ\mathrm{ingest}
 \label{eq:build_compose}
\end{equation}
is \emph{defined on every input}: each stage either succeeds, degrades to a simpler runnable artifact, or emits a diagnostic pack that reports the missing inputs and refuses live execution. The validate-and-repair loop of Section~\ref{sec:certify} runs for at most $N$ rounds. Thus the compiler returns a pack justified by its finite battery, or a fail-closed diagnostic pack when the battery cannot justify execution.

An optional multi-candidate mode proposes several blueprints under different design biases, repairs each structurally, and selects a Pareto knee over quality, cost, latency, and ask-load. If surviving candidates disagree on a structural axis, the compiler asks the most discriminating question by expected information gain; otherwise it records the considered frontier as design provenance.

\paragraph{The design ledger: controlled reuse across builds.}
Compilation is not memoryless. Every completed build appends its brief, chosen blueprint, finite validation score, repair-effort score, cost, and observed failure modes to a persistent \emph{design ledger} $\mathcal{D}=\{(q_i, B_i, p_i)\}_i$. A new brief $q$ retrieves its $k$ nearest precedents by embedding similarity, or by lexical overlap when embeddings are unavailable, and conditions the designer on them:
\begin{equation}
\begin{aligned}
 &\mathrm{design}\bigl(\mathcal{W},\mathcal{K},q; \mathcal{D}_k(q)\bigr),\\
 &\mathcal{D}_k(q) =
 \TopK_{\min(k,|\mathcal{D}|)}
 \Bigl(
 \{(q_i,B_i,p_i)\in\mathcal{D}\},
 \simop\bigl(\mathbf{e}(q),\mathbf{e}(q_i)\bigr)
 \Bigr),
\end{aligned}
 \label{eq:ledger}
\end{equation}
Here $\TopK_k(S,f)$ returns the $k$ elements of $S$ with largest score under $f$, and $\mathbf{e}(q)$ is the brief embedding; ties are broken by deterministic artifact id and timestamp. For an empty ledger, $\mathcal{D}_k(q)=\varnothing$. Thus the first build still falls back cleanly, while later builds can reuse design priors that have already passed execution checks. Together with gated skill porting and resident World lessons, the ledger supplies candidates for a better starting point in the next deployment; target-side validation decides whether those candidates are actually useful.

\subsubsection{Emission to Genome and Backends}
\label{sec:emit}

A blueprint is emitted to the pack fields of Eq.~\eqref{eq:genome}: skills and agent definitions populate $\Pi_0$, verifier and probe declarations populate $\mathcal{V}$ and $\mathcal{C}$, provenance and adapter requirements populate $\mathcal{M}_{\mathrm{man}}$, and the orchestration plan plus human-input flag populate the genome. In particular, the human-input flag sets the tool policy
\begin{equation}
 h=1 \Longrightarrow \bigl(\text{human-input tool enabled in }G_{\text{tools}}\bigr)\wedge\bigl(\text{ask policy}=\text{budgeted heuristic}\bigr),
 \label{eq:emit_human}
\end{equation}
so that a team judged to need human input has a governed way to ask for it under the budgeted policy of Section~\ref{sec:ask}. If labeled calibration data exist for the target model and task distribution, this default policy can be replaced by a calibrated ask policy; otherwise it is treated as an uncalibrated heuristic. The same pack is projected onto each compatible backend $b$ through the emitter $\epsilon_b$ of Eq.~\eqref{eq:tracemap}. The emitter produces backend-compatible artifacts together with a compatibility and conflict report; target-side validation remains mandatory before activation.

\subsubsection{Finite Validation as the Objective, and Self-Evolving Repair}
\label{sec:certify}

Compilation success is not sufficient as a validation criterion. Following the API-Bank evaluation philosophy~\citep{li2023apibank}, \logos validates a pack by whether an agent can actually invoke its capabilities on a finite battery of probes $\mathcal{C}$. Table~\ref{tab:compile_checks} gives the core compile checks and the broader operational-pack stress axes.

\begin{table}[tb]
\centering\small
\setlength{\tabcolsep}{5pt}\renewcommand{\arraystretch}{1.1}
\caption{Compile checks and operational-pack stress axes. L0--L2 are cumulative compiler checks; C3--C5 are independent evaluation axes.}
\label{tab:compile_checks}
\begin{tabular}{@{}ll@{}}
\toprule
Level & Meaning \\
\midrule
L0 Parse & manifest and schemas are valid \\
L1 Instantiate & backend artifacts can be generated \\
L2 Invoke & agents call declared tools with valid parameters \\
C3 Complete & brief-derived end-to-end tasks are completed \\
C4 Generalize & held-out documents, APIs, or domains are handled \\
C5 Operate safely & injection, permission, and external-effect tests pass \\
\bottomrule
\end{tabular}
\end{table}
The Q1 mechanism checks report finite L0--L2 callability under the declared probe battery. C3--C5 are separate stress axes rather than cumulative certification levels: C5 can exceed C4 because safety fit and held-out workflow generalization are measured on different generated labels. The measured pass fraction is
\begin{equation}
 p(\text{pack}) = \nicefrac{1}{|\mathcal{C}|}\sum_{c \in \mathcal{C}} \indic{c~\text{succeeds on pack}} \in [0,1].
 \label{eq:certify}
\end{equation}
This $p(\text{pack})$ is defined only for a nonempty probe battery; if $\mathcal{C}=\varnothing$, the compiler emits an uncertified diagnostic pack rather than an activated pack. The value is a finite-battery result: it says that the pack passed the probes in $\mathcal{C}$, not that every future tool call will succeed. When $\mathcal{C}$ is sampled from a deployment-relevant call distribution, the separate anytime-valid process below can attach a statement about the underlying callability rate. When $p$ falls below a target $p^\star$, a repairer maps failures to minimal spec edits, such as adding a missing field, aligning a schema, fixing syntax, tightening an effect annotation, or improving recovery guidance, and the pack is rebuilt. Proposal probes and final validation probes are separated when a learned repairer is used; deterministic structural repairs may reuse the same finite battery because they do not learn semantic fixes from per-example failures. This defines the validate-and-repair iteration
\begin{equation}
 \text{pack}_{k+1}=\mathrm{repair}\bigl(\text{pack}_k,\mathcal{F}_k\bigr),\quad \mathcal{F}_k=\{c\in\mathcal{C}:c~\text{fails on pack}_k\},
 \label{eq:autofix}
\end{equation}
run for at most $N$ rounds or until $p(\text{pack}_k)\ge p^\star$. Each repair only adds or tightens a specification, never deleting a capability without an explicit diagnostic, but this does not make the pass rate monotone: a tighter schema or narrower permission can expose new failures. The pack is therefore re-run through the full finite battery after every round and the best-so-far diagnostic is retained. The loop terminates at a pack that passes the stated probes, at that diagnostic when the round budget is exhausted, or at a point where no failure maps to an actionable edit. This is the compiler-level instance of the paper's governing rule: promote a repair only when execution evidence supports it under the stated probes.

\paragraph{Anytime-valid rate evidence.}
The finite pass fraction of Eq.~\eqref{eq:certify} carries no confidence about unobserved calls: a single successful probe can overstate a tool's true callability, and repeated probing is unsafe if one stops after a favorable interim result. The compiler can instead probe sequentially and maintain two one-sided anytime-valid e-processes, using the portfolio construction of Eq.~\eqref{eq:portfolio} in Appendix~\ref{sec:method_ext}, for the true callability rate $p$ against a target band with margin $\xi_{\mathrm{ind}}>0$:
\begin{equation}
 \text{verdict} =
 \begin{cases}
 \textsc{conflict}, & K_t^{\uparrow}\ge 1/\alpha \wedge K_t^{\downarrow}\ge 1/\alpha,\\
 \textsc{supported}, & K_t^{\uparrow}\ge 1/\alpha,\\
 \textsc{refuted}, & K_t^{\downarrow}\ge 1/\alpha,\\
 \textsc{undecided}, & \text{probe budget exhausted,}
 \end{cases}
 \label{eq:pact}
\end{equation}
where $Y_t\in\{0,1\}$ is the binary success indicator for probe $t$, $K_t^{\uparrow}$ tests $H_0{:}p\le p^\star$, and $K_t^{\downarrow}$ tests $H_0{:}p\ge p^\star-\xi_{\mathrm{ind}}$ by applying the same one-sided portfolio process to $1-Y_t$; Appendix~\ref{sec:sage_lord_formal} states the wealth constraints. The band $(p^\star-\xi_{\mathrm{ind}},p^\star]$ is an indifference zone: the compiler neither supports nor refutes the target unless the probe budget supplies stronger evidence. A \textsc{conflict} verdict is diagnostic and requires human review or a fresh probe family. If support and refutation errors are reported separately, each stream runs at level $\alpha$; if one family-wise budget $\alpha_{\mathrm{tot}}$ is desired for the verdict family, the two streams run at $\alpha_{\mathrm{tot}}/2$. Because both tests are valid at any stopping time under their sampling contract, the process can support clear wins and refute clear failures early, saving calls without the error of peeking at an ordinary fraction.

\subsection{Routing and Verification Engines}
\label{sec:routing_verif}

Verifier quality is the safety-critical dependency shared by routing, compilation, and self-evolution. A false accept approves a wrong output and is the dangerous error; a false reject raises costs by escalating or retrying; a coverage failure means that no applicable acceptance validator exists; and verifier drift changes the meaning of historical pass rates. \logos therefore separates hard constraints from acceptance validators, records verifier provenance, and fails closed when a validator is missing, crashes, or becomes inapplicable.

\begin{figure}[tb]
 \centering
 \includegraphics[width=0.85\linewidth]{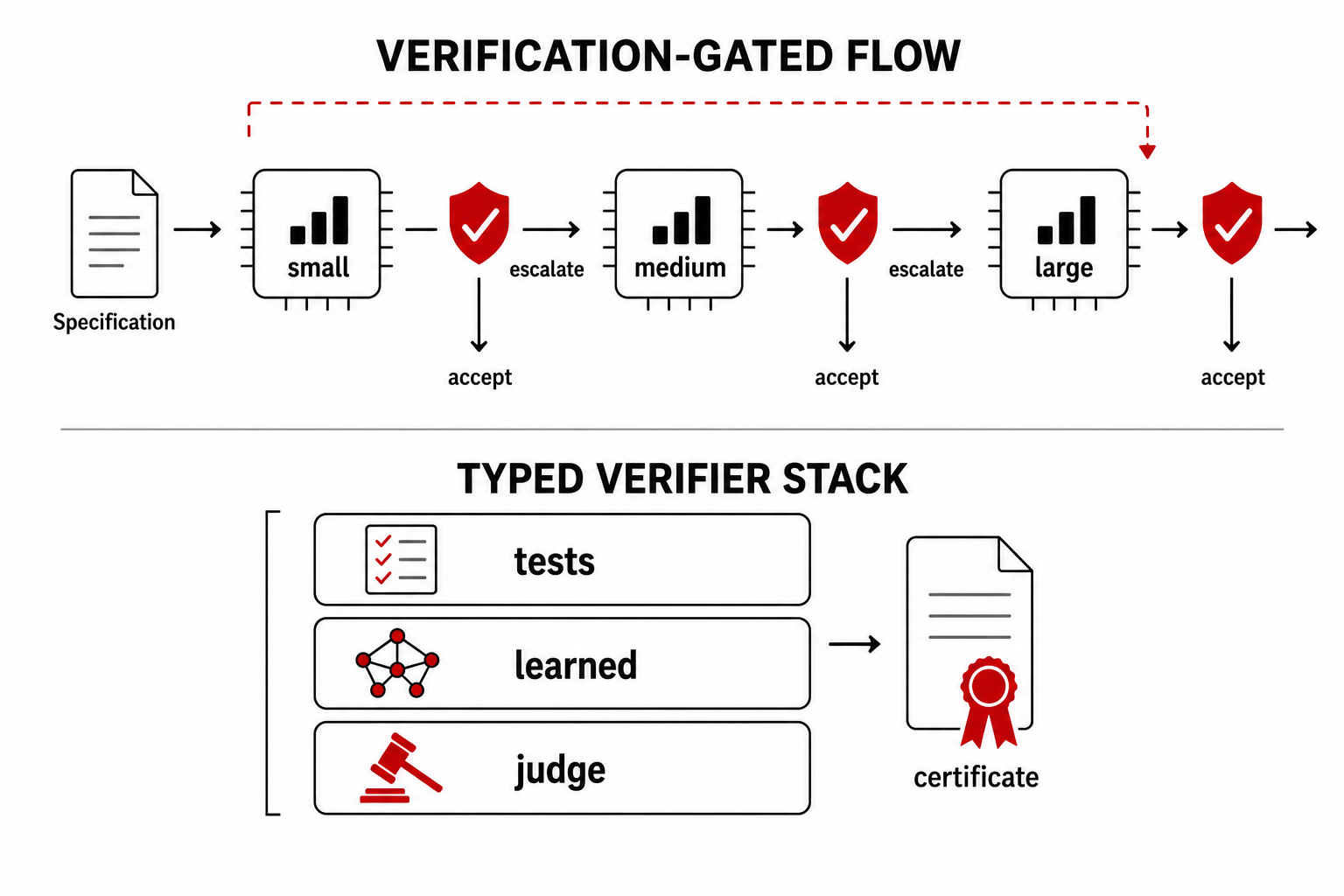}
 \caption{Routing and verification engines. The cascade escalates on verifier rejection; the verifier stack separates hard constraints from acceptance validators.}
 \label{fig:routing_verif}
\end{figure}

A compiled pack answers two operational questions on every query: which resource should act, governed by $G_{\text{route}}$, and whether its output may be trusted, governed by $G_{\text{verify}}$. Routing is dispatch under a budget, and verification is the acceptance test that decides whether the dispatch succeeded. Figure~\ref{fig:routing_verif} summarizes this coupling. Routing instantiates the cost--quality objective of Eq.~\eqref{eq:bg_route_obj}; verification instantiates the trust primitive of Eq.~\eqref{eq:bg_verifier}. They are coupled because the verification-gated cascade uses the verifier as its escalation signal, and the same verifier later gates self-evolution.

\subsubsection{Routing Engines}
\label{sec:routing}

\logos instantiates the routing policy $\rho$ of Eq.~\eqref{eq:bg_router} in three regimes of increasing closed-loop adaptivity, all optimizing the cost-regularized utility $U(\tau,a)=r(\tau,a)-\lambda_{\mathrm{cost}}c(\tau,a)$ of Eq.~\eqref{eq:bg_route_obj}.

\paragraph{Static cost routing.}
The simplest router is \emph{pre-answer} and deterministic. A declared classifier $g:\mathcal{X}\to\mathcal{K}_{\mathrm{cpx}}$ maps a query to a coarse complexity class $\mathcal{K}_{\mathrm{cpx}}=\{\textsc{simple},\textsc{complex},\textsc{design}\}$, and a fixed tier table assigns each class a model,
\begin{equation}
 \rho_{\mathrm{cost}}(x)=\delta_{\mu(g(x))},\quad \mu:\mathcal{K}_{\mathrm{cpx}}\to\{\text{model tiers}\}.
 \label{eq:costrouter}
\end{equation}
where $\delta_a$ is the point mass on action $a$, and the tiers $\{(\mu_\ell,\kappa_\ell)\}$ are exactly those emitted by the compiler's blueprint of Eq.~\eqref{eq:blueprint}. This commits the cheap model to easy queries and reserves the expensive model for harder queries, but it cannot improve from feedback: the assignment $\mu$ is static, and a misclassification is never corrected.

\paragraph{Learned routing by cost-aware policy gradient.}
To adapt the assignment from experience, \logos maintains a stochastic policy over the routing actions $\mathcal{A}_{\mathrm{route}}=\{a_1,\ldots,a_K\}$ within each complexity bucket $x$, parameterized by logits $\ell_i$ through a softmax $p_i = \softmax(\ell)_i$. After a routed run yields reward $r$ at cost $c$, the realized utility $U=r-\lambda_{\mathrm{cost}}\widehat{c}$, where $\widehat{c}$ is normalized cost, updates the chosen action's probability by the REINFORCE score-function rule~\citep{williams1992reinforce}
\begin{equation}
 \ell_i\leftarrow \ell_i+\eta_{\mathrm{lr}}A\bigl(\indic{i=i_{\mathrm{chosen}}}-p_i\bigr),\quad A=U-b(x).
 \label{eq:onlinerouter}
\end{equation}
where $\eta_{\mathrm{lr}}>0$ is a step size, with the subscript reserving $\eta$ for memory episodes in Section~\ref{sec:memory}, and $b(x)$ is a lagged per-bucket running mean utility. The baseline is updated after the policy step, so it is independent of the sampled action for this gradient estimate. Subtracting it preserves unbiasedness while reducing variance, which is crucial on the small reward streams of one deployment; the learned policy is persisted across runs.

\paragraph{Verification-gated cascade with experience.}
\label{sec:cascade}
The closed-loop router does not commit to one action before observing evidence: it treats the action set as a ladder ordered cheap-to-strong, $a_1\prec a_2\prec\cdots\prec a_K$, where each $a_k$ is a model on a possibly distinct provider, and escalates only when a verifier $\nu$, defined in Section~\ref{sec:verification}, rejects the current stage's output. Two ingredients make this adaptive rather than a fixed waterfall.

\paragraph{Experience memory.}
The router classifies each query into a fine-grained \emph{dimension} $d=\mathrm{dim}(x)$, such as \textsc{bug\_fix}, \textsc{refactor}, \textsc{long\_context}, or \textsc{math}. For each pair $(d,k)$ it maintains a Beta--Bernoulli posterior on the verifier-pass probability $\theta_{d,k}$ of stage $k$ on dimension $d$,
\begin{equation}
 \theta_{d,k}\sim\Beta(\alpha_{d,k},\beta_{d,k}),\quad (\alpha_{d,k},\beta_{d,k})\leftarrow(\alpha_{d,k}+\indic{\text{pass}}, \beta_{d,k}+\indic{\text{fail}}).
 \label{eq:cascade_post}
\end{equation}
together with an exponential-moving-average cost. The posterior tracks pass events from the configured verifier, not latent task correctness. It is meaningful when the dimension is stable enough for past pass events to predict future pass events; deployments that need continual discovery add an exploration floor or forced probes for rarely attempted stages. The optimistic verifier-pass score $\UCB(d,k)$ is the posterior mean plus $c_{\mathrm{ucb}}$ posterior standard deviations, with exploration constant $c_{\mathrm{ucb}}>0$; cold-start behavior is determined by the prior parameters and the acceptance threshold rather than guaranteed exploration of every cheap stage. The exact expression is Eq.~\eqref{eq:cascade_ucb}.

\paragraph{Entry selection and escalation.}
The cascade does not always start at stage $1$. It enters at the cheapest stage whose optimistic verifier-pass estimate clears an acceptance threshold $T_{\mathrm{acc}}$, and otherwise enters directly at the strongest stage $K$,
\begin{equation}
 s(d)=
 \begin{cases}
 \min\{k:\UCB(d,k)\ge T_{\mathrm{acc}}\}, & \text{if such } k \text{ exists},\\
 K, & \text{otherwise}.
 \end{cases}
 \label{eq:cascade_entry}
\end{equation}
so experience can skip cheap stages that rarely pass verification on a dimension. The entry rule is a verifier-pass heuristic; a deployment may additionally require the optimistic utility gain over starting at a later stage to be positive before attempting a cheap stage. The cascade then runs $k=s(d),s(d)+1,\ldots$, accepts the first output that passes $\nu$, and updates every attempted stage.

\paragraph{Observation: conditional cascade non-degradation.}
\label{prop:cascade}
Assume that every early output accepted by $\nu$ is correct for the task family; early stages are state-isolated, do not consume the fallback's context or budget, and trigger a fresh fallback execution with the same distribution as policy $a_K$ if all earlier stages reject; and expected costs are evaluated on the same task distribution. Then cascade accuracy is at least the accuracy of always using $a_K$. If the expected-cost condition in Appendix~\ref{sec:bg_ext} also holds, expected cost is lower than always running $a_K$. With an imperfect verifier, the accuracy loss is bounded by the reach-conditioned false-accept mass, not by an unconditional verifier FPR. HumanEval in Section~\ref{sec:exp_cascade} adds a second empirical effect beyond fallback non-degradation: multiple candidates plus an execution verifier can select a correct early output that a single strongest rollout missed. The cascade reduces to static routing when $K=1$ and to an always-strong policy when no stage clears $T_{\mathrm{acc}}$.

\subsubsection{Collective Routing}
\label{sec:chorus}

The collective router treats a heterogeneous worker pool as a single governed service. In a \emph{fast} path it selects one worker; in a \emph{deep} path it constructs a short multi-step workflow. The default decision to deepen is an inspectable routing-margin trigger; deployments may add an expected-utility guard that compares predicted quality gain, additional cost, and verifier availability. Every selected worker, intermediate result, verifier verdict, and escalation remains visible in the normalized trace. When no applicable verifier or reward stream exists, the router can still orchestrate work, but the result is reported as an unevaluated mechanism outcome rather than an accuracy improvement.

\paragraph{The selection head and its three learning signals.}
\label{sec:chorus_head}
Let $\mathcal{W}_{\mathrm{pool}}=\{w_1,\ldots,w_n\}$ be the worker pool, with relative costs $c_j>0$ and optional specialization tags. Fast mode uses a stochastic policy over workers conditioned on an inspectable context key $\kappa(x)$, implemented as a per-context logit table $\mathbf{z}_\kappa\in\reals^{n}$ or an optional linear head over a frozen encoder:
\begin{equation}
 \pi(w_j\mid x) = \softmax\bigl(\mathbf{z}_{\kappa(x)}/T_{\mathrm{soft}}\bigr)_j,
 \label{eq:chorus_policy}
\end{equation}
with shared temperature $T_{\mathrm{soft}}>0$. The head only chooses; it never generates text. Three learning signals write to the same persisted head, so offline pre-training and online deployment learning compose.

\emph{Soft-label distillation.}
Offline, every worker is rolled out on each training query and scored by a verifier, yielding mean rewards $\bar r_j$. Let $\epsilon_q$ be a configured quality tolerance, $\bar r_{\max}=\max_k\bar r_k$, $\tilde c_j=c_j/\max_k c_k$, and $\hat r_j=\bar r_j-\lambda_{\mathrm{cost}}\tilde c_j$. The head is distilled toward a \emph{quality-constrained cost-adjusted soft target}
\begin{equation}
\begin{aligned}
 p^\star_j
 &= \softmax\bigl(\tilde r_j/T_{\mathrm{soft}}\bigr)_j,\\
 \tilde r_j
 &=
 \begin{cases}
  \hat r_j,
  & \bar r_j \ge \bar r_{\max}-\epsilon_q,\\
  \min\{\hat r_j,
       \min_{\ell:\bar r_\ell\ge \bar r_{\max}-\epsilon_q}\hat r_\ell
       -(\bar r_{\max}-\epsilon_q-\bar r_j)\},
  & \text{otherwise.}
 \end{cases}
\end{aligned}
 \label{eq:chorus_soft}
\end{equation}
by minimizing the cross-entropy $H(p^\star,\pi)$. Thus cost can select the cheaper worker among near-equivalent options without allowing a clearly weaker worker to be selected on price alone.

\emph{Evolutionary refinement.}
The distilled head is then refined by a separable evolution strategy against the same quality-constrained replay utility. Equivalently, $V(\pi)=\widehat{\mathbb{E}}[\tilde r_{a_\pi(x)}(x)]$, where $a_\pi(x)=\argmax_j\pi(w_j\mid x)$ is the worker chosen by the deterministic replay policy and $\tilde r$ is the reward table after near-frontier cost adjustment. The expectations are computed by fixed replay of the recorded per-query, per-worker reward table, so candidate policies cost no additional LLM calls.

\emph{Online bandit updates.}
In deployment, each observed outcome applies the same variance-reduced score-function update as Eq.~\eqref{eq:onlinerouter}, with a per-update clip and the same quality-frontier guard used offline. Because ordinary deployment observes rewards only for selected workers, the online frontier is a heuristic unless the policy includes exploration, inverse-propensity correction, or doubly robust evaluation. If online evidence shows that a worker is outside the configured tolerance from the best observed reward in its context, cost cannot promote that worker by itself. The exact guarded update is Eq.~\eqref{eq:chorus_online}.

\paragraph{Routing-margin-gated deepening and the conductor.}
\label{sec:chorus_deep}
Complexity buckets alone decide fast-versus-deep coarsely. The router additionally consults the learned head's own routing margin: writing $\pi_{(1)}\ge\pi_{(2)}$ for the two largest selection probabilities on query $x$, the query is escalated to the deep mode iff the decision margin falls below a fixed threshold $\varepsilon_{u}>0$,
\begin{equation}
 m(x) = \pi_{(1)}(x)-\pi_{(2)}(x) < \varepsilon_{u},
 \label{eq:chorus_margin}
\end{equation}
so, under the default policy, workflow planning is used when single-worker routing is ambiguous. The margin is an operational confidence signal, not a calibrated correctness probability. In deep mode, a conductor with a deterministic fallback emits a DAG $\mathcal{G}_{\mathrm{wf}}=(V,E)$ of at most $S_{\max}$ steps, each step specifying a subtask, worker, and access list $A_v\subseteq V$. Workflow quality is scored by the ternary shaped reward
\begin{equation}
 R(\mathcal{G}_{\mathrm{wf}}) =
 \begin{cases}
 0, & \mathcal{G}_{\mathrm{wf}}~\text{unparseable: format failure},\\
 \nicefrac{1}{2}, & \mathcal{G}_{\mathrm{wf}}~\text{well-formed, final answer wrong},\\
 1, & \mathcal{G}_{\mathrm{wf}}~\text{well-formed, final answer correct},
 \end{cases}
 \label{eq:chorus_reward}
\end{equation}
whose partial credit for valid-but-wrong plans gives learning a smoother signal and lets sampled workflows be ranked by the verifier stack.

\paragraph{Replica diversity.}
\label{sec:chorus_replica}
A routing head derives limited benefit from identical workers, but diversified reasoning plus agreement can help. When the pool contains $N\ge2$ replicas of one model, the router runs each replica under a distinct deterministic directive from Appendix~\ref{sec:exp_appendix} and aggregates by majority vote when possible, else by the \emph{medoid}
\begin{equation}
 \hat{y} = \argmax_{i\in[N]} \nicefrac{1}{(N-1)}\sum_{j\neq i} \simop\bigl(y_i,y_j\bigr),
 \label{eq:chorus_medoid}
\end{equation}
the answer of maximal mean pairwise similarity. This is self-consistency~\citep{wang2023selfconsistency} lifted from sampled decoding paths to directed reasoning strategies; deep mode can express the same pattern as isolated attempts joined by an aggregator.

The router has deterministic fallbacks: without learned evidence it uses specialization-and-cost ranking, without an LLM the conductor uses a deterministic planner, and without replicas consensus is skipped. Every decision exposes its context key, provenance, probabilities, and routing margin for audit.

\subsubsection{Verification Engines}
\label{sec:verification}

Routing's escalation signal, the compiler's validation loop, and self-evolution's evidence gate all call the same trust primitive $\nu$ of Eq.~\eqref{eq:bg_verifier}. \logos realizes $\nu$ as a composite over a typed stack declared in $G_{\text{verify}}$, using the hard-constraint and acceptance-validator contract formalized in Eq.~\eqref{eq:verifier_composite}. We write $\mathrm{ok}(\tau,y)$ for task-correctness under the declared scorer or oracle; it is used only to analyze verifier errors, not as an available runtime signal.

\paragraph{The verifier stack.}
A pack's verifier stack is a catalog of verifier instances plus the active hard constraints and acceptance-validator requirements that use them:
\begin{equation}
\begin{aligned}
 G_{\text{verify}} &= \bigl(\mathcal{V}_{\mathrm{hard}},\mathcal{V}_{\mathrm{acc}},C_{\mathrm{hard}},\mathcal{P}_{\mathrm{acc}}\bigr),\\
 C_{\mathrm{hard}} &\subseteq \mathcal{V}_{\mathrm{hard}},\\
 \mathrm{kind}(\mathcal{V}_{\mathrm{hard}}) &\subseteq
 \{\textsc{schema},\textsc{syntax},\textsc{lint},\textsc{security},\textsc{sandbox}\},\\
 \mathcal{P}_{\mathrm{acc}} &=
 \{(S_j,A_j,\mathrm{scope}_j)\}_{j=1}^{J},\\
 S_j &\subseteq\mathcal{V}_{\mathrm{acc}},\\
 \mathcal{V}_{\mathrm{acc}}
 &= \{\textsc{unit tests},\textsc{exact scorer},\textsc{execution},
 \textsc{regression},
 \textsc{learned},\textsc{judge}\}.
\end{aligned}
 \label{eq:verifier_stack}
\end{equation}
Each verifier $v\in\mathcal{V}_{\mathrm{hard}}\cup\mathcal{V}_{\mathrm{acc}}$ declares metadata
\begin{equation}
 d_v=(\text{name},\text{scope},\text{assurance},\text{cost},\text{failure mode},\text{sufficiency predicate},\operatorname{app}_v).
\end{equation}
The assurance class is one of \textsc{deterministic-oracle}, \textsc{execution}, \textsc{trained-with-audit}, \textsc{judge-supported}, or \textsc{heuristic}, and $\operatorname{app}_v(\tau,y)$ states whether the verifier is applicable to the task/output pair. Its return value is a typed verdict in $\{\textsc{pass},\textsc{fail},\textsc{not-applicable},\textsc{error}\}$; only \textsc{pass} can contribute to evidence acceptance. The hard constraints are all necessary: they validate schemas, parse syntax, run linters, screen commands or queries against allow-lists, for example with an AST check that rejects mutating or multi-statement SQL, and enforce sandbox limits. They cannot by themselves accept an answer. Each requirement $(S_j,A_j,\mathrm{scope}_j)$ declares the applicable task scope and an aggregation rule $A_j$ over typed verdicts, such as all-of, any-of, $k$-of-$n$, or precedence with fallback. Deterministic task evidence such as unit tests, exact scorers, and execution checks normally appears in the earliest requirements; held-out regression checks and trained verifiers $\nu_\theta$ with recorded training/calibration provenance~\citep{cobbe2021gsm8k,zhang2024genrm} come next; LLM judges are used only as judge-supported evidence when no stronger evidence applies~\citep{zheng2023llmjudge}. Write $C_{\mathrm{hard}}^{\mathrm{app}}(\tau,y)=\{v\in C_{\mathrm{hard}}:\operatorname{app}_v(\tau,y)\}$, $\mathcal{P}_{\mathrm{acc}}^{\mathrm{app}}(\tau,y)=\{(S_j,A_j):(S_j,A_j,\mathrm{scope}_j)\in\mathcal{P}_{\mathrm{acc}},~\tau\in\mathrm{scope}_j\}$, $\mathcal{R}_j(\tau,y)=\{\nu(\tau,y):\nu\in S_j,\operatorname{app}_\nu(\tau,y)\}$, and $\mathrm{ready}_{\mathrm{acc}}(\tau,y)$ for resolved required scopes with applicable verdicts. The output-level evidence verdict is
\begin{equation}
\begin{aligned}
 B_{\mathrm{output}}(\tau,y) &= B_{\mathrm{hard}}(\tau,y)\wedge B_{\mathrm{acc}}(\tau,y),\\
 B_{\mathrm{hard}}(\tau,y)
 &= \bigwedge_{\nu_k\in C_{\mathrm{hard}}^{\mathrm{app}}(\tau,y)}b(\nu_k(\tau,y)),\\
 B_{\mathrm{acc}}(\tau,y)
 &=
 \begin{cases}
 0, & \mathcal{P}_{\mathrm{acc}}^{\mathrm{app}}(\tau,y)=\varnothing
      ~\text{or}~\neg\mathrm{ready}_{\mathrm{acc}}(\tau,y),\\
 \displaystyle
 \bigwedge_{(S_j,A_j)\in\mathcal{P}_{\mathrm{acc}}^{\mathrm{app}}(\tau,y)}
 A_j\bigl(\mathcal{R}_j(\tau,y)\bigr),
 & \text{otherwise.}
 \end{cases}
\end{aligned}
 \label{eq:verifier_composite}
\end{equation}
The hard-constraint empty conjunction is true, but if no acceptance requirement is sufficient and applicable, the second factor is $0$. Applicability is checked under root-policy metadata; an applicability-check error is an \textsc{error}, hence a rejection. Each requirement is fail-closed: an execution error, missing result, or failed aggregation counts as a rejection, never an implicit pass, which is what makes the verifier usable as the escalation signal of a cascade and the gate of an evolution loop.

Sufficiency is therefore an auditable declaration rather than a guarantee of
semantic completeness. A validator requirement can be marked sufficient only for a named task scope,
output type, data slice, and risk class, with a versioned scorer or test suite,
coverage notes, and an owner. For high-risk changes, write-capable tools,
permission expansion, safety prompts, or verifier edits, root policy can require
multiple validator requirements before a separate approval rule is considered. If a
unit test or learned verifier covers only a subset of the real requirement, the
uncovered part remains residual risk and must be recorded in the gate
evidence record rather than hidden behind the word ``pass.'' Human approval is an
authority decision, not a verifier verdict; it can authorize an evidenced action
or release, but it cannot turn missing semantic evidence into a passed verifier.

Verifier quality is therefore a first-order assumption, not a detail. The relevant failure modes are false accept, where an incorrect or unsafe output passes; false reject, where a usable output is escalated or discarded; coverage failure, where no sufficient validator applies; and verifier drift, where a learned or judge-supported validator changes behavior across model or policy updates. Reported verifier tables must state whether ``false accept'' means acceptance among incorrect candidates, error among accepted candidates, or the joint unsafe-accepted rate. High-risk automatic promotion should not depend on judge-supported or heuristic evidence alone: those validators can triage, rank, or request human review, but release requires deterministic, execution, trained-with-audit, or explicitly approved evidence appropriate to the risk class plus the authority checks of Eq.~\eqref{eq:promotion_contract}. Accepted outputs remain audit targets, and stage value in a cascade is interpreted after accounting for the reach-conditioned false-accept mass admitted by the chosen verifier.

\paragraph{Evidence-supported reasoning.}
\label{sec:evidence_graph}
For free-form tasks where no deterministic check applies, a judge alone is the weakest verifier. \logos strengthens it by requiring \emph{structured evidence}. When $G_{\text{verify}}$ enables the evidence protocol, agents emit typed claims that the kernel assembles into a directed \emph{evidence graph}
\begin{equation}
\begin{aligned}
 \mathcal{E}_{\mathrm{ev}}&=(\mathcal{N},\mathcal{A}_{\mathrm{ev}}),\\
 \mathrm{type}(n)&\in\{\textsc{diagnosis},\textsc{hypothesis},\textsc{confirmation},\textsc{refutation},\textsc{recommendation}\}.
\end{aligned}
 \label{eq:evidence_graph}
\end{equation}
whose nodes $\mathcal{N}$ are claims and whose directed edges $\mathcal{A}_{\mathrm{ev}}$ encode support and contradiction. The kernel computes a configured support score over acyclic support paths, breaks cycles by timestamp and provenance priority, and flags any contradiction path above the declared threshold; a recommendation that rests on a refuted hypothesis is marked unsupported for escalation or rejection by the surrounding verifier policy. The evidence graph is not a proof system and is not, by itself, an acceptance validator: LLM-created claims and edges can themselves be wrong. It is a structured audit object that turns an opaque judge acceptance into acceptance of a non-contradicted chain of grounds, making unsupported or contradicted reasoning easier to detect and replay.

\paragraph{From verdicts to anytime-valid rate evidence.}
A verdict on one output is not a rate statement. The compiler may validate a finite probe battery by Eq.~\eqref{eq:certify}, but a statement about an underlying callability or stage-success rate requires sampled trials. \logos attaches those rate statements with the two-sided e-process construction of Eq.~\eqref{eq:pact}: one e-process testing $H_0{:}p\le p^\star$ to support the target and one testing $H_0{:}p\ge p^\star-\xi_{\mathrm{ind}}$ to refute it, valid at every stopping time by Ville. The two streams can be reported at separate level $\alpha$, or at $\alpha_{\mathrm{tot}}/2$ each when a single family-wise $\alpha_{\mathrm{tot}}$ is desired. The same e-process algebra, applied to per-task gains rather than callability, is the evidence gate of execution-gated self-evolution, to which we now turn.
\subsection{Evidence-Gated Self-Evolution: Memory, Skills, Workflows, and the Adoption Gate}
\label{sec:evolution}

\begin{figure}[tb]
 \centering
 \includegraphics[width=0.85\linewidth]{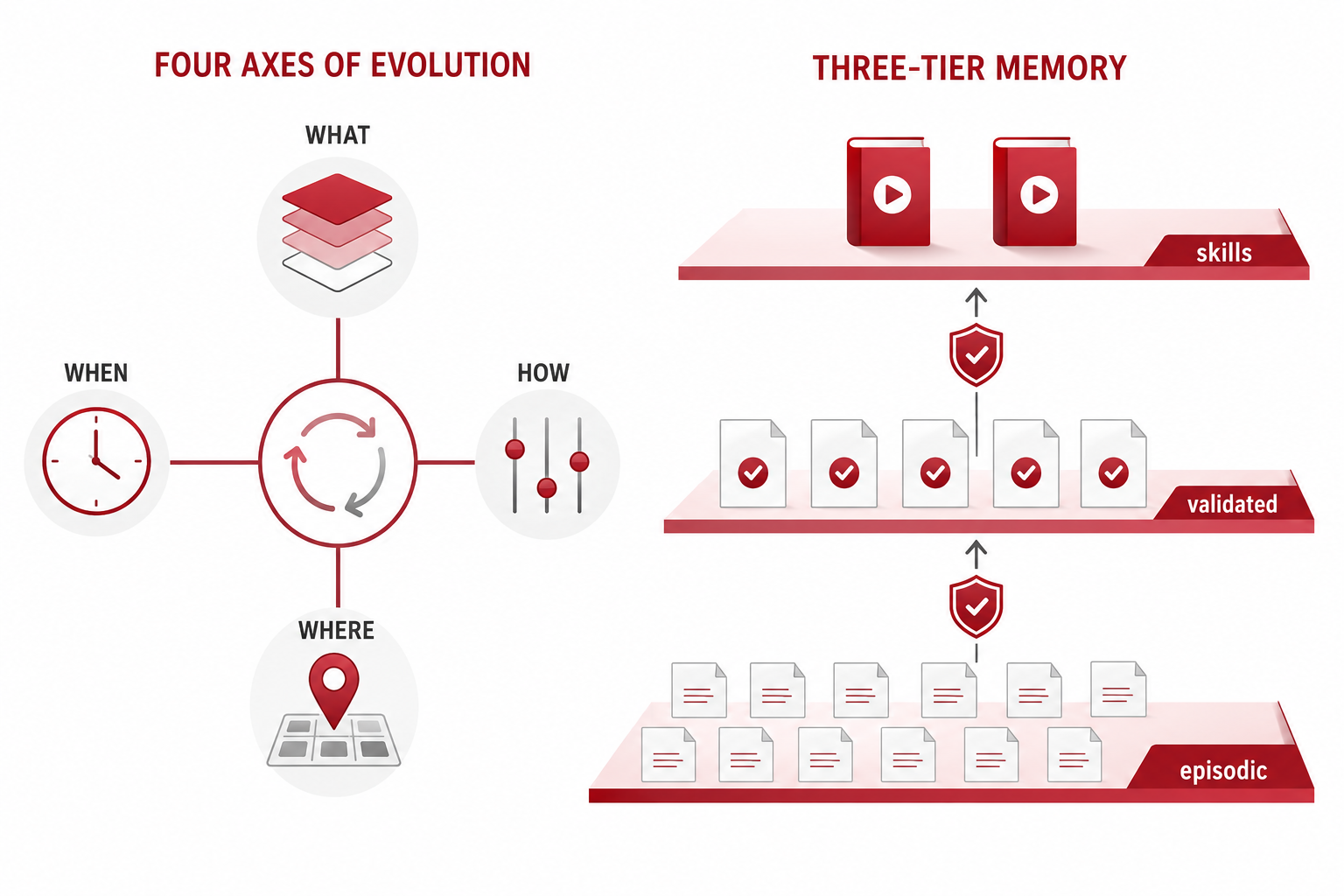}
 \caption{Gated memory promotion. Raw episodes remain quarantined by default, validated lessons are promoted only under stable task evidence, and durable skills retain provenance and remain subject to target-side re-verification.}
 \label{fig:taxonomy}
\end{figure}

The evolution layer improves a deployed pack from its own traces through three hooks: recall before a run, trace capture during the run, and scoring plus candidate evolution after the run. \textbf{Adoption invariant.} Memory, skills, prompts, tools, roles, and workflows may propose changes; none of them may alter the live deployment until a frozen external gate authorizes promotion. More explicitly, a candidate may read proposal data, generate edits, and be evaluated in a sandbox, but it cannot change the live artifact, gate definition, evaluator version, held-out sampler, credentials, or permission envelope. Those are root-policy objects. Figure~\ref{fig:taxonomy} summarizes memory promotion hygiene; this section defines the memory substrate, skill optimizer, workflow search, and the paired execution gate that governs accepted changes.

Terminology in the remainder follows the release lifecycle: a gate \emph{accepts} or \emph{rejects} a candidate; the deployment manager \emph{promotes} an accepted candidate into the live pack; artifact rollback restores an earlier pack/config state; and external effects require confirmation, compensation, or transaction controls rather than assuming rollback is possible. Benchmark tables sometimes use ``adoption'' as shorthand for accepted-and-promoted candidate decisions.

\subsubsection{Gated Memory}
\label{sec:memory}

Experience is stored in three tiers of increasing trust: raw \emph{episodic} memory $\mathcal{M}_{\text{raw}}$, \emph{validated} memory $\mathcal{M}_{\text{val}}$, and durable \emph{skills} $\mathcal{M}_{\text{skill}}$. An episode is a tuple
\begin{equation}
 \eta = \bigl(\text{intent}, \text{experience}, \mathbf{u}, Q, n, \bar{Q}\bigr),
 \label{eq:episode}
\end{equation}
with a semantic embedding $\mathbf{u}$, a learned value $Q$, a visit count $n$, and an exponential moving average $\bar{Q}$. Following memory-as-reinforcement-learning~\citep{gao2025survey}, recall for a query embedding $\mathbf{q}$ has two phases. Phase~A filters by cosine similarity $\simop\in[0,1]$ and keeps the $k_1$ nearest candidates,
\begin{equation}
 \mathcal{A}(\mathbf{q})=\TopK_{k_1}\Bigl(\{\eta:\simop(\mathbf{q},\mathbf{u}_\eta)\ge\delta\},\eta\mapsto\simop(\mathbf{q},\mathbf{u}_\eta)\Bigr).
 \label{eq:phaseA}
\end{equation}
Phase~B re-ranks candidates by a value-aware score that blends standardized similarity and standardized value,
\begin{equation}
 \mathrm{score}(\eta)=(1-\lambda)z\bigl[\simop(\mathbf{q},\mathbf{u}_\eta)\bigr]+\lambda z[Q_\eta],\quad \lambda\in[0,1].
 \label{eq:phaseB}
\end{equation}
where $z[\cdot]$ standardizes within the candidate pool, with $z[x]=0$ when the pool variance is zero. The top $k_2$ are injected only if they clear a confidence threshold. Because recall selects a slate of memories and observes only an aggregate run reward, this is closer to a slate or combinatorial contextual bandit with aggregate feedback than to a single-action bandit. The update below is therefore a credit-assignment heuristic: $Q$ moves toward the realized reward rather than a bootstrapped temporal-difference target,
\begin{equation}
 Q\leftarrow(1-\alpha_{\text{eff}})Q+\alpha_{\text{eff}}r,\quad \alpha_{\text{eff}}=\alpha_{\mathrm{mem}}\cdot\max\bigl(\simop(\mathbf{q},\mathbf{u}_\eta),\varepsilon_{\mathrm{sim}}\bigr).
 \label{eq:qupdate}
\end{equation}
with base learning rate $\alpha_{\mathrm{mem}}\in(0,1]$ and similarity floor $\varepsilon_{\mathrm{sim}}\in[0,1]$, so $\alpha_{\mathrm{eff}}\le1$ and more relevant recalls learn faster. Recall can equivalently add a UCB exploration bonus, and capacity is pruned by a value that blends learned value, empirical success, recency, and usefulness,
\begin{equation}
 V(\eta)=\omega_Q\hat{Q}_\eta+\omega_s\text{succ}_\eta+\omega_r\text{recency}_\eta+\omega_u\text{useful}_\eta,\quad \textstyle\sum_\bullet \omega_\bullet=1.
 \label{eq:prune}
\end{equation}
with $\hat{Q}_\eta,\text{succ}_\eta,\text{recency}_\eta,\text{useful}_\eta\in[0,1]$ normalized before blending and non-negative weights $\omega_\bullet$; Appendix~\ref{sec:exp_appendix} gives the defaults. Promotion is gated, but the default evidence is observational rather than causal: an episode enters $\mathcal{M}_{\text{val}}$ only if its rolling pass rate is high and stable,
\begin{equation}
 \text{pass}(\eta) \ge \theta_{\text{pass}} \wedge |\Delta_{\text{run}}(\eta)| \le \epsilon_{\text{stab}},
 \label{eq:promote}
\end{equation}
where $\text{pass}(\eta)$ is the empirical success rate over runs that recalled it and $\Delta_{\text{run}}(\eta)$ its window-to-window change. Because recall is selected, not randomized, this is a heuristic stability screen and may be biased toward easy tasks. Promotion-sensitive deployments therefore require paired recall/dropout evidence or an explicit target-side gate before writing a durable skill; otherwise the entry remains a quarantined lesson. A durable skill is written only after at least $u_{\min}$ successful uses and, when available, positive paired contribution evidence.

Additional memory mechanics---associative recall, compression, promotion hygiene, per-task working memory, pack-versus-World lesson scopes, and the optional calibrated acquisition rule---are summarized in Appendix~\ref{sec:impl_appendix}. Memory can influence behavior only after validation and provenance checks, reducing but not eliminating persistent prompt-injection, memory-poisoning, and regression risk. The default bypass policy uses no recalled memory when the full source context fits, when retrieval confidence or provenance is low, or when the expected memory cost exceeds the estimated gain.

\subsubsection{Optimizing Skills by Textual Feedback}
\label{sec:skillopt}

Skills are natural-language documents, so \logos optimizes them in text space with a textual-feedback loop in the spirit of TextGrad, OPRO, and reflective prompt evolution~\citep{yuksekgonul2024textgrad,yang2024opro,agrawal2025gepa}. The word ``gradient'' is used only as an analogy for a heuristic feedback direction; no differentiability or descent guarantee is assumed. The \logos-specific constraint is that every edit must pass held-out utility, regression, portability, root-policy, and audit checks. Let $\theta$ be a skill document and $\mathcal{D}_{\mathrm{tr}}=\{(\tau_j, y_j^\star)\}$ a training set with scorer $g$. One step comprises:
\begin{enumerate}[leftmargin=1.6em,itemsep=1pt,topsep=1pt]
 \item \textbf{Rollout.} Execute $\theta$ on a minibatch and score it, $g_j = g\bigl(\Pi_\theta(\tau_j), y_j^\star\bigr)$.
 \item \textbf{Reflect with textual feedback.} An LLM analyzes the failures $\{j: g_j < \vartheta\}$ using score threshold $\vartheta\in(0,1)$ and proposes a patch $P=\{\text{edit}_1,\ldots\}$, where each edit is an anchored append, insert-after, replace, or delete operation. We regard $P$ as a heuristic analogue of $-\nabla_\theta g$ in language space.
 \item \textbf{Aggregate and set the edit budget.} Patches are de-duplicated and clipped to at most $\beta$ edits, $\tilde P = \mathrm{select}(\mathrm{merge}(P), \beta)$; the \emph{edit budget} $\beta$ plays the role of a step size.
 \item \textbf{Apply.} $\theta' = \mathrm{apply}(\theta, \tilde P)$, returning a typed diagnostic and rejecting the candidate if a required edit anchor is absent.
 \item \textbf{Validation gate.} The proposal split is used only to construct candidates. If there is one candidate, a disjoint gate split $\mathcal{H}_{\mathrm{gate}}$ re-runs the current skill and the candidate through the fixed paired gate or the optional anytime-valid gate. If $K>1$ candidates are generated, a separate selection split chooses the candidate, and a final gate split, not used for selection, decides acceptance:
 \begin{equation}
 \mathrm{decision}=\mathrm{Gate}\bigl(\theta,\theta';\mathcal{H}_{\mathrm{gate}},\delta_{\min},R_{\max},\epsilon_{\text{reg}},R_{\mathrm{root}}\bigr)\in\{\textsc{accept},\textsc{reject}\}.
 \label{eq:skillgate}
 \end{equation}
\end{enumerate}
The gate discards any candidate that lacks held-out evidence of gain under the chosen regression regime, fails portability emission, violates the root policy, or cannot produce an auditable patch record; it persists the best-so-far $\theta^\star$ for resumable optimization. In the execution gate, utility evidence must clear a small-sample margin: a one-task improvement on a gate split of ordinary size is treated as too fragile to justify promotion. Once accepted and promoted, a learned skill is injected as guarded secondary guidance: the task prompt, answer format, benchmark policy, and question scope take precedence, and conflicting guidance is ignored rather than allowed to become a silent override. Appendix~\ref{sec:exp_appendix} gives secondary bookkeeping details such as failure attribution, utility tracking, edit-budget schedules, and reject-history buffers. Because $\theta$ is an opaque Markdown artifact, the same rollout--reflect--gate loop optimizes skill playbooks, role definitions, and the typed patch grammar used for directive synthesis.

\paragraph{Adaptive edit budget.}
An optional trust-region controller grows the edit budget when predicted and realized gains agree, shrinks it when prediction overshoots, and penalizes semantic drift. Appendix~\ref{sec:method_ext} gives the update rules; the held-out gate still decides acceptance.

\subsubsection{Workflow Evolution: Cost-Aware Architecture Search}
\label{sec:workflow}

Beyond context and skills, \logos can evolve the workflow $\Gamma$. It supports a population-based \aflow-style search over phase DAGs~\citep{zhang2024aflow} and a query-dependent \maas-style agentic supernet~\citep{zhang2025maas}. The supernet maintains a distribution over operators conditioned on a complexity bucket, so easy queries can use cheap architectures while harder queries can allocate more computation. Writing $p_i$ for the probability of operator $i$, the controller uses a cost-aware policy gradient with utility
\begin{equation}
 U=r-\lambda_{\text{cost}}\widehat{\mathrm{cost}}(i),\quad \text{logit}_i\leftarrow\text{logit}_i+\eta_{\mathrm{lr}}U\bigl(\indic{i=i_{\text{chosen}}}-p_i\bigr).
 \label{eq:maas}
\end{equation}
where $\widehat{\mathrm{cost}}(i)$ normalizes operator cost and $\eta_{\mathrm{lr}}$ is the step size of Eq.~\eqref{eq:onlinerouter}. A held-out gate keeps only checkpoints with positive gate-sample evidence under the declared regression budget; this is an empirical release rule for workflow candidates, not a proof of population non-regression unless the optional anytime-valid gate or a separate fixed-sample confidence rule is enabled.

\paragraph{Variance-reduced update.}
The raw gradient of Eq.~\eqref{eq:maas} is unbiased but high-variance. In variance-reduced mode, the controller uses the same idea as the learned router: replace raw utility by an advantage against a lagged per-bucket baseline. This improves stability on small deployment streams; the explicit update is in Appendix~\ref{sec:method_ext}.

\subsubsection{Paired-Execution Adoption Gate}
\label{sec:echo}

\begin{figure}[tb]
 \centering
 \includegraphics[width=0.85\linewidth]{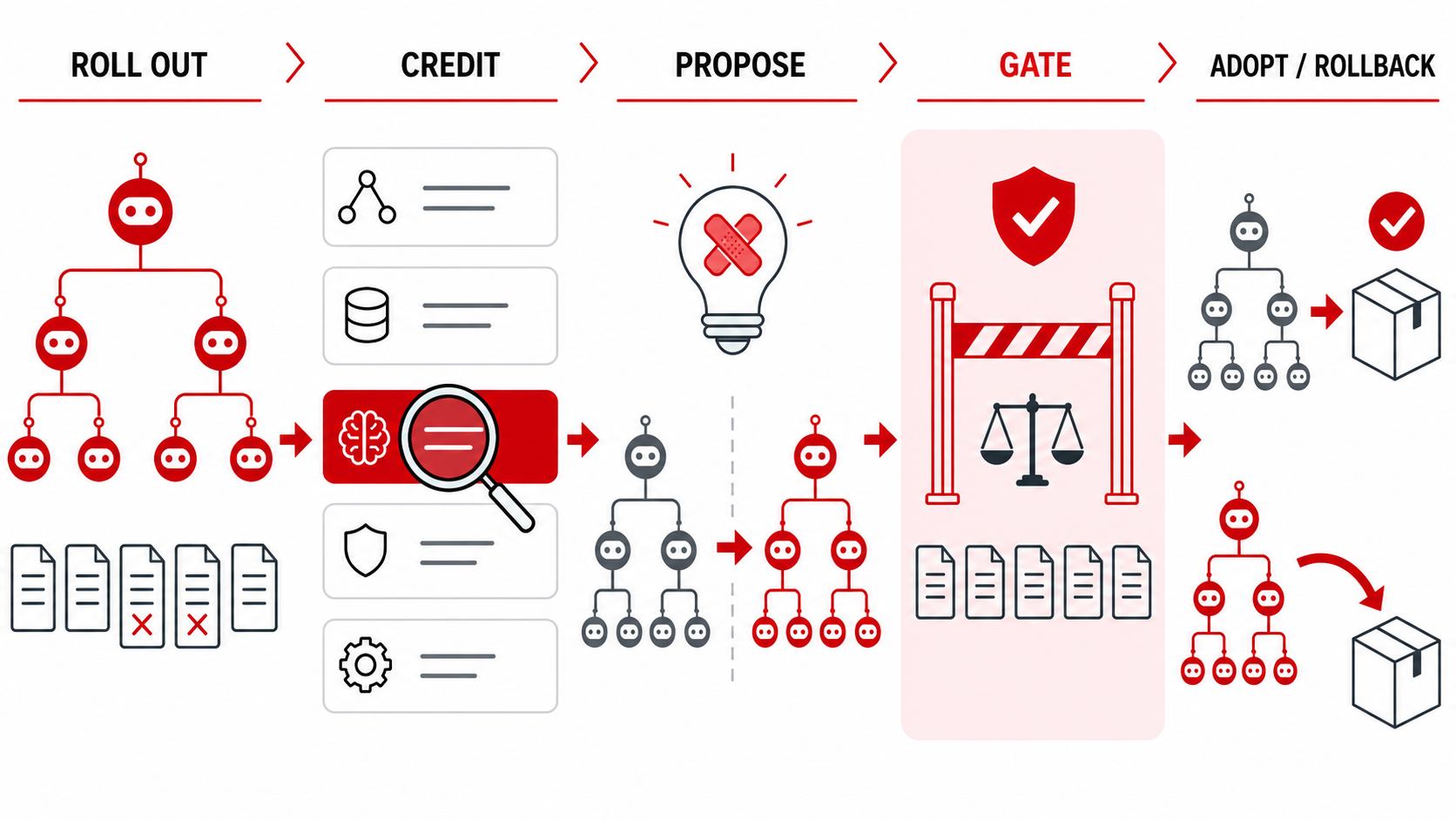}
 \caption{The paired adoption loop: propose from failures, assign credit, evaluate baseline and candidate on held-out execution, then accept only when declared gain, regression, and policy criteria pass.}
 \label{fig:echo}
\end{figure}

The paired-execution controller decides which asset to change and whether the change has sufficient held-out evidence, directly addressing credit dilution and proxy evaluation. It treats every evolvable asset as a typed \emph{component}
\begin{equation}
 \kappa\in\mathcal{K}_{\mathrm{comp}}=\{\text{skill}, \text{workflow}, \text{memory}, \text{role}, \text{prompt}, \text{tool}\},
 \label{eq:component}
\end{equation}
and runs the loop of Figure~\ref{fig:echo}: baseline rollout, credit assignment, proposal, optional candidate selection, final paired-execution gate, accept or reject while leaving the live pack unchanged.

\paragraph{Cross-component credit assignment.}
\label{sec:echo_credit}

To avoid credit dilution, the default selector treats attribution as uncertain unless the trace supports a primary component. It returns a primary component, optional secondary components, or \textsc{abstain}. When attribution evidence is weak, it abstains or falls back to the conservative one-candidate case rather than optimizing an arbitrary component. Writing $a(\zeta)$ for the primary attribution, only primary-attributed failures trigger automatic proposals; secondary labels are logged for optional attribution analysis or human review, and abstentions are not optimized blindly. A step-level variant assigns blame to one trace event when the fix is event-scoped.

\paragraph{Axiomatic credit.}
An opt-in Shapley-style alternative handles cases where a late component merely exposes an earlier corruptor~\citep{shapley1953value}. It estimates each component's marginal contribution on held-out tasks and reports a culprit only when an anytime-valid interval supports the attribution; otherwise it abstains rather than inventing blame.

\paragraph{The paired-execution gate.}
\label{sec:echo_gate}

A proposed change yields candidate deployment $\mathfrak{D}'$ with the same root-policy envelope as the baseline $\mathfrak{D}$ unless a human-governed configuration path explicitly changes it. The gate evaluates by re-running both $\mathfrak{D}'$ and $\mathfrak{D}$ on a final gate task set $\mathcal{H}_{\mathrm{gate}}$ that is disjoint from the proposal examples and, when $K>1$ candidates are compared, disjoint from the selection split used to pick the candidate. Reported before/after numbers use a third report holdout that is not used for proposal, selection, or adoption. The gate uses the same execution harness as operational execution, not a proposal-time proxy. The comparison is paired by task identifier. When common random numbers, deterministic tools, and fixed model snapshots are available, the per-task difference reuses the same execution randomness:
\begin{equation}
 d_{i,b}=m(\mathfrak{D}',\tau_i;\xi_{i,b})-m(\mathfrak{D},\tau_i;\xi_{i,b}),\quad \bar d_i=B^{-1}\sum_{b=1}^{B}d_{i,b}.
 \label{eq:paired_diff}
\end{equation}
Here $B$ is the number of repeated executions per task. Hosted APIs often cannot satisfy that stronger coupling. Then baseline and candidate order is randomized or interleaved, model/scorer identifiers and tool-state snapshots are recorded, and repeated executions are treated as independent task-paired replicates rather than as common-random-number pairs. Benchmark experiments usually choose score-only $m=s$. Deployments may choose full utility $m=u$ from Eq.~\eqref{eq:task_utility} only after declaring a bounded normalization or a bounded failure value for every task. Let $d_i=\bar d_i$ when replications are used and otherwise the single task-paired difference, let $\widehat{\Delta}_{\mathcal{H}}=|\mathcal{H}|^{-1}\sum_i d_i$ be the empirical gate gain, and let $R_{\mathcal{H}}=\bigl|\{i:d_i<-\epsilon_{\text{reg}}\}\bigr|$ count observed regressed tasks. The evidence gate accepts the candidate only if
\begin{equation}
 \text{accept} \iff \widehat{\Delta}_{\mathcal{H}} \ge \delta_{\min} \wedge R_{\mathcal{H}} \le R_{\max},
 \label{eq:echo_gate}
\end{equation}
where $R_{\max}$ caps the number of tolerated observed regressions on the gate sample. Setting $R_{\max}=0$ and $\epsilon_{\text{reg}}=0$ enforces zero observed task-level regressions on that sample; larger values implement the regression-budgeted empirical regime of Section~\ref{sec:bg_misevolution}. An accepted candidate is promoted only if the root-policy and required-approval checks in Eq.~\eqref{eq:promotion_contract} also pass. Defaults are in Appendix~\ref{sec:exp_appendix}.
For binary accuracy, a common default is a one net task-equivalent gain on the gate set. For continuous metrics such as token F$_1$, rubric scores, graded safety, or deployment utility, $\delta_{\min}$ must instead be a pre-declared minimum practical effect in that metric's units; it is not automatically meaningful as one task. A deployment may also express the regression budget as both a count and a rate, and may add a maximum per-task loss or critical-slice veto so that a severe per-task failure cannot be hidden by many small gains. For benchmarked skill adoption, $\delta_{\min}$ is a utility-evidence requirement: a candidate that merely preserves the gate score without fixing any gate task is recorded as non-regressing but is not promoted.
If the candidate fails to return a result for a baseline gate task, \logos assigns the pre-declared failure value $m_{\mathrm{fail}}$ and keeps the task in both $\widehat{\Delta}_{\mathcal{H}}$ and $R_{\mathcal{H}}$; dropped tasks cannot shrink the denominator or inflate the mean over survivors. Score-only benchmark gates use $m_{\mathrm{fail}}=0$. Utility gates must define $u_{\mathrm{fail}}$ or a normalized zero before the gate runs.

\paragraph{Optional anytime-valid gate.}
The fixed threshold in Eq.~\eqref{eq:echo_gate} is an empirical release rule. The stricter variant replaces the fixed threshold with separate gain and regression e-processes, records \textsc{accept}, \textsc{reject-harmful}, \textsc{reject-insufficient-utility}, \textsc{undecided-insufficient-evidence}, or \textsc{conflict}, and spends the resulting composite p-values through a LORD-style ledger under the sampling contract of Appendix~\ref{sec:sage_lord_formal}. Its statistical scope is the declared gate stream with score-only or pre-declared bounded normalized utility: verifier false accepts, exposed or adaptively reused holdouts, adaptive overfitting to repeated aggregate feedback, distribution shift beyond the sampled gate, and unbounded utility scales require separate controls.

\paragraph{Safety non-regression axis.}
Safety is not only an average score. Root policy can require a mean safety margin, but high-risk deployments also attach hard vetoes for critical categories, permission/effect invariant violations, severe-event counts, and per-slice or quantile regressions. We write the mean component as $\delta_{\text{safety}} = s_{\text{safety}}(\mathfrak{D}')-s_{\text{safety}}(\mathfrak{D})\ge0$ under the same root-policy envelope, but this condition is only one part of the safety gate. The safety axis is mandatory for high-risk changes, permission expansions, write-capable tools, safety-surface prompts, and any deployment whose root policy declares a safety objective; it is optional only for low-risk offline mechanism checks that state the omission. A candidate that raises task accuracy but erodes safety or triggers a critical safety veto is treated as a regression and blocked before it reaches the LORD ledger.

\paragraph{Evidence records and the holdout firewall.}
Each gate emits two artifacts. The \emph{proposal-safe} view contains the decision, aggregate counts, score deltas, and hashes only; it withholds task identifiers and per-example failures so the next proposal cannot specialize directly to the gate split. Aggregate feedback can still leak information after repeated submissions, so the root policy defines a gate-exposure budget
\begin{equation}
 \mathcal{B}_{\mathrm{expose}}=(\text{max decisions},\text{max aggregate releases},\text{max cumulative privacy loss}),
 \label{eq:gate_exposure_budget}
\end{equation}
and retires or refreshes final-gate examples after exposure, repeated selection pressure, budget exhaustion, or distribution drift. The \emph{provenance-checkable} evidence record contains the schema version, gate family, metric, model and scorer context when available, task-set hash, baseline and candidate output hashes, gain process, regression process, safety process if enabled, decision, and timestamp. Deterministic replay additionally requires immutable model snapshots or archived outputs, fixed tool state, and deterministic scorers. A stronger \emph{machine-verifiable} certificate is available only for deterministic verifiers, such as unit tests or schema checks, where an independent checker can recompute the verdict from the artifact bundle. LLM-judged or free-form gates therefore remain provenance-checkable audit records rather than formal proof certificates.

\textbf{Comparison.} The fixed paired gate is simple and execution-grounded; the optional anytime-valid gate adds error semantics with lower adoption power. Section~\ref{sec:exp_evolve} shows the empirical difference: ungated and proxy rules adopt many non-improving candidates, the paired gate reduces that rate, and the anytime-valid gate is conservative, adopting only when streaming evidence is conclusive.

Two properties explain why the paired gate reduces proxy-driven misevolution in the measured setting. First, scoring uses the real gate harness, not a self-judged proxy, so acceptance depends less directly on proposal-time surrogate scores; finite-sample gate error and verifier error remain. Second, candidates are evaluated on copies and written back only after acceptance, so rejection leaves the deployed artifact unchanged. Accepted changes therefore carry gate-sample non-regression evidence under the stated metric, subject to gate estimation error, distribution shift between gate and deployment, and any explicitly tolerated regression budget.

\subsubsection{The Outer Loop: Objective Stopping for Bounded Growth}
\label{sec:growth_loop}

Deployment composes single promotions into a growth loop: cycle $t$ runs the pack, records outcome $X_t\in\{0,1\}$, and invokes gated enhancement only when ledger evidence warrants. To prevent cost drift, \logos halts at the earliest stopping time of Eq.~\eqref{eq:growth_stop}: done with enough evidence, futility with human escalation, budget cap, or human stop signal.

The done predicate is computed from recorded harness outcomes, not the model's self-report, so the agent cannot declare itself done or hide a plateau. When the anytime-valid gate is enabled, adoption evidence remains valid at stopping times and is charged to the LORD ledger; the outer rule then affects cost but not the false-adoption accounting. All rules are opt-in.

\subsubsection{Bounded, Exploratory, and Self-Referential Evolution}
\label{sec:advanced_evolution}
Four opt-in instruments govern extended operation without bypassing the gate: retirement protection, memory-edit policies, diversity-preserving search~\citep{mouret2015mapelites}, and gated self-referential tool-code repair~\citep{zhang2025dgm,robeyns2025sica}. Appendix~\ref{sec:advanced_appendix} formalizes and benchmarks all four.

\subsection{Auditable Human--Agent Loop Engineering}
\label{sec:hitl}

\begin{figure}[tb]
 \centering
 \includegraphics[width=0.9\linewidth]{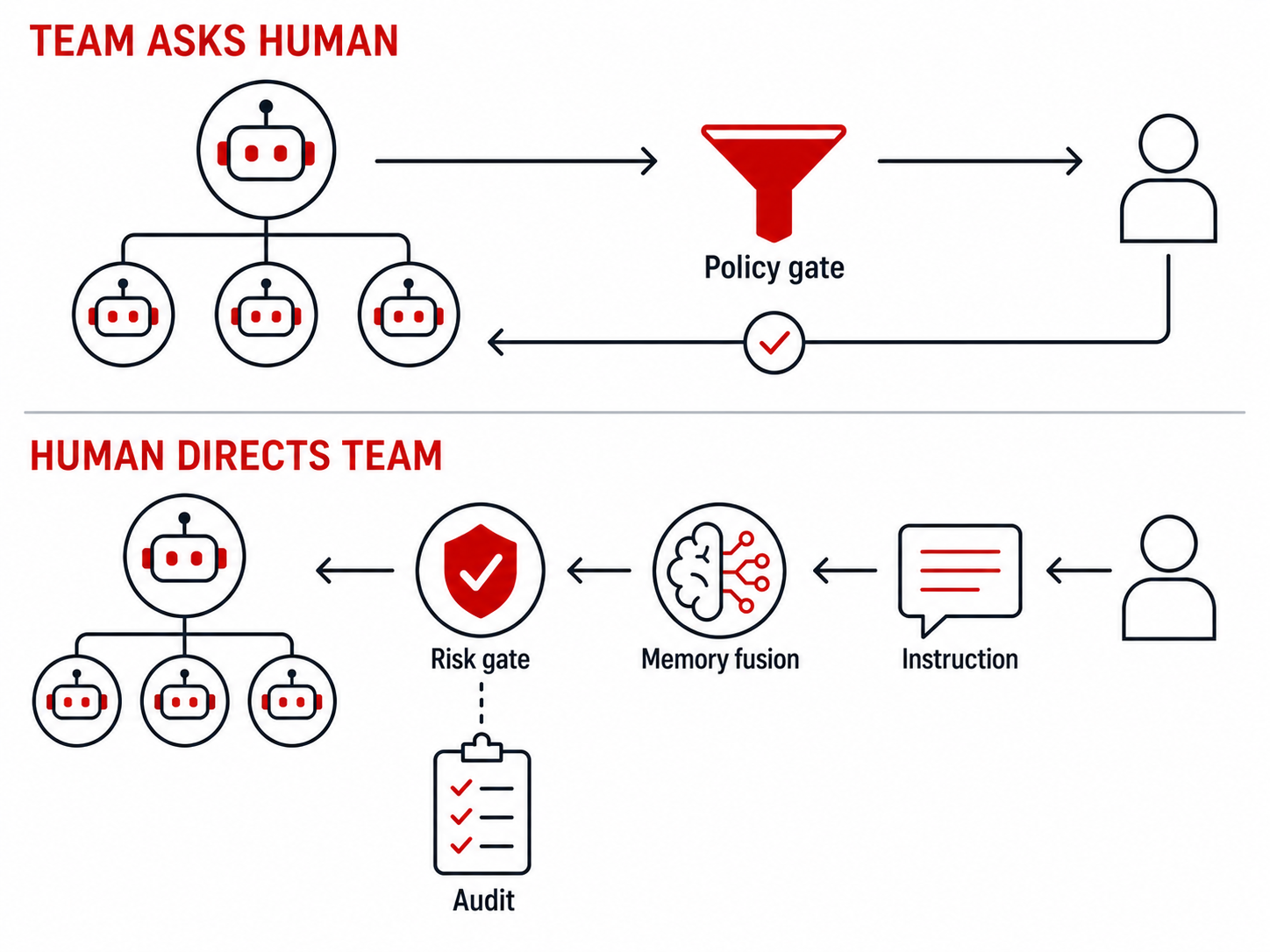}
 \caption{Auditable human--agent control through budgeted questions, directive synthesis, risk gating, approval, and audit.}
 \label{fig:hitl}
\end{figure}

The paired gate decides whether a change improves measured behavior; this section ensures evolution follows human intent and remains promotable. In this paper, ``evolving with humans'' means governed bidirectional adaptation. Figure~\ref{fig:hitl} shows two auditable loops. In the \emph{execution loop}, the MAS asks for missing information, credentials, approval, or authority and records the answer as a scoped event. In the \emph{evolution loop}, humans steer future behavior through directives, approvals, rejections, and policy changes, and those directives become candidates subject to evidence and authority checks. Here \emph{auditable} means inspectable against explicit event, evidence, authority, and audit contracts; human judgment remains governed by role, scope, provenance, and approval policy. A resident workplace state extends both loops across interactions.

\paragraph{Operator-facing observability and asynchronous control.}
\logos projects the normalized execution trace into an operator-facing event interface that exposes decision-relevant state, not hidden chain-of-thought. The observable records include delegation edges, active and completed tasks, tool calls, verifier verdicts, cost and risk signals, pending questions, proposed artifact changes, approvals, rejections, pause and resume events, and audit hashes. A human operator may issue scoped asynchronous directives to a run, agent, or World; approve, reject, or narrow a proposed action; answer typed questions; and interrupt or resume a governed checkpoint without losing provenance. This paper specifies the event, control, precedence, and audit contract rather than a particular interface architecture, visual design, or deployment surface.

\subsubsection{The MAS Questions the Human}
\label{sec:ask}

Rather than silently guessing under uncertainty, a \logos team can ask the human for input along two timelines.

\paragraph{Pre-run clarification.}
Before executing a task, the team analyzes its own information gaps. An LLM inspects the objective and material and emits a set of structured questions $\mathcal{Q}_{\mathrm{ask}}$, where the subscript distinguishes this set from the brief space $\mathcal{Q}$ of Section~\ref{sec:bg_mas},
\begin{equation}
 \mathcal{Q}_{\mathrm{ask}}=\bigl\{(\beta_i,\text{why}_i,\varsigma_i,\text{opts}_i)\bigr\},\quad \varsigma_i\in\{\text{blocking},\text{optional}\}.
 \label{eq:clarify}
\end{equation}
where $\beta_i$ is the question, $\text{why}_i$ its justification, $\varsigma_i$ its severity, and $\text{opts}_i$ optional choices. Only blocking questions are surfaced up to a budget. Answers become a confirmed-requirements block. In non-interactive or timed-out settings, optional preferences may proceed with stated assumptions, but blocking facts, credential gaps, permission ambiguity, external effects, and root-policy conflicts produce a typed ticket or fail-closed rejection rather than a guessed execution.

\paragraph{Mid-run questioning.}
During execution, an enabled ask-user tool handles ambiguity, missing credentials, or risky branches. Low-risk preference gaps can be resolved under stated assumptions. Missing credentials, authority changes, irreversible effects, and safety-relevant uncertainty pause the run, create an operator ticket, or reject the branch according to the root policy.

\paragraph{Budgeted escalation and the $\askf$ objective.}
Asking too often is as harmful as asking too rarely. \logos therefore gates every question through a budgeted \emph{ask policy} $\pi_{\text{ask}}$ that decides whether to escalate,
\begin{equation}
 \pi_{\text{ask}}(\beta, \text{context}) \in \{\text{ask}, \text{resolve alone}\}.
 \label{eq:askpolicy}
\end{equation}
An LLM judge scores residual uncertainty, deliverable impact, and the cost of guessing wrong. The policy escalates only when the resulting risk score, inspired by value-of-information reasoning but implemented as a deployable heuristic, clears a budgeted threshold:
\begin{equation}
 \widehat{\mathrm{VOI}}_{\mathrm{ask}}(\beta)=(1-\text{conf})\cdot\text{impact}\cdot\text{cost}_{\text{wrong}},\quad \pi_{\text{ask}}(\beta)=\text{ask}\iff \widehat{\mathrm{VOI}}_{\mathrm{ask}}(\beta)\ge\theta_{\text{ask}}.
 \label{eq:voi}
\end{equation}
subject to a per-run budget. As $G_{\text{att}}$ is spent, $\theta_{\text{ask}}$ rises, making further interruptions harder. Eq.~\eqref{eq:voi} is a deployable surrogate for Eq.~\eqref{eq:bg_voi}, not an exact value-of-information computation: $(1-\text{conf})$ approximates whether an answer would change the action, and $\text{impact}\cdot\text{cost}_{\text{wrong}}$ approximates the utility at stake.

For labeled blockers we evaluate selective escalation with the Ask-F1 family of metrics used by HiL-Bench~\citep{elfeki2026hilbench}. A question is a hit when a task-specific matcher, fixed before evaluation, pairs it to at least one gold blocker and the answer resolves the blocker enough for the task to proceed. Matching is bipartite and one-to-one unless a benchmark declares that a single broad question resolves multiple blockers; optional questions are counted in the denominator when they interrupt the user but not in blocker recall. We compute
\begin{equation}
 \mathrm{P}=\nicefrac{n_{\text{hit}}}{n_{\text{ask}}},\quad \mathrm{R}=\nicefrac{n_{\text{res}}}{n_{\text{blk}}},\quad \askf=\nicefrac{2\mathrm{P}\mathrm{R}}{\mathrm{P}+\mathrm{R}}.
 \label{eq:askf1}
\end{equation}
where $n_{\text{hit}}$ counts matched useful questions, $n_{\text{ask}}$ counts surfaced questions, $n_{\text{res}}$ counts resolved gold blockers, and $n_{\text{blk}}$ counts gold blockers. We use the standard zero-denominator conventions: $\mathrm{P}=1$ when no questions are asked and no false question exists, $\mathrm{R}=1$ when no gold blocker exists, and $\askf=0$ when $\mathrm{P}+\mathrm{R}=0$. Precision penalizes needless interruptions; recall penalizes unresolved gold blockers. The run reward used by Sections~\ref{sec:memory} and~\ref{sec:echo} is then blended with this interaction quality at the same aggregation level as $\askf$, namely task, batch, or campaign as declared by the experiment,
\begin{equation}
 r'=(1-\omega)r+\omega\bigl(2\askf-1\bigr),\quad \omega\in[0,1].
 \label{eq:rewardblend}
\end{equation}
with weight $\omega$ controlling how much asking quality counts. In learned mode, $\pi_{\text{ask}}$ is an evolvable ask-policy skill optimized by the textual-feedback loop against Eq.~\eqref{eq:askf1}; its edits pass the same non-regression gate as any other skill change.

\paragraph{Risk-calibrated questioning.}
When calibration data are available, an optional questioning policy replaces raw confidence with isotonic calibration and chooses questions by model-estimated information gain over the design posterior. Eqs.~\eqref{eq:clara}--\eqref{eq:clara_eig} give the threshold-selection heuristic and question-selection rule. A formal selective-risk claim would require independent calibration, threshold-selection, and final-certification data or an equivalent simultaneous confidence method; the present experiments use this as a calibrated mechanism, not as a certified selective classifier.

\subsubsection{Evolving from Casual Instructions Fused with Memory}
\label{sec:directive}

The complementary channel lets a human steer evolution with informal directives such as ``always cite a source'' or ``prefer brevity.'' Raw appends accumulate contradictions, so \logos applies a \emph{directive-synthesis} operator that fuses the directive $d_{\mathrm{dir}}$ with accumulated state:
\begin{equation}
 C' = \mathrm{synth}\bigl(d_{\mathrm{dir}}, \mathcal{M}_{\text{skill}}, \mathcal{M}_{\text{val}}, \zeta_{\text{recent}}\bigr),
 \label{eq:directive}
\end{equation}
where $C'$ is the candidate synthesized instruction state, $\mathcal{M}_{\text{skill}}$ are learned guidelines, $\mathcal{M}_{\text{val}}$ validated memory, and $\zeta_{\text{recent}}$ the latest outcome. The operator merges, overrides, or scopes the directive; if intent is ambiguous or contradictory, it asks one targeted clarification before committing.

For long-lived deployments, an evolving \emph{goal} carries a persistent learned state. A directive is compiled into a versioned rule or insight record, then updates the goal's learned state atomically while leaving the human-authored objective untouched,
\begin{equation}
\begin{aligned}
 \text{goal.learned\_state}
 &\leftarrow
 \operatorname{merge}_{\mathrm{prec}}\bigl(
 \text{goal.learned\_state},
 \mathrm{record}(d_{\mathrm{dir}},\text{scope},\text{precedence},\text{expiry})
 \bigr),\\
 \text{goal.objective} &\leftarrow \text{goal.objective}.
\end{aligned}
 \label{eq:learnedstate}
\end{equation}
where $\operatorname{merge}_{\mathrm{prec}}$ resolves conflicts by root-policy precedence, scope, expiry, and recency while preserving provenance. The learned state is re-injected on later runs. If a directive implies a workflow change, synthesis proposes a phase DAG that re-enters Section~\ref{sec:workflow}. Every synthesized change remains a gated proposal. Directives that would change the human objective, root policy, evaluator, permission envelope, or approval policy are not learned-state edits; they require the deployment's human-governed configuration path.

\subsubsection{Workplace Residency and World-Scoped Shared Knowledge}
\label{sec:agora}

A resident workplace state keeps a deployment active between interactions. It binds packs, runs, schedules, tickets, and shared lessons to a persistent \emph{World},
\begin{equation}
 \mathbb{W}=\bigl(\text{objective},\text{packs},\text{tickets},\text{schedules},\mathcal{L}\bigr),
 \label{eq:agora_world}
\end{equation}
so the team participates in workplace rhythms instead of waiting to be invoked.

\paragraph{World-scoped lessons with shrinkage.}
Each World carries a shared lesson store $\mathcal{L}=\{\ell\}$, writable by both humans and the evolution loop. Lesson $\ell$ accumulates helpful votes $h_\ell$ and harmful votes $f_\ell$; instead of a noisy raw ratio, the store uses the posterior mean under a uniform $\Beta(1,1)$ prior,
\begin{equation}
 \mathrm{score}(\ell)=\nicefrac{(h_\ell+1)}{(h_\ell+f_\ell+2)}=\mathbb{E}\bigl[\theta_\ell\mid h_\ell,f_\ell\bigr],\quad \theta_\ell\sim\Beta(1+h_\ell,1+f_\ell).
 \label{eq:agora_score}
\end{equation}
which shrinks small samples toward $\nicefrac12$. Recall then blends lexical relevance with earned usefulness,
\begin{equation}
 \mathrm{rank}(\ell;q)=(1-\lambda_{\mathcal{L}})\nicefrac{|T(q)\cap T(\ell)|}{(|T(q)|+1)}+\lambda_{\mathcal{L}}\mathrm{score}(\ell),\quad \lambda_{\mathcal{L}}\in[0,1].
 \label{eq:agora_rank}
\end{equation}
over token sets $T(\cdot)$, and every use is recorded for later votes. Unlike pack-scoped memory, $\mathcal{L}$ is World-scoped and co-edited by humans; the same hygiene gates apply before promotion.

\paragraph{Autonomous cadence under a busy-guard.}
Each World carries schedules $\varsigma_{\mathrm{sch}}$ for \textsc{run}, \textsc{self\_evolve}, and \textsc{reconcile}. A busy-guard skips and re-arms ticks while a campaign is in flight, preventing concurrent self-modifications. Scheduled work re-enters the standard lifecycle, so autonomy adds cadence, not a gate bypass; missing information becomes a typed ticket in the World inbox.

\subsubsection{Evidence-Supported Adoption}
\label{sec:evidence}

Every proposed change, whether produced by paired execution, textual skill optimization, or directive synthesis, passes a uniform safety layer before it becomes permanent. The risk layer has two timescales. Before execution, a footprint screen classifies proposals by touched files, permissions, tools, and safety surfaces; this decides whether sandboxing, extra validators, or human approval are mandatory. After gate execution, the measured held-out score change updates the same risk record before final promotion. We write the combined post-gate classification as
\begin{equation}
 \mathrm{risk}(p) =
 \begin{cases}
 \textsc{high}, & p~\text{touches a safety surface} \vee \Delta s < -\epsilon_{\mathrm{hi}},\\
 \textsc{medium}, & |p|_{\text{edits}} \ge n_{\mathrm{edit}} \vee (-\epsilon_{\mathrm{hi}} \le \Delta s < 0),\\
 \textsc{low}, & \text{otherwise,}
 \end{cases}
 \label{eq:risk}
\end{equation}
where $p$ is the proposal, $|p|_{\text{edits}}$ its edit footprint, $\Delta s=\Delta s_{\mathrm{gate}}$ the measured held-out score change, $\epsilon_{\mathrm{hi}}>0$ a score-drop threshold, and $n_{\mathrm{edit}}\in\mathbb{N}$ an edit-footprint threshold. Before the gate has run, the $\Delta s$ branch is treated as unknown and the footprint/safety-surface branch governs escalation. Safety surfaces include root-policy text, verifier declarations, approval rules, credential scopes, write-capable tool definitions, safety prompts, and adapters that can change external effects. An approval policy escalates by level under \emph{auto}, \emph{on-risk}, or \emph{always}. The guardian first verifies the candidate inside the sandbox; after atomic promotion it can continue staged monitoring with rollback or escalation on verifier failure,
\begin{equation}
 \text{guardian}(\nu) \longmapsto \{\text{accepted}, \text{remediated}, \text{rolled back}, \text{escalated}\}.
 \label{eq:guardian}
\end{equation}
A human gate composes these signals with fixed precedence: root policy is evaluated first and cannot be bypassed; mandatory safety constraints come next, followed by role-based human approval, scoped human directives, risk heuristics, and agent preference. Human feedback is input to adaptation, not an unconditional override. It is interpreted under role, scope, provenance, expiry, tenant or World, conflict status, and the approval requirement declared by the root policy. The invariant is that only verified, non-regressing, authority-compatible changes become permanent learning. Each promotion records a typed evidence object, proposal, risk, scores, verifier verdict, approver, and correlation id, in an append-only audit log; a drift monitor triggers re-validation on sustained drops.

\paragraph{Online false-adoption accounting with LORD.}
When the anytime-valid gate is enabled, each decision controls one candidate acceptance, but a long-lived deployment still needs stream accounting. \logos therefore routes candidates through a LORD-style ledger~\citep{javanmard2018online,ramdas2018saffron}: decisions spend renewable alpha-wealth on p-values derived from the gain, regression, and safety streams; rejected null hypotheses recycle testing wealth; and safety vetoes set the composite p-value to one. The ledger has formal online error semantics only under the conditional super-uniformity, predictable spending, and dependence assumptions stated with Eqs.~\eqref{eq:lord}--\eqref{eq:evalue_composite}; otherwise it is reported as conservative audit accounting.

\subsection{Security, Privacy, and Artifact Scope}
\label{sec:security_artifacts}

\paragraph{Threat model and side effects.}
\logos treats compilation, operation, and evolution as effect-typed software execution. The \emph{effect envelope} declared by root policy is the allowed set of effect classes, write scopes, network destinations, rate limits, idempotency requirements, approval rules, and compensation mechanisms for a deployment. By default, data tools are read-only, SQL tools are guarded by AST allow-lists, destructive tools are mocked or stubbed during validation, and external writes require transactional wrappers with idempotency keys and an effect log. A deployment that enables write-capable tools must declare the tool's effect class, authentication scope, network allow-list, rate limit, rollback action when available, and human-approval policy. The rollback guarantee is structural for pack files and configuration: rejected candidates do not replace the live artifact. \textbf{Structural rollback restores the \pack. It does not automatically undo emails, payments, tickets, database writes, robot actions, or other external effects.} Irreversible effects must be mediated by dry runs, previews, two-phase commits, idempotency keys, transaction logs, compensating actions, or human confirmation.

The main threat surfaces are prompt injection through retrieved documents or tool output, memory poisoning through persistent recalls, verifier exploitation through false accepts, permission expansion through write-capable tools, holdout leakage through repeated gates, and cross-tenant leakage through traces, memories, or indexes~\citep{owasp2025agenticThreats,owasp2026agenticTop10,nvidia2025codeExecution}. \logos mitigates these with provenance tags, instruction/data separation, allow-lists, quarantine and expiry metadata, gated promotion, frozen evaluators, deterministic checks before judge-style validators, least-privilege credentials, one-shot or refreshed gate families, tenant-scoped stores, redaction, and audit ACLs. These are deployment controls and residual-risk statements rather than benchmark wins; the executable evidence appears in the Q17--Q20 action-safety, tenant-isolation, stale-memory, and rollback/audit probes.

\begin{table}[tb]
\centering\small
\setlength{\tabcolsep}{3.8pt}\renewcommand{\arraystretch}{1.12}
\caption{Security boundary assumed by the \logos release contract.}
\label{tab:threat_boundary}
\begin{tabular}{@{}L{0.24\linewidth}L{0.34\linewidth}L{0.32\linewidth}@{}}
\toprule
Object & Trust assumption & Residual risk \\
\midrule
Root policy store & written through a human-governed configuration path; not writable by ordinary candidate processes & misconfiguration or compromised administrator can authorize unsafe envelopes \\
Credential envelope & live credentials are issued at call time with scoped permissions and are redacted from traces and prompts & provider compromise or leaked external secret remains outside the pack gate \\
Candidate sandbox & candidate evaluation receives copied artifacts, scoped credentials, and dry-run or transactional effect handles & sandbox escape or incorrect mock/live separation can invalidate isolation \\
Verifier and scorer & versions, applicability, and sufficiency declarations are frozen for a gate decision & test-coverage gaps, false accepts, and verifier drift remain core dependencies \\
Gate holdout & final examples and per-example failures are hidden from proposal generation and refreshed under exposure budgets & repeated aggregate feedback can still leak information over time \\
Audit sink & append-only or hash-chained export is outside ordinary pack mutation & if the audit sink shares compromised administration, tamper evidence weakens \\
Backend provider & adapter and provider behavior are part of the trusted computing base for a given deployment & backend model, tool, pricing, or API drift can change behavior and requires revalidation \\
\bottomrule
\end{tabular}
\end{table}

\paragraph{Sandbox, secrets, and tenants.}
The execution policy separates proposal-time analysis from live tool execution. Validation and evolution run in read-only or transactional sandboxes whenever possible, and destructive calls are replaced by mocks unless an operator explicitly enables a staged environment. Secrets are excluded from traces and proposal prompts; adapters expose only scoped credentials at call time and redact them from audit records. Multi-tenant deployments isolate packs, traces, indexes, and World lessons by tenant. Prompt-injection defenses remain defense-in-depth rather than proofs: retrieved memory and documents carry provenance, tool instructions are separated from untrusted content, and high-risk directives must pass the risk gate of Eq.~\eqref{eq:risk}.

\paragraph{Audit log integrity.}
Every compile, route, verifier decision, proposal, gate result, approval, and rollback appends a typed event with a correlation id, actor, timestamp, hashes of relevant artifacts, and effect metadata. Deployments that need tamper evidence can hash-chain these records or export them to an external append-only store. The proposal-safe gate view deliberately withholds per-example final-gate failures; the full audit bundle is retained for authorized reproduction and incident review.

\paragraph{Reproducibility scope.}
This paper specifies the artifact schema, normalized event types, gate inputs and outputs, benchmark splits, metrics, and provenance needed to interpret the results. Deployment credentials, infrastructure layout, interface presentation details, and operational configuration are outside the scientific claim. Stochastic model or tool calls require archived outputs or repeatable verifiers; a durable release should freeze data and schema versions, split hashes, scorer versions, and the evidence needed to reproduce each reported decision.

\section{Experimental Evaluation}
\label{sec:experiments}

The evaluation is organized by claim rather than by module count. Four evidence classes are kept separate throughout: \emph{external benchmarks} measure task-family performance under a stated model, split, and scorer; \emph{controlled mechanism studies} isolate a gate, router, or memory rule; \emph{generated stress tests} exercise known failure modes at scale; and \emph{conformance checks} test whether authority, audit, and artifact contracts compose as specified. Synthetic and conformance results are reported as controlled engineering evidence, not as third-party production studies.

Table values are descriptive for their stated protocol. A displayed $0.000$ means that no event was observed under the stated denominator, not that the true event probability is zero. Generated suites primarily test specification and fault-injection behavior; when cases are generated from the same contract rules that \logos must preserve, the row is a conformance test rather than a statistical generalization claim.

\begin{table}[tb]
\centering\small
\setlength{\tabcolsep}{4.0pt}\renewcommand{\arraystretch}{1.12}
\caption{How to read the evaluation.}
\label{tab:evaluation_protocol}
\begin{tabular}{@{}L{0.23\linewidth}L{0.34\linewidth}L{0.33\linewidth}@{}}
\toprule
Evidence class & Unit of analysis & Claim it supports \\
\midrule
External benchmark & benchmark task with fixed scorer and stated model tier & task-family performance under the reported harness; not production utility by itself \\
Controlled mechanism study & paired task, candidate decision, or replayed routing case & whether a specific gate, router, or memory rule changes outcomes under fixed candidates \\
Generated fault injection & generated case from disclosed templates and perturbations & whether the contract catches known classes of failures; not an estimate of real-world frequency \\
Conformance check & policy, trace, effect, rollback, or tenant-isolation case & whether the specified artifact and authority contract is preserved \\
\bottomrule
\end{tabular}
\end{table}

The subsections below retain the Q1--Q22 identifiers as stable protocol labels. They are not a priority order. A claim-first reading order is governed self-evolution in Q3, human authority and effects in Q5 and Q9--Q22, pluggable construction and target-side re-verification in Q1 and Q6, verification-backed routing in Q7--Q8, and secondary task and memory studies in Q2 and Q4.

\begin{table}[tb]
\centering\footnotesize
\setlength{\tabcolsep}{4.2pt}\renewcommand{\arraystretch}{1.2}
\caption{Contribution-aligned map of the evaluation. Headline values are descriptive for the stated protocol; the final column states the main evidence boundary.}
\label{tab:claim_map}
\begin{tabular}{@{}>{\raggedright\arraybackslash}p{0.22\linewidth}>{\raggedright\arraybackslash}p{0.21\linewidth}>{\raggedright\arraybackslash}p{0.30\linewidth}>{\raggedright\arraybackslash}p{0.19\linewidth}@{}}
\toprule
Contribution & Primary evidence & Headline result & Boundary \\
\midrule
Multimodal compiler & Q1 compiler stress and generated conformance & C3/C4/C5 rates $0.882/0.720/0.757$; multimodal conformance covers $96{,}000$ generated cases & generated briefs and source templates \\
Pluggable execution and evidence & adapter conformance, Q7 verifier stress, Q8 routing replay & adapter event coverage and unsupported diagnostics are $1.000$; composite verifier false-accept rate is $0.007$ at $0.796$ coverage & adapter stubs and verifier quality bound the claim \\
Evidence-governed self-evolution & Q3 common-candidate replay and gate ablations & the paired gate adopts $0\%$ harmful-family candidates and $100\%$ beneficial-family candidates in $5{,}000$ fixed decisions; no gate adopts all candidates & controlled candidate families, not a population estimate over real proposals \\
Human authority and effects & human-loop conformance, Q5, Q9--Q22 & no violation of encoded authorization rules is observed for \logos in the generated conformance suite, versus $0.750$ for an ungoverned policy & controlled known-risk tasks and conformance rules \\
Evidence-scoped evaluation & protocol table, negative controls, and conformance checks & negative candidate cells are treated as routing or gating decisions, while conformance checks remain separate from external benchmarks & evidence classes remain distinct; no universal accuracy or safety claim \\
\bottomrule
\end{tabular}
\end{table}

Each table states its unit of analysis, scorer, and interpretation. Self-evolution separates proposal data, optional candidate-selection data, final-gate data, and report holdouts. Where paired comparisons are meaningful, the same task identifiers and candidate records are held fixed. Appendix~\ref{sec:exp_appendix} records protocol details and uncertainty conventions using artifact-level identifiers.

\subsection{Q1: Compilation Fidelity}
\label{sec:exp_compile}

Q1 asks whether an emitted pack is callable under the compiler's finite validation battery. A support-triage brief is compiled under three synthesis modes, and an OpenAPI specification is ingested into typed tools. Table~\ref{tab:compile} reports schema, tool-execution, and finite-probe validation. These checks establish artifact callability under the declared probes, not held-out business utility.

\begin{table}[tb]
\centering\small
\setlength{\tabcolsep}{4.0pt}\renewcommand{\arraystretch}{1.12}
\caption{Compilation fidelity under the three synthesis modes and OpenAPI tool ingestion. Values are finite-probe pass rates in percent; N/A marks probes that do not apply to the artifact.}
\label{tab:compile}
\begin{tabular}{lccccc}
\toprule
Mode / artifact & Agents & Skills & Schema valid & Tool exec. & Probe pass \\
\midrule
Deterministic template & $1$ & $2$ & N/A & N/A & $100$ \\
LLM single-playbook & $1$ & $2$ & N/A & N/A & $100$ \\
LLM structured team & $7$ & $5$ & N/A & N/A & $100$ \\
\midrule
OpenAPI tool ingest, $2$ tools & N/A & N/A & $100$ & $100$ & $100$ \\
\bottomrule
\end{tabular}
\end{table}

The finite probes above are basic checks for artifact callability. They do not,
by themselves, show the harder business-automation property: whether a generated
pack remains useful on held-out work without over-broad permissions. We
therefore add a model-independent C3--C5 fault-injection suite over $200$ seeds and
$240$ generated business briefs per seed, spanning support, compliance, API
workflow, retrieval, finance, and HR. The generator uses declared templates and
perturbation rules with known labels for C3 runnable completion, C4 held-out
workflow success, C5 safety and permission fit, tool-call success, synthetic
repair effort, and time to first runnable pack. These are independent axes
rather than cumulative levels. Table~\ref{tab:compiler_l3_l5} is controlled
evidence for the compiler contract and should be read together with external
deployments or independent build studies; the generated families make coverage
visible at the level of support, compliance, API workflow, retrieval, finance,
and HR briefs. The repair column is a score derived
from generated repair operations, not measured human labor or wall-clock editing
time.

\begin{table}[tb]
\centering\small
\setlength{\tabcolsep}{3.0pt}\renewcommand{\arraystretch}{1.12}
\caption{Compiler C3--C5 fault-injection stress over $48{,}000$ generated operational-pack briefs. Time and repair score apply only to generated builders; N/A marks the pre-specified hand-authored template proxy.}
\label{tab:compiler_l3_l5}
\begin{tabular}{@{}>{\raggedright\arraybackslash}p{0.24\linewidth}ccccccc@{}}
\toprule
Builder & \shortstack{Time\\(s)} & \shortstack{Repair\\score} & C3 & C4 & C5 & \shortstack{Excess\\perm.} & \shortstack{Tool\\success} \\
\midrule
One-shot LLM & $12.8$ & $43.5$ & $0.631$ & $0.286$ & $0.334$ & $0.264$ & $0.597$ \\
Framework template & $37.1$ & $30.7$ & $0.732$ & $0.421$ & $0.492$ & $0.172$ & $0.713$ \\
\logos compiler & $30.7$ & $\mathbf{14.1}$ & $0.882$ & $0.720$ & $0.757$ & $0.063$ & $0.883$ \\
Hand-authored template proxy & N/A & N/A & $\mathbf{0.923}$ & $\mathbf{0.808}$ & $\mathbf{0.829}$ & $\mathbf{0.052}$ & $\mathbf{0.925}$ \\
\bottomrule
\end{tabular}
\end{table}

The hand-authored proxy is a high-coverage template family supplied to the
harness, not a timed expert study. Because its construction time and repair
operations are outside the measured generation loop, it is reported as a quality
reference under the same finite probes rather than as a timed build method.

\subsection{Q2: End-to-End Task Success}
\label{sec:exp_task}

Q2 asks whether compilation helps task success after holding the base model fixed. The setting compares a compiled example-pack team with direct single-model calls on matched reasoning subsets, with \texttt{gpt-4o-mini} subset rates shown separately. Table~\ref{tab:task} gives a positive MATH-500 result for \texttt{gpt-5.1}, while GSM8K and the \texttt{gpt-4o-mini} MATH-500 subset show that orchestration overhead can dominate. The conclusion is therefore conditional: compilation helps when decomposition or suite-specific constraints add useful structure, and \logos selects the compiled route only when target-workload validation supports it.

\begin{table}[tb]
\centering\small
\setlength{\tabcolsep}{5pt}\renewcommand{\arraystretch}{1.1}
\caption{Compiled team versus direct single-model calls on matched task subsets. Values are pass@1 in percent; $\Delta$ is compiled minus single. Sample sizes, sampling policy, and paired-test metadata are in Appendix~\ref{sec:exp_appendix}.}
\label{tab:task}
\begin{tabular}{@{}lccc c cc@{}}
\toprule
& \multicolumn{3}{c}{\texttt{gpt-5.1}} & & \multicolumn{2}{c}{\texttt{gpt-4o-mini} subset} \\
\cmidrule(lr){2-4}\cmidrule(lr){6-7}
Suite & single & compiled & $\Delta$ & & single & compiled \\
\midrule
GSM8K & $99.0$ & $96.0$ & $-3.0$ & & $93.0$ & $95.0$ \\
MATH-500 & $86.6$ & $89.4$ & $\mathbf{+2.8}$ & & $80.0$ & $65.0$ \\
\bottomrule
\end{tabular}
\end{table}

\logos therefore treats the compiled team as a candidate rather than an
unconditional replacement for direct inference. Under the validated-route rule,
direct execution and the configured \logos solver are compared on a validation
split; the candidate is selected only if it clears the declared gain margin and
regression cap; and a disjoint report split should then be run with the selected
solver before claiming deployed performance. Table~\ref{tab:task_validated_route}
is a validation-route audit over the measured cells in Table~\ref{tab:task}, not
a separate report-holdout estimate. Its purpose is to show how negative
candidate evidence becomes a direct-route decision.

This Q2 table is not compute-matched: compiled runs may include additional
agent calls, tool checks, or verifier-visible tests. We therefore use it as a
route-selection audit under the measured harness, not as evidence that a
compiled architecture dominates a single call at equal token, latency, or
verifier budget.

\begin{table}[tb]
\centering\small
\setlength{\tabcolsep}{4.2pt}\renewcommand{\arraystretch}{1.1}
\caption{Validation-route audit applied to the Q2 matched task cells. Values
are validation pass@1 in percent; no disjoint report split is included in this
summary.}
\label{tab:task_validated_route}
\begin{tabular}{@{}llcccc@{}}
\toprule
Model & Suite & Direct & Candidate & Selected route & Route $\Delta$ \\
\midrule
\texttt{gpt-5.1} & GSM8K & $99.0$ & $96.0$ & \textsc{direct} & $0.0$ \\
\texttt{gpt-5.1} & MATH-500 & $86.6$ & $89.4$ & \textsc{compiled} & $\mathbf{+2.8}$ \\
\texttt{gpt-4o-mini} & GSM8K & $93.0$ & $95.0$ & \textsc{compiled} & $\mathbf{+2.0}$ \\
\texttt{gpt-4o-mini} & MATH-500 & $80.0$ & $65.0$ & \textsc{direct} & $0.0$ \\
\bottomrule
\end{tabular}
\end{table}

\paragraph{Boundary.}
Compilation helps when decomposition or verification creates useful structure,
and it can hurt when orchestration overhead dominates. The \texttt{gpt-4o-mini} MATH-500
and HotpotQA cells therefore become candidate-level boundary evidence: the
deployable \logos policy routes around those regimes and selects direct
execution unless the compiled team validates on the target workload.

\subsection{Q3: Disciplined Self-Evolution}
\label{sec:exp_evolve}

Q3 asks whether release decisions are governed by execution evidence rather than proposal-time optimism. The primary analysis is a \emph{semantic unit-style gate check} implemented as a common-candidate replay: incumbent outputs, candidate outputs, task identifiers, proxy signals, and safety metadata are fixed, and only the gate rule changes. This isolates the gate semantics under constructed candidate families, but it is not a population estimate over naturally occurring proposals. A second analysis runs the full proposal-and-selection loop, where selection pressure and candidate quality may differ across settings. Table~\ref{tab:gate_replay} is therefore the cleanest gate check; Figure~\ref{fig:gate_ablation} and Table~\ref{tab:evolve_direct} show the harder end-to-end boundary. The paired gate removes harmful adoption in the fixed replay and reduces it in the full loop, while the optional anytime-valid gate is more conservative when evidence is weak.

\begin{table}[tb]
\centering\small
\setlength{\tabcolsep}{3.8pt}\renewcommand{\arraystretch}{1.1}
\caption{Semantic common-candidate gate check over $2{,}500$ deterministic decisions ($500$ seeds $\times$ five candidate families). Beneficial, neutral, and harmful adoption are separated; net utility prices task regressions and safety loss. Where shown, $\pm$ denotes seed-level harness uncertainty.}
\label{tab:gate_replay}
\begin{tabular}{@{}>{\raggedright\arraybackslash}p{0.19\linewidth}ccccc@{}}
\toprule
Gate rule & Adopt & Beneficial & Neutral & Harmful & \shortstack{Net utility\\per decision} \\
\midrule
No gate & $1.000$ & $0.200$ & $0.200$ & $0.600$ & $-0.083{\pm}0.003$ \\
Proxy signal & $0.800$ & $0.200$ & $0.200$ & $0.400$ & $+0.041{\pm}0.002$ \\
Paired gate & $0.200$ & $0.200$ & $0.000$ & $0.000$ & $\mathbf{+0.050{\pm}0.001}$ \\
Anytime-valid gate & $0.198{\pm}0.002$ & $0.198{\pm}0.002$ & $0.000$ & $0.000$ & $\mathbf{+0.050{\pm}0.001}$ \\
\bottomrule
\end{tabular}
\end{table}

In this replay, the family label and oracle outcome are fixed by the
candidate-generation rules. The gate receives only the recorded task-paired
outputs, proxy signals, and safety metadata used by the stated rule, not a
hidden family label. Thresholds are fixed before replay. The check is useful for
gate semantics and boundary failures, but it should not be read as independent
evidence about proposal quality in the wild.

A larger model-independent robustness replay, run after the main table was fixed,
keeps the qualitative ordering unchanged. Over $5{,}000$ common-candidate
decisions ($1{,}000$ seeds $\times$ five families), no-gate adoption again
accepts all candidates with $60.0\%$ harmful adoption and negative utility
($-0.156{\pm}0.002$ per decision); the proxy rule accepts $40.0\%$ harmful
candidates with near-zero utility ($0.005{\pm}0.001$); and both evidence gates
accept only the beneficial family, with no observed harmful adoption and
$0.049{\pm}0.001$ utility per decision. Because this replay is generated from
the same constructed candidate-family templates rather than an independent
external benchmark, we use it as a semantic robustness check rather than
replacing the headline protocol in Table~\ref{tab:gate_replay}.

\begin{figure}[tb]
 \centering
 \includegraphics[width=0.86\linewidth]{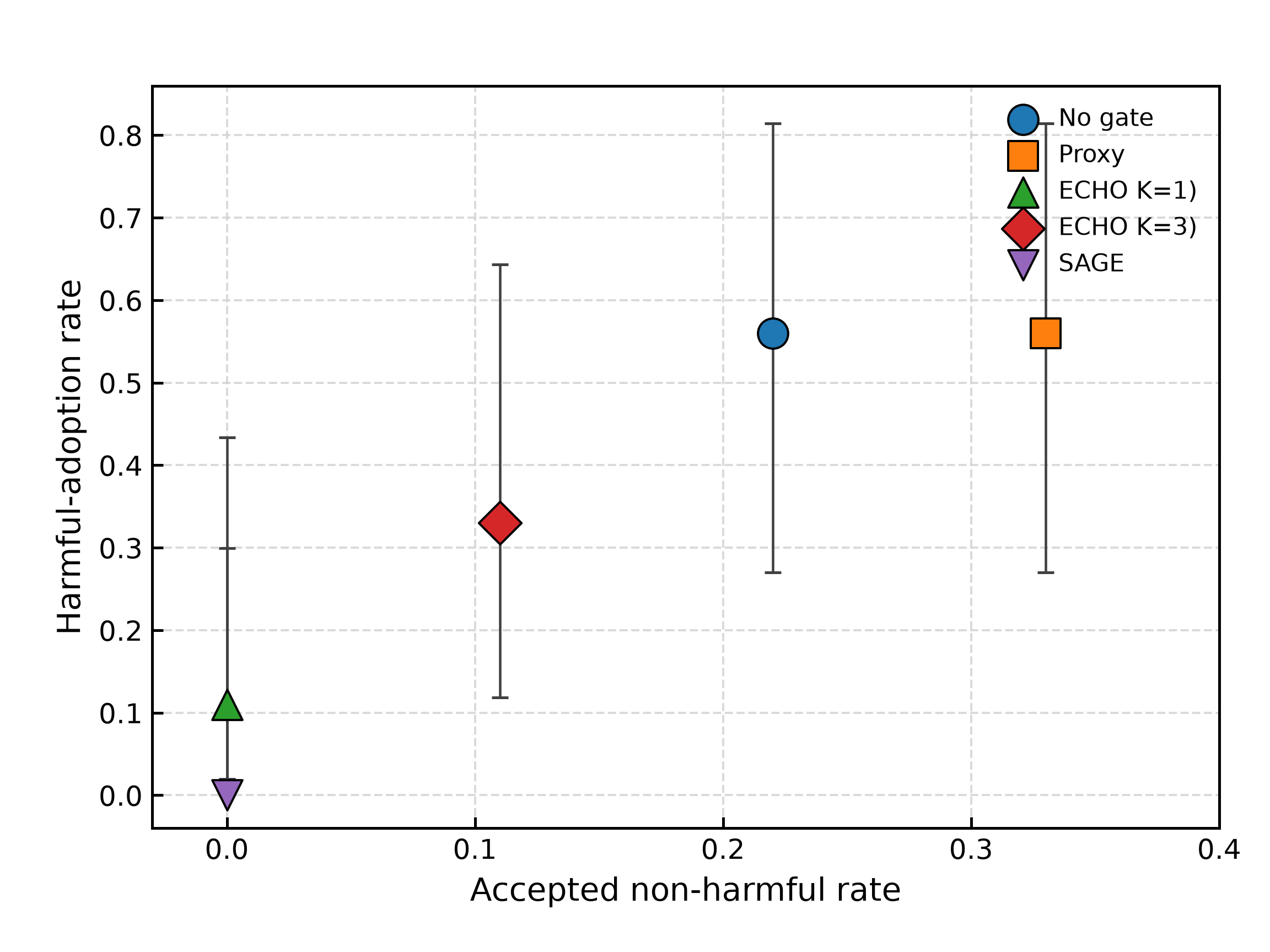}
 \caption{Safety-power frontier for the evolution gate ablation.}
 \label{fig:gate_ablation}
\end{figure}

\begin{table}[tb]
\centering\small
\setlength{\tabcolsep}{3.0pt}\renewcommand{\arraystretch}{1.12}
\caption{Multi-seed gate matrix on non-saturated suites. No report-holdout regressions are observed for the anytime-valid gate because it fully rejects in these low-signal cells; the paired gate substantially reduces regressions but has one residual MATH-500 regression.}
\label{tab:evolve_direct}
\begin{tabular}{@{}llccc@{}}
\toprule
Suite & Gate & \shortstack{Mean $\Delta$\\($\uparrow$)} & \shortstack{Mean Regr.\\($\downarrow$)} & \shortstack{Seeds with\\regr.} \\
\midrule
HotpotQA & off & $-0.028$ & $1.7$ & $2/3$ \\
 & proxy & $0.000$ & $1.0$ & $2/3$ \\
 & execution & $\mathbf{0.000}$ & $\mathbf{0.0}$ & $0/3$ \\
 & anytime-valid & $\mathbf{0.000}$ & $\mathbf{0.0}$ & $0/3$ \\
\midrule
MATH-500 & off & $-0.028$ & $1.7$ & $3/3$ \\
 & proxy & $-0.028$ & $2.7$ & $3/3$ \\
 & execution & $-0.007$ & $0.7$ & $1/3$ \\
 & anytime-valid & $\mathbf{0.000}$ & $\mathbf{0.0}$ & $\mathbf{0/3}$ \\
\bottomrule
\end{tabular}
\end{table}

\subsection{Q4: Gated Memory Quality}
\label{sec:exp_memory}

Q4 tests two memory claims: value-aware retrieval can learn what helps in a controlled setting, and gated memory should act only when there is signal. In the synthetic corpus, value blending learns the updated fact and pruning cuts store size without losing recall. On LongMemEval and LoCoMo~\citep{wu2025longmemeval,maharana2024locomo}, passive memory is negative rather than positive: both Base and the gated-memory variant trail the no-memory condition in Table~\ref{tab:prism_lift}. The measured positive case is gated memory promotion when explicit training signal exists: the gate accepts a LongMemEval lift and rejects the LoCoMo candidates in this run; the no-gate comparison adopts and regresses. LongMemEval-V2 extends this benchmark direction toward workflow knowledge, dynamic state, environment-specific failure modes, and premise awareness in web-agent histories~\citep{wu2026longmemevalv2}.

\begin{table}[tb]
\centering\small
\setlength{\tabcolsep}{4.0pt}\renewcommand{\arraystretch}{1.12}
\caption{Gated-memory mechanism isolation on a synthetic memory corpus. Retrieval, TTL, and recall are hit@1 in percent; Store is the retained entry count.}
\label{tab:memory}
\begin{tabular}{lcccc}
\toprule
Memory variant & Retrieval ($\uparrow$) & TTL ($\uparrow$) & Recall@noise ($\uparrow$) & Store \\
\midrule
Cosine-only, no $Q$ & $100$ & $0$ & $100$ & $38$ \\
Value-aware, no pruning & $100$ & $100$ & $100$ & $38$ \\
Value-aware, full gated memory & $100$ & $100$ & $100$ & $12$ \\
\bottomrule
\end{tabular}
\end{table}

\begin{table}[tb]
\centering\small
\setlength{\tabcolsep}{4.0pt}\renewcommand{\arraystretch}{1.12}
\caption{Passive memory on long-horizon memory benchmarks. With no training signal, Base and gated memory underperform the no-memory condition in this matched pipeline.}
\label{tab:prism_lift}
\begin{tabular}{lccc}
\toprule
Benchmark & None & Base & Gated memory \\
\midrule
LongMemEval (pass rate, $\uparrow$) & $0.83$ & $0.66$ & $0.66$ \\
LoCoMo (mean score, $\uparrow$) & $0.52$ & $0.46$ & $0.45$ \\
\bottomrule
\end{tabular}
\end{table}

\begin{table}[tb]
\centering\small
\setlength{\tabcolsep}{4.0pt}\renewcommand{\arraystretch}{1.12}
\caption{Gated memory evolution with explicit training signal. Values are pass rates; $\Delta$ is in percentage points, and Reg. counts report-holdout regressions.}
\label{tab:prism_evolve}
\begin{tabular}{@{}lcccccccc@{}}
\toprule
& \multicolumn{4}{c}{Paired gate} & \multicolumn{4}{c}{No gate} \\
\cmidrule(lr){2-5}\cmidrule(lr){6-9}
Benchmark & Before & After & $\Delta$ pts & Reg. & Before & After & $\Delta$ pts & Reg. \\
\midrule
LongMemEval & $0.86$ & $0.90$ & $\mathbf{+4.0}$ & $1$ & $0.89$ & $0.89$ & $0.0$ & $5$ \\
LoCoMo & $0.64$ & $0.64$ & $0.0$ & $0$ & $0.62$ & $0.59$ & $-3.0$ & $8$ \\
\bottomrule
\end{tabular}
\end{table}

Thus the deployed memory policy is not ``always retrieve and always learn.''
When no task-side signal is available, the memory layer may bypass retrieval or keep
episodes quarantined; when explicit signal is available, memory edits enter the
same paired gate as other artifact changes. In Table~\ref{tab:prism_evolve},
this policy yields a LongMemEval lift and leaves LoCoMo unchanged, while the
negative LoCoMo row appears only in the no-gate ablation.

\subsection{Q5: Human-Interaction Mechanism Checks}
\label{sec:exp_hitl}

Q5 tests questioning and directive synthesis as mechanism checks. The
controlled ask-policy suite saturates at the tested tier, so it is reported as
a mechanism check rather than a fine-grained comparison of ask policies.
Directive synthesis is discriminating: memory-fused synthesis follows every
override directive, removes contradictions, and asks on ambiguous cases across
the evaluated configurations in Table~\ref{tab:hitl}. The claim is therefore
scoped to override handling and governed escalation in the tested cases, not to
all user-dialogue settings.

\begin{table}[tb]
\centering\small
\setlength{\tabcolsep}{4.0pt}\renewcommand{\arraystretch}{1.12}
\caption{Directive synthesis sanity check over six override cases. Values are counts; follow and clarify should be high, contradiction should be low.}
\label{tab:hitl}
\begin{tabular}{@{}>{\raggedright\arraybackslash}p{0.50\linewidth}ccc@{}}
\toprule
\shortstack[l]{Directive operator\\and backend} & Follow & Contradiction & \shortstack{Clarify\\trigger} \\
\midrule
Naive append, rule-based evaluator & $4/6$ & $6/6$ & $0/6$ \\
Memory-fused guard, rule-based evaluator & $\mathbf{6/6}$ & $\mathbf{0/6}$ & $\mathbf{6/6}$ \\
Memory-fused guard, LLM-assisted evaluator & $\mathbf{6/6}$ & $\mathbf{0/6}$ & $\mathbf{6/6}$ \\
\bottomrule
\end{tabular}
\end{table}

The safe-delegation study evaluates the same oversight rule under known risk labels for low-risk, reversible, external-effect, and authority-changing tasks. It compares manual work, always asking, full autonomy, a static risk rule, and the \logos evidence/risk policy over $500$ seeds and $800$ tasks per seed. Table~\ref{tab:safe_delegation} observes no unauthorized external-effect outcomes for \logos in this generated suite, matching the conservative policies, while reducing human interventions from $100$ to $39.0$ per $100$ tasks and preserving substantially more policy-compliant success than full autonomy. The result is a controlled policy measurement over generated risk labels.

\begin{table}[tb]
\centering\small
\setlength{\tabcolsep}{4pt}\renewcommand{\arraystretch}{1.12}
\caption{Safe-delegation study over $400{,}000$ simulated workplace tasks. Success is policy-compliant completion or correct escalation; incidents are unauthorized external-effect outcomes; cycle time is normalized.}
\label{tab:safe_delegation}
\begin{tabular}{@{}lccccc@{}}
\toprule
Policy & Success & Human/100 & Unauth. & Incident & Cycle \\
\midrule
Manual & $0.961$ & $100.0$ & $0.000$ & $0.000$ & $3.70$ \\
Always ask & $\mathbf{0.982}$ & $100.0$ & $0.000$ & $0.000$ & $3.00$ \\
Fully autonomous & $0.632$ & $\mathbf{0.0}$ & $0.300$ & $0.110$ & $\mathbf{1.75}$ \\
Static risk rule & $0.922$ & $30.0$ & $0.000$ & $0.000$ & $2.08$ \\
\logos evidence/risk & $0.940$ & $39.0$ & $0.000$ & $0.000$ & $2.07$ \\
\bottomrule
\end{tabular}
\end{table}

\subsection{Q6: Reuse Safety and Target-Side Re-Verification}
\label{sec:exp_transfer}

Q6 treats reuse as target-side release evidence rather than automatic transfer. Imported skills, design precedents, and workplace lessons are candidates, not permissions to mutate the live system. In the Q14 operational-control study, direct import accepts all unsafe candidates in the generated unsafe-import family, whereas target-side re-verification observes no unsafe import adoption while retaining useful imports in that same suite. External benchmarks such as SWE-bench, BFCL, and LongMemEval~\citep{jimenez2024swebench,patil2025bfcl,wu2025longmemeval} are therefore used as target-side holdouts for re-verification. Appendix~\ref{sec:exp_transfer_full} records the corresponding boundary: reuse is a proposal mechanism, whereas utility must be established again in the target environment.

\subsection{Q7: Accuracy-Side Verification-Gated Cascade}
\label{sec:exp_cascade}

Q7 asks when a verification-gated cascade improves the accuracy side of routing. The setting compares a cheap floor, a strong reference, and a cascade policy on matched reasoning and code subsets. Figure~\ref{fig:cascade} shows the positive regime: execution-graded code benefits because the verifier can reject cheap-stage failures before escalation. GSM8K is saturated and MATH-500 remains bounded by verifier quality. The operational claim is that a sound task verifier can turn cheap attempts into a useful first stage without making unchecked early accepts.

\begin{table}[tb]
\centering\small
\setlength{\tabcolsep}{3.5pt}\renewcommand{\arraystretch}{1.12}
\caption{Verifier release stress over $200{,}000$ generated release decisions from encoded candidate families and verifier-failure modes. False accepts are unsafe candidates accepted by the verifier stack; coverage is the non-escalated fraction.}
\label{tab:verifier_release}
\begin{tabular}{@{}>{\raggedright\arraybackslash}p{0.26\linewidth}ccccc@{}}
\toprule
Verifier stack & \shortstack{False\\accept} & \shortstack{False\\reject} & Coverage & Escalate & Cost \\
\midrule
Schema only & $0.909$ & $0.000$ & $\mathbf{1.000}$ & $0.000$ & $\mathbf{0.02}$ \\
Deterministic tests & $\mathbf{0.000}$ & $0.179$ & $0.780$ & $0.220$ & $0.18$ \\
Judge only & $0.090$ & $0.099$ & $\mathbf{1.000}$ & $0.000$ & $0.55$ \\
\logos composite & $0.007$ & $\mathbf{0.000}$ & $0.796$ & $0.204$ & $0.34$ \\
\bottomrule
\end{tabular}
\end{table}

For generated stress tables, the effective coverage unit is the family and
failure mode recorded by the harness; large case counts are seed-and-perturbation
expansions, not independent deployments. We therefore use the counts as
denominators for the controlled grid rather than as estimates of production
incident frequency.

A separate cascade stress isolates the verifier assumption itself. Over $200$ seeds and $2{,}000$ generated tasks per seed, the task generator spans fixed failure-mode templates rather than independent deployment incidents. Within that controlled grid, a perfect execution verifier lifts accuracy above the strong fallback by accepting cheap correct answers that the single strong rollout misses, while a noisy verifier preserves most of the gain but admits a measured false-accept mass. An entry-skip stress then tests the experience rule of Eq.~\eqref{eq:cascade_entry}: on dimensions where early low-cost stages rarely pass, adaptive entry keeps accuracy fixed and cuts cascade cost by $20.0\%$ relative to a fixed waterfall.

\begin{table}[tb]
\centering\small
\setlength{\tabcolsep}{4.5pt}\renewcommand{\arraystretch}{1.12}
\caption{Cascade verifier and entry-skip stress. The verifier stress uses $400{,}000$ generated tasks; the entry-skip stress uses another $400{,}000$ tasks with a hard dimension where cheap and middle stages are unreliable.}
\label{tab:cascade_stress}
\begin{tabular}{@{}>{\raggedright\arraybackslash}p{0.35\linewidth}cccc@{}}
\toprule
Condition & Accuracy & Cost & \shortstack{False\\accept} & \shortstack{Mean\\calls} \\
\midrule
Cheap only & $0.646$ & $\mathbf{0.150}$ & -- & $1.00$ \\
Strong only & $0.941$ & $2.500$ & -- & $1.00$ \\
Cascade, perfect verifier & $\mathbf{0.991}$ & $0.638$ & $\mathbf{0.000}$ & -- \\
Cascade, noisy verifier & $0.965$ & $0.596$ & $0.027$ & -- \\
\midrule
Fixed waterfall & $0.956$ & $2.055$ & -- & $2.20$ \\
Adaptive entry & $0.956$ & $\mathbf{1.643}$ & -- & $\mathbf{1.06}$ \\
\bottomrule
\end{tabular}
\end{table}

In Table~\ref{tab:flagship}, pass$^1$ denotes the benchmark's first-trial success criterion over repeated scenario variants, $\dagger$ marks the official $5$-shot MMLU-Pro protocol, and $^{\mathrm{j}}$ marks judge-supported safety scoring rather than deterministic correctness.

\begin{figure}[tb]
 \centering
 \includegraphics[width=0.9\linewidth]{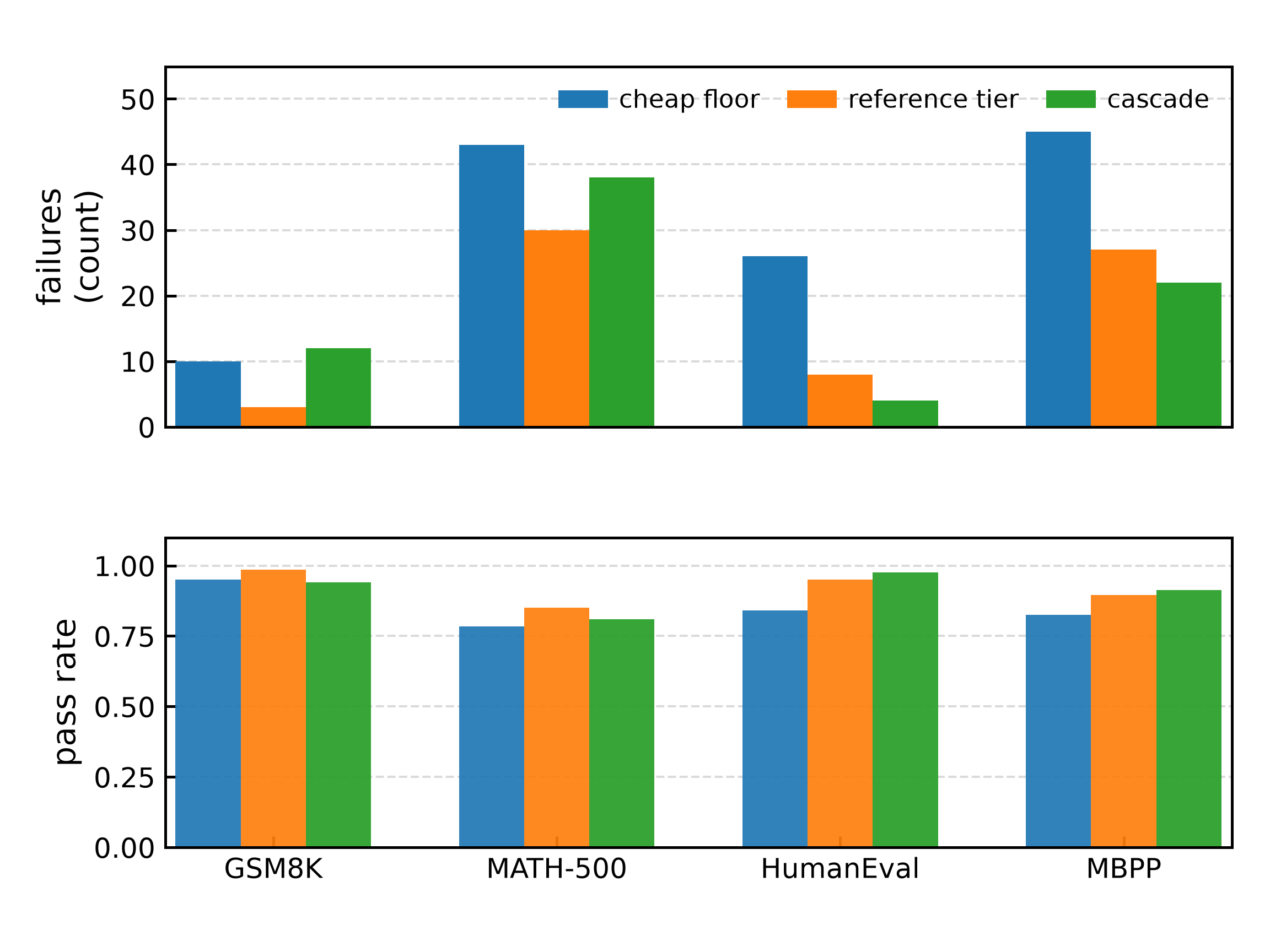}
\caption{Verification-gated cascade summary on matched subsets, with $n$ and paired tests in Appendix~\ref{sec:exp_appendix}. Execution-graded code is the positive regime; gains over the strong single model are not significant.}
 \label{fig:cascade}
\end{figure}

For deployment, the cascade is again a candidate behind the validated-route rule:
it is enabled when validation evidence shows a non-regressing accuracy or
utility trade-off, and otherwise the strong direct route remains the selected
solver. The reasoning-suite null results therefore remain useful diagnostics of
verifier coverage, while the deployed routing policy avoids treating an
unvalidated cascade as a universal replacement for the strong reference.
The Q7 accuracy rows are not equal-token or equal-latency comparisons. Cascade
runs may spend additional calls and verifier executions before accepting an
answer. We therefore interpret the result as a routing frontier under the
recorded harness, with normalized cost and mean-call summaries reported in the
stress and appendix tables, rather than as an equal-budget model comparison.

\subsection{Q8: Collective Routing}
\label{sec:exp_collective}

Q8 tests whether collective routing can expose a better cost--accuracy frontier than a cost-blind router. The held-out replay in Figure~\ref{fig:chorus_frontier} shows the positive regime: a learned head can preserve the strong model's mean accuracy at lower cost when the reward stream identifies cheap successes. To avoid over-interpreting a small replay, Table~\ref{tab:chorus_controlled_frontier} adds a larger controlled replay over $120$ seeds with $900$ training and $900$ held-out tasks per seed. It shows the trade-off directly: the quality-constrained cost-aware policy gives up $1.8$ points of accuracy relative to the always-strong policy while cutting cost by $47.0\%$. The tested $\lambda_{\mathrm{cost}}$ values lie in the same piecewise-constant decision region of the quality-frontier guard, so Table~\ref{tab:chorus_controlled_frontier} reports them as one policy. Routing claims therefore require an applicable verifier or reward stream, and the chosen cost tolerance is a deployment decision.

\begin{table}[tb]
\centering\small
\setlength{\tabcolsep}{5.0pt}\renewcommand{\arraystretch}{1.12}
\caption{Collective-router controlled cost--accuracy replay. Costs are normalized to the same cheap/strong units as the cascade study; oracle is the per-query best worker after observing held-out outcomes.}
\label{tab:chorus_controlled_frontier}
\begin{tabular}{@{}>{\raggedright\arraybackslash}p{0.36\linewidth}cccc@{}}
\toprule
Policy & Accuracy & Cost & \shortstack{Cheap\\pick} & \shortstack{Strong\\pick} \\
\midrule
Single cheap & $0.680$ & $\mathbf{0.150}$ & $1.00$ & $0.00$ \\
Single strong & $0.937$ & $2.500$ & $0.00$ & $1.00$ \\
Collective, $\lambda_{\mathrm{cost}}\in\{0.1,0.3,0.6\}$ & $0.920$ & $1.325$ & $0.50$ & $0.50$ \\
Oracle upper bound & $\mathbf{0.951}$ & -- & -- & -- \\
\bottomrule
\end{tabular}
\end{table}

\begin{figure}[tb]
 \centering
\includegraphics[width=0.85\linewidth]{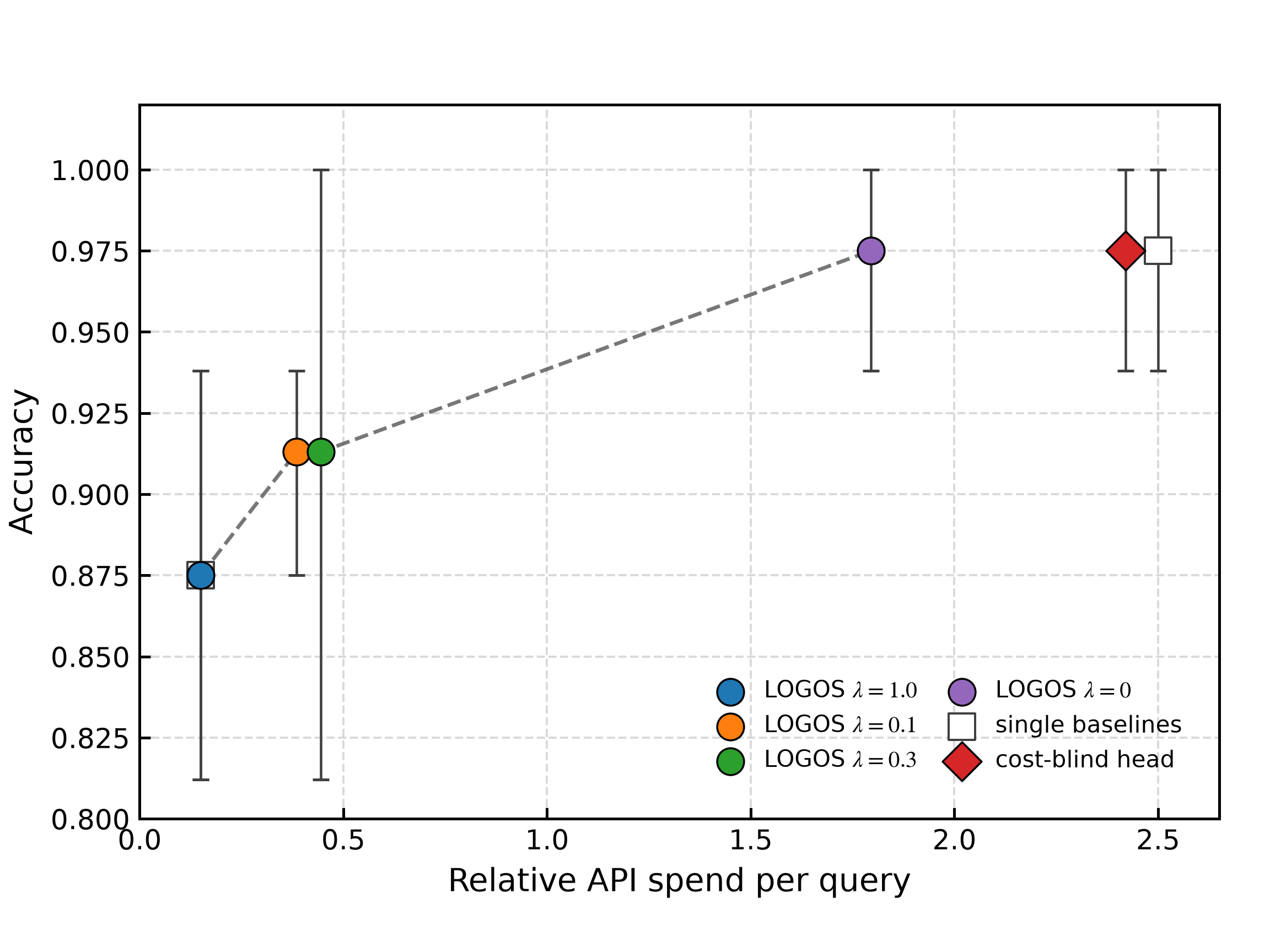}
\caption{Collective-router cost--accuracy frontier on fixed held-out replay.}
 \label{fig:chorus_frontier}
\end{figure}

\subsection{Q9--Q22: Contract and Fault-Injection Suite}
\label{sec:exp_operational}

Q9--Q22 are grouped as a contract and fault-injection suite rather than fourteen independent benchmark claims. Each row uses a negative-control policy that omits the relevant contract rule, then compares it with the \logos policy on the same generated ground truth. The studies cover document and API drift, evidence search, impact analysis, release control, target-side re-checking, root-policy precedence, grounding, tool safety, tenant isolation, memory invalidation, rollback, audit, feedback handling, and workflow composition. These rows are contract-checkable regression tests, not competitive system comparisons, and the raw case counts should be read together with the template families, perturbation rules, seed ranges, and evaluators or oracles that generated and labeled them.

\begin{table}[tb]
\centering\small
\setlength{\tabcolsep}{3.2pt}\renewcommand{\arraystretch}{1.2}
\caption{Operational control studies for Q9--Q22. The fourth column contrasts a negative-control policy with \logos on the primary contract metric; the final column reports a secondary cost, safety, or utility metric. Arrows do not imply that higher is always better: the metric name states the direction, for example harmful adoption, leakage, stale use, fabricated pass, and incident should decrease.}
\label{tab:operational_frontier}
\begin{tabular}{@{}>{\raggedright\arraybackslash}p{0.06\linewidth}
                >{\raggedright\arraybackslash}p{0.17\linewidth}
                >{\raggedright\arraybackslash}p{0.27\linewidth}
                >{\raggedright\arraybackslash}p{0.22\linewidth}
                >{\raggedright\arraybackslash}p{0.18\linewidth}@{}}
\toprule
Q & Surface & Primary metric & Neg. control $\to$ \logos & Secondary metric \\
\midrule
Q9 & document drift & F$_1$ / over-invalidation & $0.857\to\mathbf{1.000}$ & $0.333\to0.000$ \\
Q10 & API drift & F$_1$ / stale-call block & $0.500\to\mathbf{1.000}$ & $0.333\to1.000$ \\
Q11 & evidence search & recall@3 / index work & $0.404\to\mathbf{0.960}$ & $0.000\to0.750$ \\
Q12 & impact analysis & impact F$_1$ / over-invalidation & $0.935\to\mathbf{0.995}$ & $1.000\to0.028$ \\
Q13 & longitudinal release & harmful adoption / utility & $0.400\to\mathbf{0.000}$ & $-0.012\to0.015$ \\
Q14 & target re-check & unsafe import / useful import & $1.000\to\mathbf{0.000}$ & $0.000\to1.000$ \\
Q15 & root policy & policy accuracy & $0.200\to\mathbf{1.000}$ & -- \\
Q16 & grounding & fabricated pass / supported pass & $1.000\to\mathbf{0.000}$ & $0.000\to1.000$ \\
Q17 & tool safety & unsafe execute / prompts per 100 & $0.250\to\mathbf{0.000}$ & $0.0\to21.2$ \\
Q18 & tenant isolation & leakage / recall & $1.000\to\mathbf{0.000}$ & $0.000\to1.000$ \\
Q19 & obsolete memory & stale use / adaptation steps & $0.800\to\mathbf{0.000}$ & $5.000\to1.000$ \\
Q20 & rollback/audit & rollback / tamper detect & $0.740\to\mathbf{1.000}$ & $0.000\to1.000$ \\
Q21 & HITL feedback & unsafe accept / clarify & $1.000\to\mathbf{0.000}$ & $0.000\to1.000$ \\
Q22 & governed workflow & success / incident & $0.824\to\mathbf{0.951}$ & $0.400\to0.000$ \\
\bottomrule
\end{tabular}
\end{table}

Appendix~\ref{sec:exp_operational_details} gives row-level details. These generated suites test contract composition: changed material should update the right evidence, attractive but unsafe releases should be blocked, external effects should be refused or escalated rather than treated as reversible, and human feedback should obey the same precedence rules as other proposals. A $1.000$ or $0.000$ cell means that the encoded rule passed or failed on the generated grid; it is not an independent proof of production safety.

\section{Related Work}
\label{sec:related}

\logos sits at the confluence of several lines of research; this section isolates the points of difference.

\paragraph{Release engineering, policy, and assurance.}
\logos deliberately borrows from continuous delivery, MLOps, policy-as-code,
runtime assurance, software-supply-chain provenance, capability security, and
assurance-case practice. Versioned artifacts, canaries, rollback, audit logs,
approval workflows, and least-privilege credentials are not new by themselves.
The difference is the object under control: an agent team can use tools,
remember outcomes, delegate work, and propose edits to the prompts, memories,
skills, tools, and workflows that shape future behavior. \logos adapts these
engineering traditions to that self-evolving setting by making every learned
artifact a release candidate under a portable trace, verifier, root-policy, and
promotion contract.

\paragraph{Mixed-initiative and adjustable autonomy.}
Mixed-initiative interaction and adjustable autonomy study how humans and
automated systems share control, when systems should ask, and how authority
should shift under uncertainty~\citep{horvitz1999mixed,allen1999mixed,scerri2002adjustable}. \logos inherits that concern but changes the controlled object:
the system is not only choosing an action during a task, it is proposing edits
to prompts, memories, tools, roles, and workflows that will shape future tasks.
The human--agent loop therefore becomes a release contract with scoped events,
root-policy precedence, promotion gates, and audit records, rather than only a
dialogue-management policy.

\paragraph{Agent interoperability and resource-aware infrastructure.}
Recent industrial work emphasizes complementary deployment axes. Google
Agent2Agent focuses on interoperable communication, task state, and capability
discovery across agents~\citep{google2025a2a}. NVIDIA's small-language-model
position paper argues that agentic systems should route work across
resource-appropriate models rather than defaulting to one large model for every
task~\citep{belcak2025slmAgents}. The HyperAgents study examines
self-referential improvement with safety precautions such as sandboxing and
human oversight~\citep{zhang2026hyperagents}. \logos complements these efforts
by treating interoperability, model selection, and self-improvement as
mechanisms that must still pass artifact-level release governance.

\paragraph{Multi-agent systems and frameworks.}
LLM multi-agent frameworks such as AutoGen, MetaGPT, ChatDev, GPTSwarm, OpenHands, and OpenAI Agents provide useful but incompatible abstractions, schemas, and traces~\citep{wu2023autogen,hong2024metagpt,qian2024chatdev,zhuge2024gptswarm,wang2024openhands,openai2025agents}. \logos places compatible execution backends behind the interface of Section~\ref{sec:kernel}. Trace normalization to Eq.~\eqref{eq:trace} supports portable evidence and artifact re-expression, while target-side verification remains mandatory.

\paragraph{Automated design of agentic systems.}
DSPy compiles LM pipelines, ADAS programs agents from an archive, AFlow searches workflows, and \maas samples an agentic supernet per query~\citep{khattab2024dspy,hu2024adas,zhang2024aflow,zhang2025maas}. \logos incorporates workflow search, but its compiler addresses the upstream production problem: synthesizing a probe-checked or fail-closed team, tools, skills, roles, and genome from bounded multimodal operational inputs under the fallback and finite-validation contract of Eqs.~\eqref{eq:fallback}--\eqref{eq:certify}.

\paragraph{Model routing and output verification.}
FrugalGPT, RouteLLM, and cascade analyses study cost-aware deferral~\citep{chen2023frugalgpt,ong2024routellm}. \logos casts its routing regimes as Eq.~\eqref{eq:bg_route_obj} and uses a verifier, not a learned confidence score, as the cascade escalation signal. On verification, LLM-as-judge and trained verifiers assess outputs in isolation~\citep{zheng2023llmjudge,cobbe2021gsm8k,zhang2024genrm}; \logos orders deterministic, learned, and judge checks by typed precedence, strengthens free-form judgment with evidence graphs, and records anytime-valid rate evidence shared by compilation, routing, and evolution.

\paragraph{Adaptive evaluation and reusable holdouts.}
Reusable-holdout and adaptive-data-analysis work shows that repeatedly inspecting held-out failures can invalidate ordinary evaluation~\citep{dwork2015reusable,dwork2015preserving}. \logos adopts the same warning operationally: proposal examples, gate examples, and report holdouts are separated; proposal-safe gate feedback contains aggregate counts and hashes rather than per-example failures; and the optional anytime-valid gate prevents stopping and monitoring from creating a hidden multiple-testing channel under its stated assumptions.

\paragraph{Learned orchestration behind a single interface.}
Graph-based agent search, workflow generation, and query-conditioned architecture search all treat orchestration as an optimization object~\citep{zhuge2024gptswarm,zhang2024aflow,shang2024agentsquare,zhang2025maas}. \logos adds a platform-native collective router: a cost-aware decision head, online bandit updates, routing-margin-gated deep workflows, replica diversity, and verification-gated explanations.

\paragraph{Self-evolving and self-improving agents.}
Self-evolving agents improve from experience through reflection, skill libraries, prompt evolution, evidence-supported curricula, or code self-improvement~\citep{gao2025survey,shinn2023reflexion,wang2024voyager,madaan2023selfrefine,agrawal2025gepa,arai2026eveagent,zhang2025dgm}. \logos differs in making deployment control the acceptance rule: the paired gate adopts only under execution-closed, held-out task-paired evidence with explicit credit assignment. Its textual skill optimization is related to TextGrad, OPRO, PromptBreeder, and GEPA~\citep{yuksekgonul2024textgrad,yang2024opro,fernando2024promptbreeder,agrawal2025gepa}, but \logos embeds that style of optimization in a memory- and HITL-aware, backend-portable, auditable loop.

\paragraph{Memory, human interaction, and oversight.}
Long-lived agents need memory and human oversight~\citep{park2023generative,packer2023memgpt}. \logos combines three-tier value-aware memory, gated promotion, and optional calibrated acquisition with the same release-governance boundary used for skills and workflows. Clarifying-questions and human-in-the-loop benchmarks study when agents should ask, including HiL-Bench's Ask-F1 precision/blocker-recall metric\citep{gan2024clarqllm,suri2025clarifybench,elfeki2026hilbench}; \logos connects asking to pre-run and mid-run selective escalation, directive synthesis, and risk-proportional approval, rollback, and audit. The contribution is the combination of these controls with MAS compilation and adapter-based operation.
\section{Conclusions}
\label{sec:conclusions}

This paper presents \logos, a pluggable governance layer for agent teams whose prompts, tools, memories, roles, and workflows evolve over time. Its central contribution is an evidence-governed release contract: multimodal inputs are compiled into probe-checked Agent Packs, execution is captured through normalized traces, and every proposed change remains isolated until evidence, policy, approval, and audit requirements permit promotion.

\logos makes multi-agent evolution observable and controllable. Agents can request missing information, credentials, or authority, while humans can asynchronously steer, constrain, approve, reject, or pause their actions. These interactions are recorded as governed events with explicit scope, provenance, precedence, expiry, and audit semantics.

The core principle is the separation of evidence from authority. Evaluation can show that a candidate works within a defined scope, but only human-owned policy can authorize its release. This boundary enables continuous improvement without allowing the model to become its own evaluator or release manager. \logos therefore treats adaptation as scoped candidate release: behavior may change after evidence, while authority remains explicit, traceable, and human-owned.

\bibliographystyle{unsrtnat}
\bibliography{ref}

\newpage
\appendix
\section{Appendix Guide and Notation}
\label{sec:notation}

The appendix supports two reading paths. Appendix~\ref{sec:worked_example}
first gives a concrete pack, so readers can see the object before returning to
formal details. Appendices~\ref{sec:bg_ext} and~\ref{sec:method_ext} then
collect derivations and statistical machinery. The remaining appendices record
operational contracts, advanced opt-in mechanisms, and experimental protocol.

\begin{table}[tb]
\centering\small
\setlength{\tabcolsep}{5pt}\renewcommand{\arraystretch}{1.16}
\caption{Appendix reading map.}
\label{tab:appendix_roadmap}
\begin{tabular}{@{}L{0.28\linewidth}L{0.62\linewidth}@{}}
\toprule
Appendix & Use it for \\
\midrule
Guide and notation & Essential symbols and grouped mechanism families. \\
Worked example & A concrete pack, team, genome, gate, and trace-oriented artifact. \\
Extended background & Formal definitions supporting the primer and lifecycle contract. \\
Extended method & Exact update rules, e-processes, and controller details. \\
Operational contract & Event, memory, scheduling, and execution semantics. \\
Advanced mechanisms and experiments & Optional evolution instruments and extended protocol/results. \\
\bottomrule
\end{tabular}
\end{table}

The notation below is deliberately selective. Symbols used only once are left
near their defining equations; Tables~\ref{tab:notation} and~\ref{tab:notation2}
keep only the symbols needed to navigate multiple sections. Bold lowercase
letters denote vectors, and a prime, as in $\mathfrak{D}'$ or $\Pi'$, marks a
candidate version.

\begin{table}[tb]
\centering\small
\setlength{\tabcolsep}{5pt}\renewcommand{\arraystretch}{1.18}
\caption{Core notation used across the paper.}
\label{tab:notation}
\begin{tabular}{@{}L{0.25\linewidth}L{0.67\linewidth}@{}}
\toprule
Symbol & Meaning \\
\midrule
$\Pi$ & Agent team: roles, policies, contexts, tools, shared state, and execution budget. \\
$\mathsf{Pack}$ & Versioned release artifact containing the team, genome, verifiers, probes, and manifest. \\
$G$ & Agent Genome: routing, memory, verification, promotion, safety, sandbox, tool, and attention policies. \\
$R_{\mathrm{root}}$ & Human-owned root policy: objective, evaluators, holdouts, permissions, approvals, effect policy, credentials, and audit sink. \\
$\mathfrak{D}$ & Governed deployment object: pack plus root policy, backend, credentials, external state, and audit sink. \\
$q,\mathcal{S}$ & Natural-language objective and bounded source material supplied to the compiler. \\
$\tau,P_q,\mathcal{H}$ & Task instance, deployment task distribution, and held-out gate set. \\
$\zeta,\chi_{\mathrm{eff}}$ & Normalized execution trace and external-effect record. \\
$s,u,m,J$ & Task score, task utility, gate metric, and expected deployment utility. \\
$\mathfrak{D}',\mathfrak{D}_{+}$ & Candidate deployment and accepted next version, if the gate promotes it. \\
$\nu,B$ & Verifier and its Boolean acceptance decision. \\
\bottomrule
\end{tabular}
\end{table}

\begin{table}[tb]
\centering\small
\setlength{\tabcolsep}{5pt}\renewcommand{\arraystretch}{1.18}
\caption{Mechanism notation used repeatedly in later sections.}
\label{tab:notation2}
\begin{tabular}{@{}L{0.25\linewidth}L{0.67\linewidth}@{}}
\toprule
Symbol & Meaning \\
\midrule
$\mathcal{W},\mathcal{K},B$ & Tool specs, skill specs, and the compiler blueprint. \\
$\rho,a_k,d$ & Routing policy, cascade stage, and task dimension. \\
$\mathcal{M}_{\mathrm{raw}},\mathcal{M}_{\mathrm{val}},\mathcal{M}_{\mathrm{skill}}$ & Raw, validated, and durable skill-memory tiers. \\
$\eta,Q,\mathbf{u}$ & Memory episode, learned value, and semantic embedding. \\
$\theta,P,\beta$ & Skill document, proposed textual patch, and edit budget. \\
$\kappa,\widehat{\Delta}_{\mathcal{H}},R_{\mathcal{H}}$ & Target component, held-out gain, and observed regressions in the paired gate. \\
$K_t,\alpha$ & Anytime-valid evidence wealth and error budget. \\
$\mathcal{Q}_{\mathrm{ask}},\pi_{\text{ask}},\askf$ & Candidate questions, ask policy, and Ask-F$_1$ interaction metric. \\
$\mathbb{W},\mathcal{L}$ & Resident World and its shared lesson store. \\
$\mathcal{K}_u=(\mathcal{D}_u,\Sigma_u,\mathcal{L}_u)$ & Platform knowledge state: design precedents, skill library, and shared lessons. \\
\bottomrule
\end{tabular}
\end{table}

Table~\ref{tab:components} groups the implementation mechanisms by the role
they play in the lifecycle. The main text defines the high-use mechanisms; the
appendix keeps only brief reminders for the longer tail.

\begin{table}[tb]
\centering\small
\setlength{\tabcolsep}{4.5pt}\renewcommand{\arraystretch}{1.16}
\caption{\logos mechanisms grouped by lifecycle role.}
\label{tab:components}
\begin{tabular}{@{}L{0.20\linewidth}L{0.36\linewidth}L{0.34\linewidth}@{}}
\toprule
Lifecycle role & Mechanism family & Purpose \\
\midrule
Build and certify & compiler and rate evidence & Compile source material into a pack and attach finite or anytime tool-call evidence. \\
Route and coordinate & verifier-gated and collective routing & Choose models, agents, and workflows under cost, uncertainty, and verifier feedback. \\
Remember and improve & gated memory, textual skill updates, workflow search & Retrieve useful experience, optimize skills, and search workflows while preserving gate discipline. \\
Gate releases & paired execution, attribution, and anytime-valid gates & Attribute failures, evaluate candidates by held-out execution, and optionally add anytime-valid evidence. \\
Human and workplace control & budgeted questions, directive synthesis, resident state & Ask budgeted or calibrated questions, fuse human directives, and keep resident World-scoped state. \\
Advanced evolution & optional retirement, memory editing, diversity, and tool repair & Retain proven artifacts, edit memory, preserve search diversity, and gate self-referential tool repair. \\
\bottomrule
\end{tabular}
\end{table}

\section{Worked Example: Security-Exception Review}
\label{sec:worked_example}

This appendix makes the abstract objects in Eqs.~\eqref{eq:compile}
and~\eqref{eq:genome} concrete through a security-governance example. It shows
how the compiler output forms a release bundle: a team description,
typed skills and tools, a phase plan, a genome, verifier declarations, and
audit evidence.

\paragraph{The description that creates the team.}
The compiler input is a short English objective plus bounded operational material:
a security-exception policy, a vendor intake form, a ticket export, a risk
matrix, and an OpenAPI description for an exception-management service. The
brief is deliberately operational rather than abstract: review vendor security
exception requests, normalize facts from the intake and ticket state, map
policy evidence, score residual risk, ask a human only for blocking authority
or approval gaps, draft a customer-safe recommendation, and verify schema,
evidence, and safety before any ticket update. The emitted artifact is more
than a prompt. The compiler reads the source types, infers which work should be
split across specialists, assigns tools and authority boundaries, and emits the
release components summarized below.

\paragraph{What an Agent Pack contains.}
An \pack is the portable release unit. Table~\ref{tab:worked_pack_components}
lists the concrete components in this example. The same division is what
lets the runtime inspect, verify, reroute, evolve, and roll back the team as a
software artifact rather than as an opaque conversation.

\begin{table}[tb]
\centering\footnotesize
\setlength{\tabcolsep}{3.5pt}\renewcommand{\arraystretch}{1.12}
\caption{Concrete \pack components in the security-exception example.}
\label{tab:worked_pack_components}
\begin{tabular}{@{}L{0.20\linewidth}L{0.31\linewidth}L{0.39\linewidth}@{}}
\toprule
Component & Example contents & Release-governance role \\
\midrule
Manifest & \texttt{pack.json}: name, version, domain, agents, skills, tools, certification, compiler metadata & makes the team versioned, inspectable, and schema-checkable \\
Source material & policy excerpt, intake form, ticket CSV, risk matrix, OpenAPI schema & bounds what the compiler may rely on and what evidence claims can cite \\
Agents & seven role files under \texttt{agents/} & materialize the initial MAS $\Pi_0$: who acts, verifies, writes, and asks \\
Skills & five progressive \texttt{SKILL.md} playbooks & provide procedural knowledge without forcing every detail into every prompt \\
Tools & \texttt{csv\_preview}, \texttt{safe\_sql}, \texttt{ask\_user}, \texttt{http\_request} & expose typed capabilities with read/write, network, and human-input policy \\
Execution plan & \texttt{execution\_plan.json} with intake, policy map, risk score, human review, drafting, verification & turns the task into an auditable phase DAG rather than an unstructured chat \\
Genome & routing, memory, verifier, promotion, safety, sandbox, tools, attention budget & governs runtime behavior and what may evolve \\
Goal state & \texttt{goal.json} & records the objective, status, and evolution trigger for resident operation \\
Evidence artifacts & provenance, trace excerpt, operator questions, gate evidence record, example output & provide audit records without leaking final-gate examples to future proposals \\
\bottomrule
\end{tabular}
\end{table}

Table~\ref{tab:worked_sources} shows how this single description becomes
typed pack material. The compiler emits a versioned release artifact, not only
a prompt: its manifest, agents, skills, tools, verifiers,
probes, and policies are inspectable.

\begin{table}[tb]
\centering\small
\setlength{\tabcolsep}{4pt}\renewcommand{\arraystretch}{1.12}
\caption{Worked example sources and the pack elements they induce.}
\label{tab:worked_sources}
\begin{tabular}{@{}>{\raggedright\arraybackslash}p{0.27\linewidth}
                >{\raggedright\arraybackslash}p{0.31\linewidth}
                >{\raggedright\arraybackslash}p{0.32\linewidth}@{}}
\toprule
Input material & Compiler interpretation & Emitted pack artifact \\
\midrule
Policy excerpt & required facts, decision rules, authority boundary & policy-mapping skill; verifier evidence protocol \\
Vendor intake form & concrete request facts and missing fields & intake phase; risk-score inputs \\
Ticket CSV & structured operational state & CSV/SQL-readable tools and intake cross-checks \\
Risk matrix & deterministic risk tier and downgrade rules & risk-scoring skill; semantic check on risk application \\
OpenAPI document & draft-vs-submit effect boundary & draft payload schema; submit approval is human-gated \\
\bottomrule
\end{tabular}
\end{table}

\paragraph{Which agents are created.}
The selected blueprint is a specialist team rather than a single reviewer
because the task combines extraction, policy interpretation, risk scoring,
human authority, drafting, and independent verification. Table~\ref{tab:worked_agents}
lists the concrete agents, the signal that caused the compiler to create them,
and the output each one is expected to produce.

\begin{table}[tb]
\centering\footnotesize
\setlength{\tabcolsep}{3.2pt}\renewcommand{\arraystretch}{1.12}
\caption{Agents emitted for the security-exception pack.}
\label{tab:worked_agents}
\begin{tabular}{@{}L{0.18\linewidth}L{0.22\linewidth}L{0.30\linewidth}L{0.20\linewidth}@{}}
\toprule
Agent & Creation signal & Responsibility & Output \\
\midrule
Team lead & multi-phase workflow with effect boundary & dispatches phases, tracks blockers, prevents unauthorized effects & phase assignments and final status \\
Intake analyst & intake form plus ticket CSV & normalizes vendor, system, data class, exception, dates, owners, and missing facts & structured request facts \\
Policy mapper & policy excerpt and citation requirement & maps facts to clauses, controls, and evidence spans & cited policy map \\
Risk assessor & risk matrix and compensating controls & computes inherent and residual risk and approval route & risk tier and rationale \\
Control owner & missing authority and human approval & asks only blocking authority or approval questions & scoped human question or approval record \\
Response writer & customer-facing recommendation & drafts a conditional decision and API draft payload & decision text and draft payload \\
Verifier & schema, evidence, and safety gates & checks payload schema, evidence support, and external-effect policy & pass/fail verdict and evidence fields \\
\bottomrule
\end{tabular}
\end{table}

\paragraph{Genome and other runtime features.}
The manifest's genome selects an orchestrator-subagent topology, a
verification-gated cascade over model tiers, value-aware memory, a budgeted ask
policy, a schema/evidence/security verifier stack, and a promotion rule that
requires human review and safety non-regression for release-sensitive changes.
If labeled calibration data are available, the ask policy can be replaced by a
calibrated questioning policy. The relevant genome fields declare an
orchestrator-subagent routing pattern, a verifier-gated model-tier cascade, a
memory write policy that waits for evaluation success, a human-input tool under
a budgeted ask policy, required payload/evidence/security verifiers, and a
safety rule that treats approval submission as a human-authorized effect.

The execution plan gives the team an auditable DAG. The \texttt{control\_review}
phase has \texttt{requires\_human=true}, so missing authority does not become a
model guess. A typical operator question in this example is: ``Who is the
authorized security owner for ticket SEC-1042, and may they approve a
time-bounded compensating-control exception?'' If the answer is absent, the
team can still draft a recommendation, but the external approval effect remains
blocked.

\begin{table}[tb]
\centering\footnotesize
\setlength{\tabcolsep}{3.5pt}\renewcommand{\arraystretch}{1.12}
\caption{Additional \logos features illustrated by the same pack.}
\label{tab:worked_features}
\begin{tabular}{@{}L{0.20\linewidth}L{0.35\linewidth}L{0.35\linewidth}@{}}
\toprule
Feature & Concrete example & Why it matters \\
\midrule
Routing cascade & extraction can start on a cheap tier; policy reasoning escalates when verification or uncertainty requires it & cost is saved only when evidence supports cheap-stage acceptance \\
Human questioning & the control-owner agent asks for missing approval authority, capped by the attention budget & the team asks for authority instead of fabricating it \\
Verifier stack & payload schema, evidence protocol, and security-regression checks all block release & schema conformance alone cannot approve the decision \\
Effect envelope & draft ticket update is allowed; approval submission prompts for human authorization & structural rollback is not confused with undoing an external effect \\
Trace artifact & route choices, observations, question events, verifier results, and final status are normalized & debugging and later evolution use portable evidence \\
Gate evidence record & the example record stores task-set hash, verifier family, decision, and safety outcome & future skill or workflow edits remain release candidates until gated \\
\bottomrule
\end{tabular}
\end{table}

\paragraph{Compile--operate--evolve walkthrough.}
The example is a minimal end-to-end contract trace. It shows which artifacts
are present when the same pack is exercised under the compile--operate--evolve
lifecycle.

\begin{table}[tb]
\centering\footnotesize
\setlength{\tabcolsep}{3.5pt}\renewcommand{\arraystretch}{1.12}
\caption{End-to-end release contract illustrated by the security-exception pack.}
\label{tab:worked_lifecycle}
\begin{tabular}{@{}L{0.16\linewidth}L{0.32\linewidth}L{0.42\linewidth}@{}}
\toprule
Step & Artifact or event & Governance check \\
\midrule
Compile & \texttt{pack.json}, role files, skills, tool declarations, execution plan & schemas validate, tools declare effects, verifier and probe metadata are present \\
Operate & normalized trace excerpt with route, observation, question, verifier, and outcome events & operator sees evidence-bearing events rather than hidden reasoning; missing authority becomes a question \\
Detect failure & trace records a rejected draft update after an API/schema mismatch & failure is attributed to a typed component, not immediately patched into the live pack \\
Propose & candidate tool-schema or skill edit with patch provenance & candidate is isolated under the same root policy and scoped credentials \\
Gate & paired baseline/candidate evaluation and \texttt{gate\_evidence.json} & promotion requires held-out gain, regression budget, safety check, and approval where required \\
Promote or reject & accepted pack version or unchanged baseline plus audit event & rollback restores structural artifacts; external effects remain governed separately \\
\bottomrule
\end{tabular}
\end{table}

This worked example instantiates the paper's central invariant. The team may
draft, ask, route, remember, and propose changes, but the live pack, approval
policy, credentials, evaluator, final holdout, and external effects remain
governed by evidence and human-owned authority.

\section{Extended Background Formulations}
\label{sec:bg_ext}

This appendix collects the formal material deferred from Section~\ref{sec:background}: definitions and derivations that sharpen the toolkit of that section but are not required to follow the main text.

\paragraph{Trace space.}
Each event of a trace, Eq.~\eqref{eq:trace}, is a tagged record rather than a full product tuple. The event universe is the tagged union
\begin{equation}
 \begin{aligned}
 \mathcal{E}_{\mathrm{tagged}} ={}&
 \mathcal{E}_{\mathrm{task}}\cup\mathcal{E}_{\mathrm{plan}}\cup
 \mathcal{E}_{\mathrm{action}}\cup\mathcal{E}_{\mathrm{observation}}\cup
 \mathcal{E}_{\mathrm{route}}\cup\mathcal{E}_{\mathrm{verdict}}\cup{}\\
 &\mathcal{E}_{\mathrm{human}}\cup\mathcal{E}_{\mathrm{outcome}}\cup
 \mathcal{E}_{\mathrm{cost}}\cup\mathcal{E}_{\mathrm{risk}}.
 \end{aligned}
 \label{eq:event_space}
\end{equation}
where each member carries its own schema, timestamp, actor/backend, correlation id, and optional hashes. The trace space of Section~\ref{sec:bg_mas} is the set of all finite event sequences, $\mathcal{Z}=\bigcup_{K=1}^{K_{\max}}\mathcal{E}_{\mathrm{tagged}}^{K}$, with declared maximum trace length $K_{\max}$.

\paragraph{Bandit regret.}
The standard yardstick for a router solving the contextual bandit of Eq.~\eqref{eq:bg_route_obj} is its \emph{regret} after $T$ rounds,
\begin{equation}
 \mathrm{Reg}_T = \sum_{t=1}^{T}\Bigl(\max_{a\in\mathcal{A}_{\mathrm{route}}}\mathbb{E}\bigl[U(\tau_t,a)\bigr] - \mathbb{E}\bigl[U(\tau_t,a_t)\bigr]\Bigr),
 \label{eq:bg_regret}
\end{equation}
the cumulative utility it forgoes relative to always playing the per-context best action; an algorithm learns in the bandit sense when $\mathrm{Reg}_T$ grows sublinearly in $T$, so its per-round loss vanishes.

\paragraph{Cascade cost condition and false-accept bound.}
Consider the verification-gated cascade of Section~\ref{sec:bg_routing} with stages $a_1\prec\cdots\prec a_K$ and per-stage costs $c_j(\tau)$. Let $J(\tau)$ be the first stage whose output the verifier accepts, setting $J(\tau)=K$ if every cheaper stage is rejected, and let $s(d_\tau)$ be the learned entry stage. The realized cost is
\begin{equation}
 C_{\mathrm{cascade}}(\tau)=\sum_{\ell=s(d_\tau)}^{J(\tau)}c_\ell(\tau).
 \label{eq:cascade_realized_cost}
\end{equation}
Expected cost improves over always running $a_K$ exactly when
\begin{equation}
 \mathbb{E}\left[\sum_{\ell=s(d_\tau)}^{J(\tau)}c_\ell(\tau)\right]
 < \mathbb{E}\left[c_K(\tau)\right].
 \label{eq:cascade_cost}
\end{equation}
The probability-times-cost expression sometimes used for intuition is valid only when per-stage costs are constant or replaced by the corresponding conditional expectations. A cascade is therefore a conditional improvement over always running the fallback $a_K$: at least as good on accuracy under verifier soundness, state isolation, and fresh-fallback semantics, and strictly better on expected cost precisely when Eq.~\eqref{eq:cascade_cost} holds. A broader Pareto claim over any fixed single action requires that $a_K$ weakly dominate that action on accuracy under the same task distribution. The entry rule of Section~\ref{sec:cascade} shrinks overhead by skipping stages whose posterior success probability is too low to justify attempting.

\begin{proposition}[Cascade false-accept bound under stated assumptions]
Under the assumptions in Section~\ref{sec:cascade}, the verification-gated cascade has accuracy at least that of always executing $a_K$. With an imperfect verifier, the excess wrong-answer probability is bounded by
\begin{equation}
 \Pr(\text{wrong early accept})
 =
 \sum_{j<K}\Pr(\operatorname{reach}_j,\neg\mathrm{ok}_j,\nu_j=1)
 \le
 \sum_{j<K}\Pr(\operatorname{reach}_j,\neg\mathrm{ok}_j)
 \alpha^{\mathrm{FA}}_{j,\mathrm{reach}},
 \label{eq:cascade_false_accept_bound}
\end{equation}
where
$\alpha^{\mathrm{FA}}_{j,\mathrm{reach}}=
\Pr(\nu_j=1\mid \operatorname{reach}_j,\neg\mathrm{ok}_j)$.
\end{proposition}
\begin{proof}
If an early stage is accepted, soundness gives a correct answer. If no early stage is accepted, state isolation and the fresh-fallback assumption make the returned output distributed exactly as $a_K$ on the same task, so the cascade cannot lose accuracy relative to $a_K$. For an imperfect verifier, the only additional error path is reaching an early stage, producing an incorrect output, and accepting it; summing those disjoint first-accept events over $j<K$ gives the displayed bound.
\end{proof}

\paragraph{Coverage proxy and platform value.}
Actual transfer gain can be negative, so the platform does not model Eq.~\eqref{eq:bg_transfer_gain} as a non-negative monotone function. Instead it uses a separate coverage proxy for roadmap planning,
\begin{equation}
\begin{aligned}
 V_{\mathrm{cov}}(\mathcal{K})
 &= \mathbb{E}_{q\sim P_{\mathcal{Q}}}\bigl[h\bigl(\mathrm{cov}(q;\mathcal{K})\bigr)\bigr],\\
 \mathrm{cov}(q;\mathcal{K})
 &=
 \begin{cases}
 0, & \mathcal{D}=\varnothing,\\
 \displaystyle\max_{(q_i,\cdot,\cdot)\in\mathcal{D}} \simop\bigl(\mathbf{e}(q),\mathbf{e}(q_i)\bigr), & \text{otherwise.}
 \end{cases}
\end{aligned}
 \label{eq:bg_coverage}
\end{equation}
where $\mathcal{D}$ is the precedent component of $\mathcal{K}=(\mathcal{D},\Sigma,\mathcal{L})$, $\mathbf{e}(\cdot)$ embeds a brief into the similarity space of Section~\ref{sec:bg_memory}, $P_{\mathcal{Q}}$ is a planning distribution over future briefs, and $h:[0,1]\to\mathbb{R}_{\ge0}$ is non-decreasing with $h(0)=0$. This $V_{\mathrm{cov}}$ is not the expected actual transfer gain; it is an option-value proxy for covering more of the use-case space. Monotonicity and submodularity follow for the proxy because adding a precedent can only raise the maximum in $\mathrm{cov}$. Actual transfer remains gated target-side and may be negative.

\paragraph{Roadmap selection as an acquisition problem.}
The same objects turn a passive archive into an active planner. Given a candidate set $\mathcal{Q}_{\mathrm{cand}}$ of prospective use cases, the natural next-deployment rule maximizes immediate utility plus the marginal platform value the deployment would add,
\begin{equation}
 q^{(u+1)} = \argmax_{q\in\mathcal{Q}_{\mathrm{cand}}}
 \hat{J}\bigl(q;\mathcal{K}_u\bigr)
 +
 \lambda_V\Bigl[V_{\mathrm{cov}}\bigl(\mathcal{K}_u \oplus q\bigr) - V_{\mathrm{cov}}\bigl(\mathcal{K}_u\bigr)\Bigr],
 \label{eq:bg_roadmap}
\end{equation}
where $\hat{J}(q;\mathcal{K}_u)$ predicts the deployment utility of $q$ from its nearest precedents, namely the realized-quality records $p_i$ in $\mathcal{D}_u$, $\mathcal{K}_u\oplus q$ denotes the knowledge state after deploying $q$, and $\lambda_V\ge0$ prices coverage option value against immediate return. If $\hat J$ is flat, deployment costs are equal, a cardinality constraint is fixed, and exact marginal values of the monotone submodular proxy are available, greedy selection inherits the classical $(1-\nicefrac{1}{e})$ coverage approximation~\citep{nemhauser1978submodular}. Those are conditions for the proxy objective, not a guarantee of actual transfer utility. The ledger stores the triples $(q_i,B_i,p_i)$ that $\hat J$ and $\mathrm{cov}$ require, while target-side gates decide whether the resulting candidates are useful.

\paragraph{Confidence sequences.}
Dual to a sequential test is a \emph{confidence sequence}, abbreviated CS: a sequence of intervals $[\mathrm{LCB}_t,\mathrm{UCB}_t]$ that covers an unknown parameter $\mu$ uniformly over time~\citep{howard2021confidence,waudbysmith2024betting},
\begin{equation}
 \Pr\bigl(\exists t: \mu \notin [\mathrm{LCB}_t,\mathrm{UCB}_t]\bigr) \le \alpha.
 \label{eq:cs}
\end{equation}
Unlike a fixed-$n$ confidence interval, a CS may be inspected at every $t$ and still enjoys simultaneous $(1-\alpha)$ coverage. For bounded observations $X_i\in[0,1]$, one can obtain a CS by a union bound over a Hoeffding tail with a summable per-step level $\sum_t\alpha_t=\alpha$, or, more tightly, by inverting a betting e-process.

\paragraph{Online false-adoption accounting.}
A long-lived deployment makes not one but many adoption decisions, and even a per-decision error $\le\alpha$ can leave many false adoptions over a lifetime. LORD-family procedures, named for Levels based On Recent Discovery~\citep{javanmard2018online,ramdas2018saffron}, treat the significance budget as renewable \emph{alpha-wealth}. \logos uses the ledger as conservative online accounting unless the stronger conditions for formal LORD++ control are explicitly satisfied. Let $\tau_1<\tau_2<\cdots$ be the past rejection/adoption times. One standard LORD++ spending level for the $k$-th decision is
\begin{equation}
 a_k =
 w_0\gamma_k
 +(\alpha-w_0)\gamma_{k-\tau_1}\indic{\tau_1<k}
 +\alpha\sum_{\substack{j\ge2:\\ \tau_j<k}}\gamma_{k-\tau_j},
 \label{eq:lord_bg}
\end{equation}
where $w_0\in(0,\alpha)$, $(\gamma_j)_{j\ge1}$ is non-negative with $\sum_j\gamma_j=1$, and terms with non-positive indices or unavailable $\tau_j$ are omitted. Because each rejected null recycles testing wealth, the spendable level does not collapse to zero over a long loop. The statistical interpretation is only as strong as the candidate-level p-values and the dependence assumptions; exposed holdouts, adaptive reuse, or invalid verifier p-values make the ledger an audit heuristic rather than a formal FDR statement.

\section{Extended Method Formulations}
\label{sec:method_ext}

This appendix collects the detailed constructions deferred from the method sections: the exact bandit and portfolio formulas whose qualitative behavior the main text describes.

\paragraph{Cascade stage value.}
The optimistic verifier-pass estimate used by the entry rule of Eq.~\eqref{eq:cascade_entry} is the Beta posterior mean of Eq.~\eqref{eq:cascade_post} plus $c_{\mathrm{ucb}}$ posterior standard deviations, clipped to $[0,1]$, with exploration constant $c_{\mathrm{ucb}}>0$:
\begin{equation}
 \UCB(d,k) =
 \min\left\{1,
 \frac{\alpha_{d,k}}{\alpha_{d,k}+\beta_{d,k}} + c_{\mathrm{ucb}}\sqrt{\frac{\alpha_{d,k}\beta_{d,k}}{(\alpha_{d,k}+\beta_{d,k})^2(\alpha_{d,k}+\beta_{d,k}+1)}}\right\}.
 \label{eq:cascade_ucb}
\end{equation}
This posterior estimates future verifier-pass probability under the same dimensioning rule and verifier, not correctness itself. Accuracy claims therefore still rely on the verifier soundness or measured false-accept assumptions in the cascade lemma.
With an imperfect verifier the cascade's accuracy guarantee degrades by the reach-conditioned false-accept mass:
\begin{equation}
 \Pr(\text{wrong early accept})
 =
 \sum_{j<K}\Pr(\operatorname{reach}_j,\neg\mathrm{ok}_j,\nu_j=1)
 \le
 \sum_{j<K}\Pr(\operatorname{reach}_j,\neg\mathrm{ok}_j)
 \alpha^{\mathrm{FA}}_{j,\mathrm{reach}},
 \label{eq:cascade_false_accept_bound_app}
\end{equation}
where $\alpha^{\mathrm{FA}}_{j,\mathrm{reach}}=\Pr(\nu_j=1\mid\operatorname{reach}_j,\neg\mathrm{ok}_j)$. This is why the verifier stack is ordered deterministic-first: the dominant reached stages should have measured, preferably near-zero, false-accept mass.

\paragraph{Collective-router online update.}
Each deployed outcome $(x,w_j,r)$ updates the selection head by the same variance-reduced score-function rule as Eq.~\eqref{eq:onlinerouter}, applied to the context row $\kappa(x)$ with a guarded utility. Let $\widehat r_{\max,\kappa}$ be the best running reward estimate in the context and let $g_j=\indic{\widehat r_j\ge \widehat r_{\max,\kappa}-\epsilon_q}$ mark workers still on the quality frontier. The cost-adjusted utility is used only for those workers; outside the frontier, the update applies a quality penalty rather than allowing price alone to increase selection:
\begin{equation}
 u_j =
 \begin{cases}
 r-\lambda_{\mathrm{cost}}\tilde c_j, & g_j=1,\\
 r-\lambda_{\mathrm{cost}}\tilde c_j-(\widehat r_{\max,\kappa}-\epsilon_q-\widehat r_j), & g_j=0.
 \end{cases}
 \label{eq:chorus_guarded_utility}
\end{equation}
With per-context running-mean baseline $b_\kappa$, and a per-update clip that bounds how far one noisy outcome can move the policy, the update is
\begin{equation}
 z_i \leftarrow z_i + \operatorname{clip}\Bigl(\eta_{\mathrm{lr}}(u_j-b_\kappa)\bigl(\indic{i=j}-\pi_i\bigr)/T_{\mathrm{soft}}, \pm s_{\max}\Bigr).
 \label{eq:chorus_online}
\end{equation}
Here $T_{\mathrm{soft}}>0$ is the selection temperature from Eq.~\eqref{eq:chorus_policy}, and $s_{\max}>0$ is the configured per-update step clip.

\paragraph{Calibrated memory acquisition.}
When the calibrated-acquisition memory mode is enabled, Phase~B of gated-memory recall replaces the $z$-score blend with a calibrated Bayesian acquisition. A logistic calibration maps cosine similarity to a prior mean $\rho_\eta=\sigma\bigl((\simop(\mathbf{q},\mathbf{u}_\eta)-c_0)/T\bigr)$ with midpoint $c_0$ and temperature $T$; a conjugate update then forms the value posterior of episode $\eta$ from its $n_\eta$ observed rewards, and episodes are ranked by the upper-confidence acquisition score
\begin{equation}
\begin{aligned}
 \mathrm{acq}(\eta)
 &= \rho_\eta\bigl(\mu_\eta^{\text{post}} + \beta_{\mathrm{acq}}\sigma_\eta^{\text{post}}\bigr),\\
 \mu_\eta^{\text{post}}
 &= \frac{\kappa\rho_\eta + n_\eta\hat r_\eta/\sigma_\varepsilon^2}{\kappa + n_\eta/\sigma_\varepsilon^2},\\
 \sigma_\eta^{\text{post}2}
 &= \nicefrac{1}{(\kappa + n_\eta/\sigma_\varepsilon^2)}.
\end{aligned}
 \label{eq:calm_acq}
\end{equation}
with prior precision $\kappa$, noise variance $\sigma_\varepsilon^2$, normalized historical reward $\hat r_\eta\in[0,1]$, and exploration weight $\beta_{\mathrm{acq}}$. Selecting the top-$k_2$ by Eq.~\eqref{eq:calm_acq} is deterministic, eliminating the $\varepsilon$-greedy noise of the baseline while directing exploration toward high-uncertainty, high-relevance memories.

\paragraph{Trust-region controller.}
The edit budget $\beta_t$ is governed by a trust-region radius $\delta_t \in [\beta_{\min}, \beta_{\max}]$ adapted from the realized-vs-predicted improvement ratio
\begin{equation}
\begin{aligned}
\widehat g_t
&= \hat{g}(\theta_{t-1}, \tilde P_t),\\
\varrho_t
&=
\frac{s(\theta_t) - s(\theta_{t-1})}{\max\{\widehat g_t,g_{\min}\}},\\
 \delta_{t+1} =
 \begin{cases}
 \min(\delta_t \cdot c_{\text{exp}}, \beta_{\max}), & \varrho_t \ge \eta_{\text{hi}} \\
 \max(\delta_t \cdot c_{\text{shr}}, \beta_{\min}), & \varrho_t \le \eta_{\text{lo}}~\text{or}~\widehat g_t\le0 \\
 \delta_t, & \text{otherwise},
 \end{cases}
\end{aligned}
 \label{eq:storm}
\end{equation}
where $g_{\min}>0$ prevents division by zero and non-positive predicted gain forces shrinkage. The validation score $s(\theta_t)$ is computed under the declared scorer. The predicted gain $\hat{g}(\theta_{t-1},\tilde P_t)=(\bar g/\bar m)m_t$ is a first-order surrogate: the exponential moving averages $\bar g,\bar m$ of past realized gains and step sizes scaled by the current step size $m_t$. The step size must be in the same edit-budget units as $\delta_t$: it is an edit count by default, or an embedding displacement mapped to calibrated budget units when a live embedding model is available. The constants $c_{\text{exp}}>1$ and $c_{\text{shr}}<1$ are expand/shrink factors with ratio thresholds $\eta_{\text{hi}},\eta_{\text{lo}}$. A proximal term guards against semantic drift: writing $D_t=\sum_{s\le t} m_s$ for the cumulative displacement from the original skill embedding, the budget actually issued shrinks with accumulated drift,
\begin{equation}
 \beta_{t+1} = \operatorname{clip}\Bigl(\Bigl\lfloor \tfrac{\delta_{t+1}}{1+c_{\text{drift}}D_t}\Bigr\rceil, \beta_{\min}, \beta_{\max}\Bigr),
 \label{eq:storm_prox}
\end{equation}
with drift coefficient $c_{\text{drift}}\ge 0$, so the controller becomes conservative as the skill drifts far from its origin even when recent ratios stay favorable.

\subsection{Formal Guarantee for the Optional Anytime-Valid Gate and LORD}
\label{sec:sage_lord_formal}

\paragraph{Guarantee statement.}
The optional anytime-valid layer supplies an adoption test only under its stated experimental contract. The proposal and any candidate-selection rule must be fixed before the final gate stream begins; final-gate examples and per-example failures must remain hidden from the proposal generator; paired diffs must be bounded or normalized; and bets must be predictable with respect to the filtration $\mathcal{F}_{t-1}$ generated by previous gate examples, outputs, verifier results, and bets.

The gate stream must be fresh i.i.d. samples from the declared gate distribution, samples drawn with replacement from a fixed gate pool, or a fixed-sample evaluation with no sequential claim. Reading a fixed finite holdout without replacement does not by itself satisfy the i.i.d. conditional-mean null used below. Under the valid streaming assumptions, each gain or regression wealth process below is a non-negative supermartingale under its null. Ville's inequality gives $\Pr(\sup_t K_t\ge 1/\alpha)\le\alpha$ at any stopping time, so the e-value $K_t$ is converted to a valid p-value as $p_t=\min\{1,1/K_t\}$.

The adoption decision uses an intersection-union rule: the composite p-value is the maximum of the gain, regression, and optional safety p-values; simultaneous contradictory stream crossings produce \textsc{conflict} and require diagnostic review. Across candidates, the LORD-style ledger spends alpha-wealth on these composite p-values. Formal LORD++ error control requires conditionally super-uniform candidate-level p-values, predictable spending, and the standard dependence assumptions; otherwise the ledger is an audit account rather than an FDR guarantee. This guarantee does not cover verifier false accepts, exposed or adaptively reused holdout examples, adaptive overfitting to repeated aggregate feedback, distribution shift beyond the sampled gate, or unbounded utility scales.

\paragraph{Portfolio e-process for compiler and adoption gates.}
The building block is a plug-in-anchored portfolio e-process for a one-sided mean null $H_0{:} \mathbb{E}[Y]\le m_0$ with $Y\in[0,1]$. Rather than committing to a single bet fraction, it mixes an adaptive aGRAPA plug-in~\citep{waudbysmith2024betting} as an anchor with a grid of fixed fractions and reports the weighted-average wealth
\begin{equation}
 K_t = w\underbrace{\prod_{s\le t}\bigl(1+\lambda_s^{\text{ag}}(Y_s-m_0)\bigr)}_{\text{anchor}} + \nicefrac{(1-w)}{|\mathcal{G}|}\sum_{j\in\mathcal{G}}\prod_{s\le t}\bigl(1+\lambda^{(j)}(Y_s-m_0)\bigr),
 \label{eq:portfolio}
\end{equation}
where $w\in[0,1]$ is the anchor weight and $\mathcal{G}$ is the grid of fixed bet fractions $\lambda^{(j)}\in[0,c/m_0]$, capped by a margin $c\in(0,1)$ below the wealth-bankruptcy fraction $\nicefrac{1}{m_0}$. The adaptive anchor bet is
\begin{equation}
 \lambda_s^{\text{ag}}=
 \operatorname{clip}\left(
 \frac{\bar Y_{s-1}-m_0}{\max\{\widehat{\mathrm{var}}_{s-1},v_{\min}\}},
 0,\frac{c}{m_0}
 \right),
 \label{eq:agrapa_bet}
\end{equation}
formed only from past data, with variance floor $v_{\min}>0$ and a zero bet before the warm-up count $n_{\min}^{\mathrm{bet}}$. The wealth $K_t$ is a non-negative supermartingale under $H_0$, being a convex mixture of supermartingales, so $\Pr_{H_0}(\sup_t K_t\ge \nicefrac{1}{\alpha})\le\alpha$ by Ville's inequality; by the same mixture principle underlying universal portfolios~\citep{cover1991universal}, the fixed-fraction component competes with the best declared grid fraction up to the grid and mixture penalty, making it more robust than a single plug-in bet. A lower-tail test of $H_0{:}\mathbb{E}[Y]\ge m_0$ is implemented by applying the same construction to $\widetilde Y=1-Y$ with null threshold $1-m_0$, which is the convention used for compiler-callability refutation.

The GAIN stream of the optional adoption gate applies the affine rescaling $Y_t=(d_t+1)/2\in[0,1]$ of the bounded per-task diff $d_t\in[-1,1]$, with null mean $m_0=(\delta_{\min}+1)/2$. The benchmark gates use score diffs, which already live in this range; a full-utility deployment must declare a bounded normalization before using this process. Because the map is linear rather than clipped, it preserves the conditional null exactly, $\mathbb{E}[d_t\mid\mathcal{F}_{t-1}]\le\delta_{\min}\iff\mathbb{E}[Y_t\mid\mathcal{F}_{t-1}]\le m_0$; a clipped rescaling would break this equivalence since $\mathbb{E}[\operatorname{clip}(X)]\ne\operatorname{clip}(\mathbb{E}[X])$ in general. The REG stream bets on the non-regression indicator $Z_t=\indic{d_t\ge -\epsilon_{\text{reg}}}$ with null mean $1-\rho_{\max}$, where $0<\rho_{\max}<1$ is the tolerated regression rate, the streaming analogue of the integer cap $R_{\max}$ in the fixed-set gate of Eq.~\eqref{eq:echo_gate}. A strict sample gate with $R_{\max}=0$ is therefore distinct from a population claim of zero regression probability. Both deployed processes are the portfolio wealths of Eq.~\eqref{eq:portfolio} instantiated on their respective streams:
\begin{equation}
 \begin{aligned}
 \textbf{GAIN:}
 H_{0,t}^{\mathrm{gain}}{:} &\mathbb{E}[d_t\mid\mathcal{F}_{t-1}] \le \delta_{\min},\\
 K_t^{\mathrm{gain}} ={}& K_t^{\mathrm{portfolio}}(Y_{1:t},m_0).
 \end{aligned}
 \label{eq:sage_gain}
\end{equation}
\begin{equation}
 \begin{aligned}
 \textbf{REG:}
 H_{0,t}^{\mathrm{reg}}{:} &\Pr(d_t<-\epsilon_{\mathrm{reg}}\mid\mathcal{F}_{t-1}) \ge \rho_{\max},\\
 K_t^{\mathrm{reg}} ={}& K_t^{\mathrm{portfolio}}(Z_{1:t},1-\rho_{\max}).
 \end{aligned}
 \label{eq:sage_reg}
\end{equation}
Here $K_t^{\mathrm{portfolio}}(\cdot,\cdot)$ denotes Eq.~\eqref{eq:portfolio} with its anchor, grid, clipping, variance floor, and warm-up conventions.

\paragraph{Growth-loop stopping rule.}
Let $N_t$ be the number of counted cycles up to cycle $t$, counting only cycles with a recorded binary outcome, and write $N_u$ analogously for prefixes ending at cycle $u$. Let $\mathcal{U}^{\mathrm{win}}_t$ be the index set of the last $\min(N_{\mathrm{win}},N_t)$ counted cycles for a window length $N_{\mathrm{win}}$, and define the rolling pass-rate $\hat p_t=\nicefrac{1}{|\mathcal{U}^{\mathrm{win}}_t|}\sum_{u\in\mathcal{U}^{\mathrm{win}}_t}X_u$, evaluated only after at least one counted cycle. Let $C_t$ denote the consumed budget, wall-clock minutes within the session, or, for the lifetime variant, spend-incurring cycles accumulated in the ledger across sessions, with cap $C_{\max}$. The loop halts at the stopping time
\begin{equation}
 T_{\mathrm{stop}} = \min\bigl\{T_{\mathrm{done}}, T_{\mathrm{plateau}}, T_{\mathrm{budget}}, T_{\mathrm{human}}\bigr\},
 \label{eq:growth_stop}
\end{equation}
whose components are, in priority order,
\begin{align}
 T_{\mathrm{done}}
 &= \inf\bigl\{t: N_t \ge n_{\min} \wedge \hat p_t \ge p^{\mathrm{tgt}}\bigr\},
 \label{eq:tau_done}\\
 T_{\mathrm{budget}}
 &= \inf\bigl\{t: C_t \ge C_{\max}\bigr\},
 \label{eq:tau_budget}\\
 T_{\mathrm{plateau}}
 &= \inf\Bigl\{t:
 \max_{u\in\mathcal{U}^{\mathrm{rec}}_t}\hat p_u
 \le
 \max_{u\in\mathcal{U}^{\mathrm{pre}}_t}\hat p_u + \delta_{\mathrm{pl}}
 \wedge
 \text{no promotion in } \mathcal{U}^{\mathrm{rec}}_t
 \Bigr\}.
 \label{eq:tau_plateau}
\end{align}
where $\mathcal{U}^{\mathrm{rec}}_t=\{u: t-N_{\mathrm{pl}}<u\le t, N_u\ge n_{\min}\}$ is the trailing span of $N_{\mathrm{pl}}$ counted cycles and $\mathcal{U}^{\mathrm{pre}}_t=\{u: u\le t-N_{\mathrm{pl}}, N_u\ge n_{\min}\}$ its non-empty history. Both sets are restricted to cycles past the evidence floor $n_{\min}$; otherwise one or two early outcomes can distort the historical best and trigger the plateau rule against genuine later improvement. $T_{\mathrm{human}}$ is a human stop signal checked at cycle boundaries. With all stopping rules disabled, the loop reduces to a budget cap plus human stop signal.

\paragraph{Calibrated escalation and question selection.}
The calibrated questioning policy recalibrates raw confidence $c$ into $\hat P(\text{correct}\mid c)$ with a monotone isotonic fit via pool-adjacent-violators on labeled calibration data, then selects the smallest confidence threshold whose accepted region has an estimated error bound below the target,
\begin{equation}
 \tau_{\mathrm{ask}} = \min\Bigl\{c: \mathrm{kl\text{-}UCB}\bigl(\hat e(c),n(c);\delta\bigr) \le \alpha \Bigr\},
  \pi_{\text{ask}}=\text{resolve alone} \iff c(x) \ge \tau_{\mathrm{ask}},
 \label{eq:clara}
\end{equation}
where $c(x)$ is the query's recalibrated confidence, $\hat e(c)$ is the empirical error among the $n(c)$ calibration points with confidence $\ge c$, and $\mathrm{kl\text{-}UCB}$ is the tight Chernoff upper bound for a fixed threshold. Because the isotonic fit and threshold are selected from data, this is a threshold-selection heuristic unless calibration fit, threshold search, and final certification are split, or a simultaneous confidence/conformal-risk method is used under exchangeability and no post-selection model update. When a question is warranted, the policy chooses it by model-estimated expected information gain over the design posterior,
\begin{equation}
 \mathrm{EIG}(\beta) = I(\mathrm{Design};\mathrm{Answer}) = H(D) - \mathbb{E}_{a}\bigl[H(D\mid a)\bigr],
 \label{eq:clara_eig}
\end{equation}
computed exactly only when the candidate-design prior $P(D)$ and answer likelihoods $P(a\mid D)$ are finite and known. With LLM-estimated likelihoods, Eq.~\eqref{eq:clara_eig} is a model-estimated EIG score rather than an exact information-theoretic quantity.

\paragraph{The LORD ledger.}
At the $k$-th adoption decision the spendable level is
\begin{align}
a_k &=
w_0\gamma_k
+(\alpha-w_0)\gamma_{k-\tau_1}\indic{\tau_1<k}
+\alpha\sum_{\substack{j\ge2:\\ \tau_j<k}}\gamma_{k-\tau_j},\\
\textstyle\sum_{j\ge1}\gamma_j &= 1,
 \label{eq:lord}
\end{align}
capped at $\alpha$, where $(\gamma_j)_{j\ge1}$ is a summable, decaying spending sequence, $\tau_1<\tau_2<\cdots$ are the past adoption times, $w_0\in(0,\alpha)$ is the initial wealth, and unavailable or non-positive-index terms are omitted. The ledger spends this level on the intersection-union p-value
\begin{equation}
 p_k =
 \begin{cases}
 1, & B_{\mathrm{safety}}(\mathsf{Pack}')=0,\\
 \max\{p_k^\text{gain},p_k^\text{reg},p_k^\text{safety}\}, & \text{safety stream enabled},\\
 \max\{p_k^\text{gain},p_k^\text{reg}\}, & \text{otherwise},
 \end{cases}
 \label{eq:evalue_composite}
\end{equation}
where $B_{\mathrm{safety}}$ is the hard safety predicate declared by root policy, $p_k^\text{gain}=\min\{1,1/K_t^\text{gain}\}$ and $p_k^\text{reg}=\min\{1,1/K_t^\text{reg}\}$ are the e-process p-values from Eqs.~\eqref{eq:sage_gain}--\eqref{eq:sage_reg}, and $p_k^\text{safety}$ is the corresponding p-value or hard veto from the safety stream. The optional anytime-valid gate accepts a candidate iff $p_k\le a_k$; promotion still requires Eq.~\eqref{eq:promotion_contract}. Each rejected null recycles wealth in the spirit of LORD++~\citep{javanmard2018online,ramdas2018saffron}, so the level does not collapse to zero over a long evolution loop. Formal online false-discovery control requires conditionally super-uniform candidate-level p-values, predictable spending, the standard LORD++ dependence assumptions, fresh valid candidate-level evaluation data, and no invalid adaptive holdout reuse; otherwise the same ledger is only conservative audit accounting.

\section{Operational Contract Details}
\label{sec:impl_appendix}

\paragraph{Resident scheduling mechanics.}
A schedule is a tuple $\varsigma_{\mathrm{sch}}=(\text{kind},\Delta_{\mathrm{sch}},t_{\mathrm{daily}})$ with an interval recurrence $\Delta_{\mathrm{sch}}$ and an optional daily wall-clock anchor $t_{\mathrm{daily}}$. When both interval and daily anchors are set, the daily anchor takes precedence; the next firing time is
\begin{equation}
 t_{\mathrm{next}} =
 \begin{cases}
 t + \Delta_{\mathrm{sch}}, & \text{interval schedule } (\Delta_{\mathrm{sch}}>0) \\
 \min\{t' > t: t' \equiv t_{\mathrm{daily}} \pmod{24\mathrm{h}}\}, & \text{daily schedule.}
 \end{cases}
 \label{eq:agora_next}
\end{equation}
A newly created schedule is first armed by computing its initial slot rather than fired retroactively; a failed dispatch is recorded and re-armed rather than retried immediately; and a \textsc{reconcile} tick advances the World's ticket/campaign state machine.

\paragraph{Working memory.}
The per-task working memory of Section~\ref{sec:memory} is a sequence of tagged text blocks $(b_1,\ldots,b_m)$ under a character budget $B_{\mathrm{wm}}$, where the subscript avoids conflict with the blueprint symbol $B$: when $\sum_i|b_i|>B_{\mathrm{wm}}$, the lowest-salience block, by a salience score $\sigma(b_i)$ that weights block type, recency, and reference count, is evicted rather than the tail truncated, and the surviving blocks are injected as a structured \texttt{<working-memory>} block that the agent reads each turn.

\section{Advanced Self-Evolution Instruments}
\label{sec:advanced_appendix}
These four opt-in instruments extend the acceptance machinery of Section~\ref{sec:echo} across extended operation. They govern how a proven capability is retained, how memory is edited, how search explores, and whether a tool may rewrite its own code. Each instrument reuses the gate rather than bypassing it; Table~\ref{tab:advanced} reports controlled harness tests.

\paragraph{Anytime-valid retirement with bounded capacity.} Accepted artifacts, such as skills, guidelines, and workflow templates, accumulate without limit, and a short run of adverse outcomes can silently down-rank a genuinely useful artifact. The retirement guard gives each artifact a time-uniform confidence sequence for its success rate. For $s$ successes over $n$ bounded outcomes $X_i\in[0,1]$, a union-bound-over-time Hoeffding sequence sets a per-step level $\alpha_n=\alpha\cdot\nicefrac{6}{(\pi^2 n^2)}$, so that $\sum_n\alpha_n=\alpha$ by the Basel-problem identity $\sum_n\nicefrac{1}{n^2}=\nicefrac{\pi^2}{6}$, and radius
\begin{equation}
 r_n=\sqrt{\nicefrac{1}{2n}\log\nicefrac{2}{\alpha_n}},
 [\mathrm{LCB}_n,\mathrm{UCB}_n]=\bigl[\bar X_n-r_n, \bar X_n+r_n\bigr]\cap[0,1],
 \label{eq:ratchet_cs}
\end{equation}
where $\bar X_n=n^{-1}\sum_{i\le n}X_i$. By the union bound, the sequence satisfies $\Pr(\exists n: \mu\notin[\mathrm{LCB}_n,\mathrm{UCB}_n])\le\alpha$, valid at every stopping time, unlike a fixed-$n$ interval. An artifact becomes proven when its cumulative $\mathrm{LCB}_n$ first clears a threshold; its proven floor is the running maximum $\phi_t=\max_{s\le t}\mathrm{LCB}_s$, which only ratchets up with evidence. A proven artifact is retired only when a recent-window $\mathrm{UCB}$ falls below $\phi_t-\eta_{\mathrm{drift}}$ for a margin $\eta_{\mathrm{drift}}$.

The cumulative CS supports the proven floor, but a sliding-window retirement rule needs separate window-valid confidence accounting to carry the same formal false-retirement guarantee. In this paper the retirement column is therefore an empirical stress result. Real drift still requires fresh evidence and can be missed under weak or delayed signals. A bounded capacity $C$ evicts the weakest, unproven, low-$\mathrm{LCB}$ artifacts first. Empirically, over $100$ seeded streams of $400$ outcomes from a noisy skill with true success rate $0.6$, as Table~\ref{tab:advanced} reports, the sequence never lets the true mean escape its band and falsely retires the skill in $0\%$ of streams. The fixed-window majority gate falsely retires it in $96\%$ of streams because its per-window false-alarm probability accumulates over time.

\paragraph{Learnable memory-editing policy.} The episodic store, Eq.~\eqref{eq:prune}, admits every experience and prunes by a fixed value blend that is blind to \emph{coverage}: it fills the budget with high-quality near-duplicates and loses the diversity needed for varied queries, a failure mode we call \emph{redundancy collapse}. The memory-edit policy recasts memory management as a memory-as-action decision: for each write it selects an action $a$ from $\{\textsc{noop},\textsc{add},\textsc{merge},\textsc{replace},\textsc{drop}\}$ given context features $\mathbf{x}$ such as max-similarity, novelty, candidate and neighbor quality, fullness, and redundancy under a linear softmax $\pi_\theta(a\mid\mathbf{x})\propto\exp(\theta_a^\top\mathbf{x})$, trained online by REINFORCE with the same EWMA baseline $b$ as the variance-reduced workflow update:
\begin{equation}
 \theta_a \leftarrow \theta_a + \kappa(r-b)\bigl(\indic{a=a_t}-\pi_\theta(a\mid\mathbf{x})\bigr)\mathbf{x},
  b\leftarrow b+\beta_0(r-b),
 \label{eq:remedy}
\end{equation}
where $\kappa>0$ is the memory-policy learning rate and $r$ is the resulting retrieval utility or a delayed task outcome. Under redundancy collapse the policy learns the interpretable rule ``\textsc{add} novel, \textsc{merge} redundant'' and lifts held-out retrieval utility $32.7\%$ over the fixed value-pruner and $34.8\%$ over a raw append store in $30$ held-out worlds with capacity $8$, as Table~\ref{tab:advanced} reports; with a strong pruner already present it collapses to \textsc{add} and shows no observed regression in this controlled study.

\paragraph{Quality-diversity illumination.} Population search, such as \textsc{AFlow}, and the supernet, Eq.~\eqref{eq:maas}, maintain a single line of descent that a deceptive landscape can trap. The diversity-preserving search adds a MAP-Elites archive over a behavior descriptor $\mathbf{b}(g)$, such as a workflow's phase-count and tool-diversity, discretized into cells; each cell keeps its fittest genome, $\mathcal{A}[\mathrm{cell}(\mathbf{b})]=\arg\max f$. Variation samples a parent elite, uniformly, by fitness, or by descriptor-space novelty $\rho(\mathbf{b})=\nicefrac{1}{k}\sum_{j\in\mathrm{kNN}}\lVert\mathbf{b}-\mathbf{b}_j\rVert$, so diverse stepping-stones are retained and reported via \emph{coverage}, defined as filled cells, and \emph{QD-score} $\sum_\text{cells}f$. A surfaced genome remains a candidate gated by the paired or optional anytime-valid gate, so illumination widens exploration without weakening safety. On a deceptive two-basin task the method reaches the narrow global optimum in $21/30$ runs versus $8/30$ for an elitist hill-climber with equal evaluation budget, at $76\%$ mean behavior-space coverage, as Table~\ref{tab:advanced} records.

\paragraph{Gated self-referential tool evolution.} Skill and prompt evolution stop at natural-language artifacts; the self-referential repair mode allows a candidate edit to the source of one typed tool, following the self-improvement setting studied by the Darwin--G\"odel Machine and SICA~\citep{zhang2025dgm,robeyns2025sica}. The release rule is deliberately narrow: a candidate cannot replace the live tool unless it passes static policy checks, isolated execution checks, and a paired no-regression gate against a tool test suite $\mathcal{T}_{\mathrm{tool}}$ assembled from prior uses of that tool. Let $c_0$ be the incumbent tool, let $c$ be the candidate, and let $\mathcal{T}_{\mathrm{pass}}(c)=\{\tau\in\mathcal{T}_{\mathrm{tool}}:\mathrm{pass}(c,\tau)\}$:
\begin{equation}
 \text{adopt}(c)\iff
 \mathcal{T}_{\mathrm{pass}}(c)\supseteq \mathcal{T}_{\mathrm{pass}}(c_0)
 \wedge
 |\mathcal{T}_{\mathrm{pass}}(c)|>|\mathcal{T}_{\mathrm{pass}}(c_0)|,
 \label{eq:forge}
\end{equation}
that is, it must keep every test the incumbent passed and pass strictly more. Higher-risk write-back remains subject to human approval. Across a battery of buggy tools the repair mode fixes $3/3$ through the full gate while blocking all four tested adversarial proposal patterns at the guard or sandbox layers, as Table~\ref{tab:advanced} reports. The evidence is scoped to the guarded rewrite mechanism and the tested adversarial patterns.

\begin{table}[tb]
\centering
\footnotesize
\setlength{\tabcolsep}{3.8pt}\renewcommand{\arraystretch}{1.2}
\caption{Advanced self-evolution instruments under controlled harness tests.}
\label{tab:advanced}
\begin{tabular}{@{}>{\raggedright\arraybackslash}p{0.12\linewidth}
                    >{\raggedright\arraybackslash}p{0.46\linewidth}
                    >{\raggedright\arraybackslash}p{0.19\linewidth}
                    >{\raggedright\arraybackslash}p{0.15\linewidth}@{}}
\toprule
Instrument & Metric & Baseline & \logos \\
\midrule
Retirement guard & false retirement of a useful noisy skill ($100$ streams) & $0.96$ (fixed window) & $\mathbf{0.00}$ \\
 & CS coverage (escape rate, $\alpha{=}0.1$) & \textemdash & $0.00\le\alpha$ \\
Memory-edit policy & held-out retrieval utility @cap$=8$ ($30$ worlds) & $0.353$ (value-prune) & $\mathbf{0.468}$ \\
 & vs. raw FIFO append store & $0.347$ & $\mathbf{+34.8\%}$ \\
Diversity search & deceptive global-optimum reach (30 seeds) & $8/30$ (elitist) & $\mathbf{21/30}$ \\
 & behavior-space coverage & \textemdash & $0.76$ \\
Tool repair & buggy tools repaired through the gate & \textemdash & $\mathbf{3/3}$ \\
 & adversarial rewrites blocked & \textemdash & $\mathbf{4/4}$ \\
\bottomrule
\end{tabular}
\end{table}

\section{Experimental Protocol and Extended Results}
\label{sec:exp_appendix}

\subsection{Protocol}
\label{sec:exp_protocol}

For every quantitative cell, the unit of analysis is a task, candidate decision, or generated policy case as stated in the table or surrounding text. Hosted-model evaluations use temperature zero unless a method explicitly requires multiple samples. Task identifiers, prompts, scorer versions, model identifiers, decoding settings, and candidate hashes are recorded with each run. Tables report concrete model identifiers where model choice affects the result. The evaluation records additionally retain provider, snapshot or run date, decoding settings, scorer harness, prompt bundle, and price index needed for an audit. Runtime connection and deployment-environment details are not part of the paper. Self-evolution uses disjoint proposal, optional selection, final-gate, and report sets; the proposal process cannot inspect final-gate examples. Paired comparisons reuse task identifiers and, for gate-only replay, the same candidate records.

\paragraph{Uncertainty convention.}
Binomial rates use Wilson $95\%$ intervals when displayed; paired task outcomes use a paired test or seed-level range as stated in the relevant table or surrounding text. Generated stress suites report the number of seeds and cases and should be interpreted as controlled mechanism evidence, not as samples from an unknown production population. For generated rows, the evidence unit is the disclosed template family, perturbation class, seed rule, and independent evaluator or oracle used to label it; raw case counts mainly show coverage of that generated grid. Case counts created from the same contract templates are not independent safety situations, so zero observed violations mean no violations of the encoded rule under that grid.

\paragraph{Gate convention.}
The default paired gate requires positive mean held-out gain, a declared minimum margin, and the stated regression budget. Table~\ref{tab:gate_defaults} gives the defaults used unless a table states otherwise. The optional anytime-valid gate uses evidence for gain and regression under the assumptions in Appendix~\ref{sec:method_ext}. Safety vetoes and root-policy conflicts override utility evidence. Hyper-parameters that materially affect a table are stated in that table or its surrounding text.

\begin{table}[tb]
\centering\small
\setlength{\tabcolsep}{4.0pt}\renewcommand{\arraystretch}{1.1}
\caption{Default promotion-gate settings. These are empirical release-rule
defaults, not statistical error guarantees.}
\label{tab:gate_defaults}
\begin{tabular}{@{}L{0.24\linewidth}L{0.24\linewidth}L{0.42\linewidth}@{}}
\toprule
Setting & Default & Interpretation \\
\midrule
$m$ & score $s$ for benchmark tables; utility $u$ only when declared & tables isolate accuracy unless cost, latency, or risk are explicitly priced \\
$\delta_{\min}$ & binary score gates: one task-equivalent improvement, $\nicefrac{1}{|\mathcal{H}_{\mathrm{gate}}|}$; continuous gates: a pre-declared minimum practical effect & prevents promotion on a pure tie or a metric-dependent numerical fluctuation \\
$\epsilon_{\text{reg}}$ & $0$ for score gates & any task-level score drop counts as a sample regression unless a tolerance is declared \\
$R_{\max}$ & $0$ for high-risk and safety-surface changes; $1$ for low-risk offline skill/memory studies with $|\mathcal{H}_{\mathrm{gate}}|\ge20$ & zero observed regressions by default for authority-sensitive changes; small empirical budget for low-risk mechanism checks \\
Safety veto & enabled for write-capable tools, approval policy, verifier changes, and safety prompts & blocks promotion even when task score improves \\
\bottomrule
\end{tabular}
\end{table}

\paragraph{Scope of measurement.}
External benchmark results are included where the task split and scorer are specified. Heavier interactive environments are treated as separate validation targets rather than represented by proxy substitutes. Controlled suites with known ground truth isolate compiler, verifier, routing, safe-delegation, and gate behavior; benchmark, user-study, and security evidence are reported as separate evidence classes when used.

\subsection{External Benchmark Protocol Details (Q1--Q4)}
\label{sec:exp_flagship}
Table~\ref{tab:flagship} reports direct single-model reference results on the stated external task families, including tool-use, mathematical reasoning, code generation, question answering, long-memory, and safety benchmarks~\citep{patil2025bfcl,yao2024taubench,barres2025tau2,shi2026tauknowledge,cobbe2021gsm8k,hendrycks2021math,chen2021codex,austin2021mbpp,yang2018hotpotqa,dua2019drop,wang2024mmlupro,wu2025longmemeval,maharana2024locomo,zhang2024agentsafetybench}. These cells validate the scorer and provide matched baselines for the memory and cascade studies; they are not, by themselves, evidence that compilation or self-evolution improves the task. The table and surrounding notes state the model roster, task count, and any restricted protocol scope.

\begin{table}[tb]
\centering\small
\setlength{\tabcolsep}{3.5pt}\renewcommand{\arraystretch}{1.1}
\caption{Official-scale direct-mode sweep across three model identifiers.}
\label{tab:flagship}
\begin{tabular}{lcccc}
\toprule
Suite & $n$ & \texttt{gpt-4o-mini} & \texttt{gpt-4.1} & \texttt{gpt-5.1} \\
\midrule
BFCL (accuracy) & $1281$ & $0.869$ & $0.872$ & $0.852$ \\
$\tau^2$-bench retail (pass$^1$) & $114\times4$ & $0.428$ & $0.787$ & $0.638$ \\
$\tau^2$-bench airline (pass$^1$)& $50\times4$ & $0.240$ & $0.560$ & $0.520$ \\
$\tau^3$-bench banking (pass$^1$) & $97\times3$ & $0.055$ & $0.065$ & $0.045$ \\
GSM8K (pass@1) & $1319$ & $0.937$ & $0.952$ & $0.966$ \\
MATH-500 (pass@1) & $500$ & $0.748$ & $0.836$ & $0.866$ \\
HumanEval (pass@1) & $164$ & $0.848$ & $0.939$ & $0.957$ \\
MBPP (pass@1) & $257$ & $0.829$ & $0.899$ & $0.903$ \\
HotpotQA (EM / F$_1$) & $7405$ & $0.607/0.745$ & $0.646/0.797$ & $0.640/0.785$ \\
DROP (EM / F$_1$) & $9535$ & $0.614/0.635$ & $0.674/0.738$ & $0.624/0.676$ \\
MMLU-Pro (accuracy)$\dagger$ & $12032$ & $0.630$ & $0.655$ & $0.805$ \\
LongMemEval, oracle (pass rate) & $500$ & $0.744$ & $0.906$ & $0.890$ \\
LoCoMo (mean F$_1$) & $1540$ & $0.516$ & $0.591$ & $0.596$ \\
Agent-SafetyBench (safety rate)$^{\mathrm{j}}$ & $2000$ & $0.670$ & $0.761$ & $0.864$ \\
\bottomrule
\end{tabular}
\end{table}

\begin{figure}[tb]
 \centering
 \includegraphics[width=0.9\linewidth]{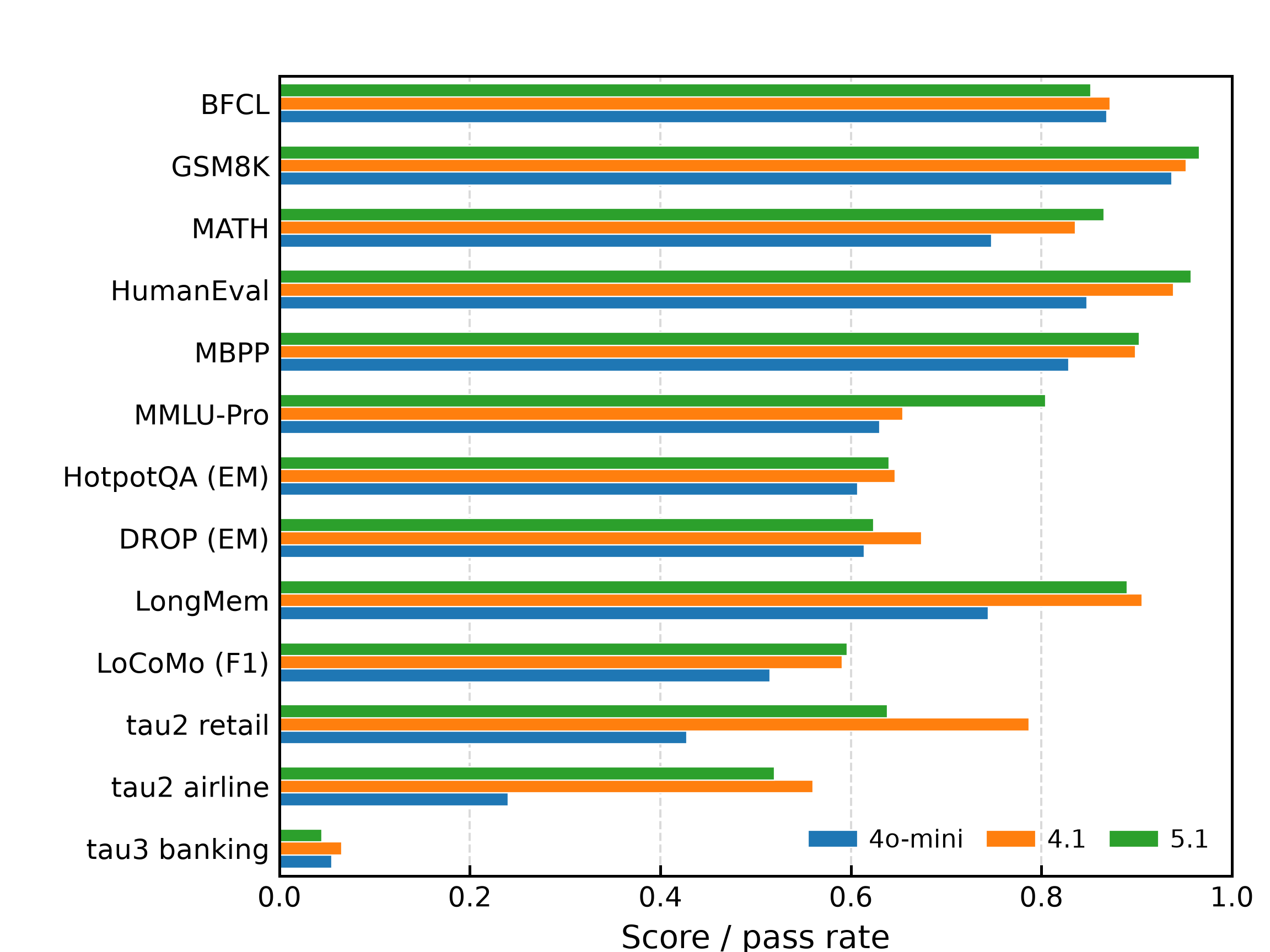}
 \caption{Grouped view of the official-scale direct-mode sweep.}
 \label{fig:flagship_sweep}
\end{figure}

\begin{figure}[tb]
 \centering
 \includegraphics[width=0.9\linewidth]{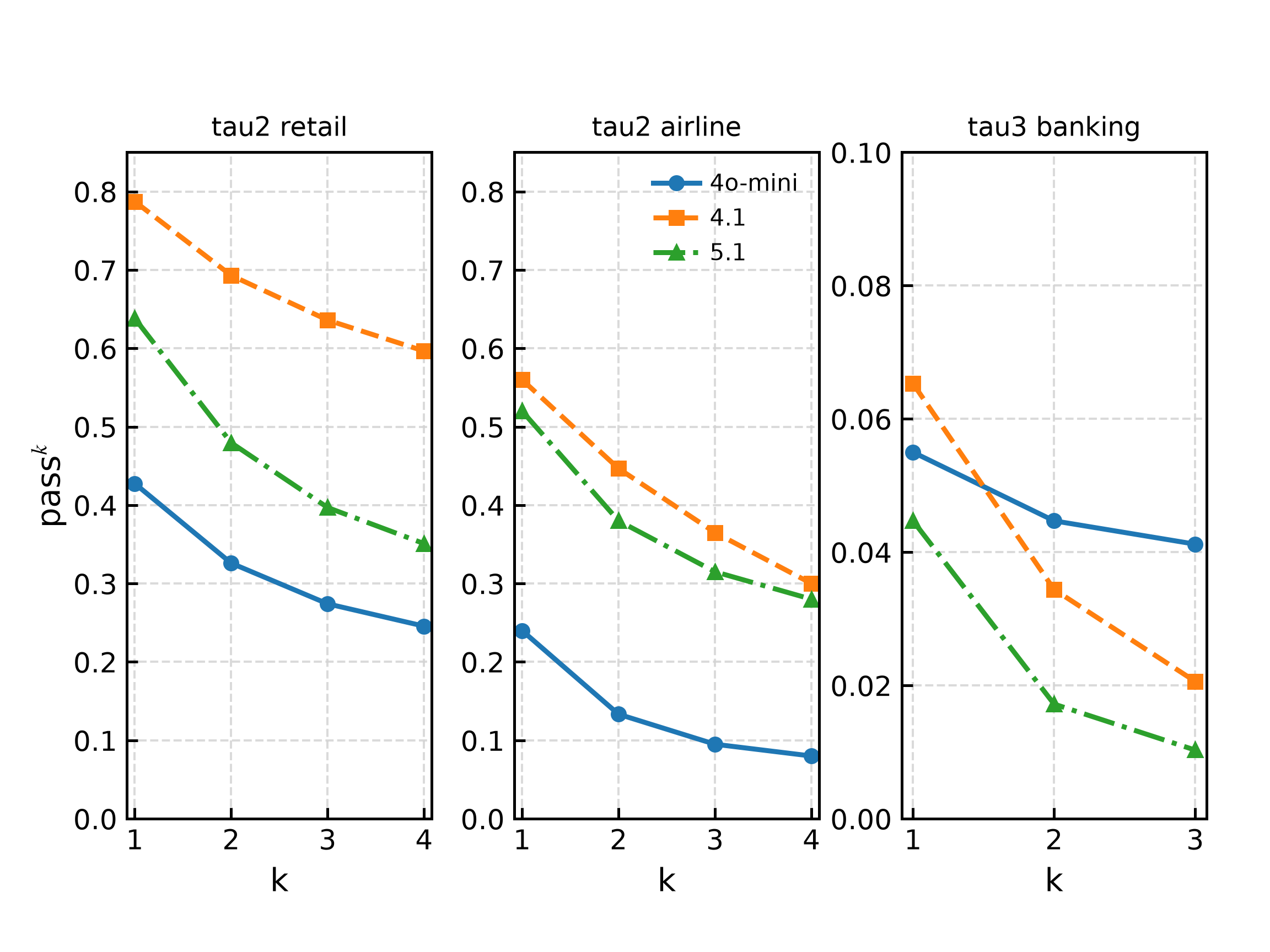}
 \caption{Reliability decay on repeated $\tau^2$ and $\tau^3$ trials.}
 \label{fig:tau_passk}
\end{figure}

Figures~\ref{fig:flagship_sweep} and~\ref{fig:tau_passk} visualize the same
direct-mode sweep: the former groups task-family scores across model
identifiers, while the latter shows pass-rate decay under repeated trials.

Three patterns are noteworthy. First, \texttt{gpt-4.1} outperforms \texttt{gpt-5.1} in several cases, including BFCL, $\tau^2$-bench, and LongMemEval, a reminder that model order is not monotone across task families. This motivates the model-agnostic, per-query routing of Section~\ref{sec:routing}. Second, the HotpotQA, DROP, and MMLU-Pro rows use the official distractor context, official DROP span-bag scoring, and the official $5$-shot MMLU-Pro prompt format. The HotpotQA and DROP runs are complete at their official sizes of $7405$ and $9535$, with official EM and token-F$_1$ in Table~\ref{tab:flagship}, and all three MMLU-Pro models use the full official $12032$-task protocol. Third, the $\tau^3$-bench \texttt{banking\_knowledge} row is a knowledge-grounded action setting: support tasks must be grounded in a policy corpus before action. The resulting pass$^1$ values remain low and nearly flat across models, consistent with the published difficulty of this domain; the row measures direct single-model inference rather than compilation, memory, or self-evolution.

\subsection{Code and SWE (Q2, extended)}
\label{sec:exp_code}
Table~\ref{tab:code} reports the code-generation rows on \texttt{gpt-5.1}. The code rows are system-level success under the compiled harness; because the compiled arm may use additional calls and verifier-visible tests, they should be compared to leaderboard single-call pass@1 only with hidden-test separation. Dedicated coding-agent benchmarks such as SWE-bench remain complementary evaluation targets~\citep{jimenez2024swebench,yang2024sweagent}.

\begin{table}[tb]
\centering\small
\setlength{\tabcolsep}{5pt}\renewcommand{\arraystretch}{1.1}
\caption{Code-generation rows under the compiled harness.}
\label{tab:code}
\begin{tabular}{@{}lccc@{}}
\toprule
Suite & direct & compiled & $\Delta$ \\
\midrule
HumanEval & $92.7$ & $96.3$ & $+3.6$ \\
MBPP & $90.0$ & $92.5$ & $+2.5$ \\
\bottomrule
\end{tabular}
\end{table}

\subsection{Self-Evolution: Aggregate Gate Ablation and Non-Saturated Cells (Q3)}
\label{sec:exp_evolve_full}
Table~\ref{tab:evolve} gives the full numeric aggregate behind Figure~\ref{fig:gate_ablation}.

\begin{table}[tb]
\centering\small
\setlength{\tabcolsep}{3.4pt}\renewcommand{\arraystretch}{1.12}
\caption{Aggregate self-evolution gate ablation on non-saturated suites.}
\label{tab:evolve}
\begin{threeparttable}
\begin{tabular}{lccccccc}
\toprule
Gate & $K$ & before & after & $\Delta$ & Regress ($\downarrow$) & Misevol. ($\downarrow$) & Rej. \\
\midrule
No gate & $1$ & $86.8$ & $85.2$ & $-1.6$ & $1.11$ & $0.56$ & $0.22$ \\
Proxy gate & $1$ & $87.3$ & $86.3$ & $-0.9$ & $1.22$ & $0.56$ & $0.11$ \\
Paired gate & $1$ & $86.8$ & $86.6$ & $-0.2$ & $0.22$ & $0.11$ & $0.89$ \\
Paired gate & $3$ & $87.5$ & $86.8$ & $-0.7$ & $0.67$ & $0.33$ & $0.56$ \\
\midrule
Anytime-valid gate & $1$ & $86.1$ & $86.1$ & $\mathbf{0.0}$ & $\mathbf{0.00}$ & $\mathbf{0.00}$ & $1.00$ \\
\bottomrule
\end{tabular}
\begin{tablenotes}\footnotesize
\item The ``before'' column differs slightly across rules because every cell re-evaluates its own baseline in a separate run, so rows are internally consistent but not baseline-matched to each other. The anytime-valid gate's zero misevolution here comes with full rejection: it refused to support gains the streaming evidence could not distinguish from noise. This is conservative error-accounted behavior under Appendix~\ref{sec:sage_lord_formal}, not evidence of high adoption power. The four advanced instruments are isolated separately in Table~\ref{tab:advanced}.
\end{tablenotes}
\end{threeparttable}
\end{table}

All before/after report regressions are counted fail-closed: if a task passed in the before run but is missing from the after run, the report treats it as a regression rather than an omitted datum. Gate means follow the same rule: a candidate missing a baseline gate task receives score $0$ on that task, and the task remains in the denominator. This prevents a harness timeout or dropped task from making an evolved candidate appear safer or more accurate than it is.

\begin{table}[tb]
\centering\small
\setlength{\tabcolsep}{4pt}\renewcommand{\arraystretch}{1.1}
\caption{API targeted strengthening rerun on \texttt{gpt-4o-mini}.}
\label{tab:api_strengthening}
\begin{tabular}{llcccc}
\toprule
Suite & Gate & Cells & $\bar\Delta$ & Regr. & Adopt \\
\midrule
HotpotQA & off & $5$ & $-0.6$ & $11$ & $5/5$ \\
HotpotQA & proxy & $5$ & $-1.2$ & $12$ & $5/5$ \\
HotpotQA & execution & $5$ & $-0.2$ & $3$ & $1/5$ \\
HotpotQA & anytime-valid & $5$ & $0.0$ & $0$ & $0/5$ \\
AIME & off & $1$ & $+16.7$ & $0$ & $1/1$ \\
AIME & proxy & $1$ & $0.0$ & $1$ & $1/1$ \\
AIME & paired, $K{=}3$ & $1$ & $0.0$ & $0$ & $0/1$ \\
AIME & anytime-valid & $1$ & $0.0$ & $0$ & $0/1$ \\
\bottomrule
\end{tabular}
\end{table}

The aggregate of Table~\ref{tab:evolve} averages over multi-seed cells on \texttt{gpt-4.1}; Table~\ref{tab:evolve_cells} reports a complementary \texttt{gpt-4o-mini} per-decision view with an $n{=}20$ report holdout, proposal/gate budget $28$, and seed $0$. At this scale, the selected \texttt{gpt-4o-mini} candidates do not regress on the report holdout: every $K{=}1$ rule adopts a candidate that either helps MATH-500 by $+5.0$ points or is neutral on HotpotQA. The acceptance rules are therefore indistinguishable when the candidate stream is harmless. The one visible discipline event is beam-by-gate with $K{=}3$ on MATH-500, where the gate rejects the selected candidate rather than adopt an unproven change. The primary gate evidence remains the non-saturated \texttt{gpt-4.1} matrix of Table~\ref{tab:evolve_direct}, where off/proxy rules adopt harmful changes on most seeds.

\begin{table}[tb]
\centering\small
\setlength{\tabcolsep}{4pt}\renewcommand{\arraystretch}{1.1}
\caption{\texttt{gpt-4o-mini} evolution cells ($n{=}20$ report holdout, proposal/gate budget $28$, seed $0$). ``adopt'' = the rule accepted the candidate. Pass@1 [\%].}
\label{tab:evolve_cells}
\begin{tabular}{@{}llcccc@{}}
\toprule
Setting & Suite & before & after & Regress & Adopt? \\
\midrule
No gate ($K{=}1$) & MATH-500 & $80.0$ & $85.0$ & $0$ & adopt \\
Proxy ($K{=}1$) & MATH-500 & $80.0$ & $85.0$ & $0$ & adopt \\
Paired ($K{=}1$) & MATH-500 & $80.0$ & $85.0$ & $0$ & adopt \\
Paired ($K{=}3$) & MATH-500 & $70.0$ & $70.0$ & $0$ & reject \\
No gate ($K{=}1$) & HotpotQA & $65.0$ & $65.0$ & $0$ & adopt \\
Proxy ($K{=}1$) & HotpotQA & $65.0$ & $65.0$ & $0$ & adopt \\
Paired ($K{=}1$) & HotpotQA & $65.0$ & $65.0$ & $0$ & adopt \\
Paired ($K{=}3$) & HotpotQA & $65.0$ & $65.0$ & $0$ & adopt \\
\bottomrule
\end{tabular}
\end{table}

\subsection{Target-Side Re-Verification Boundary (Q6)}
\label{sec:exp_transfer_full}

Cross-deployment reuse is evaluated as a release-control property rather than an unconditional accuracy claim. A recalled skill, design precedent, or workplace lesson is imported as an isolated candidate, assigned target-side provenance, and evaluated under the target pack's verifier, permissions, and holdout. The conformance suite verifies that no imported candidate is promoted directly and that unsafe imports are rejected or escalated. The result is authority-boundary evidence: transferred artifacts must pass target-side checks before they can affect the target deployment.

\subsection{Routing and Cascade Protocol (Q7)}
\label{sec:exp_cascade_full}
Table~\ref{tab:cascade} gives the numeric values behind Figure~\ref{fig:cascade}.

\begin{table}[tb]
\centering\small
\setlength{\tabcolsep}{4.0pt}\renewcommand{\arraystretch}{1.12}
\caption{Verification-gated cascade accuracy frontier.}
\label{tab:cascade}
\begin{tabular}{lccccc}
\toprule
Suite & Floor & Ref. & Cascade & vs Floor & vs Ref. \\
\midrule
GSM8K ($n{=}200$) & $0.950$ & $0.985$ & $0.940$ & $-0.010$ & $-0.045$ \\
MATH-500 ($n{=}200$) & $0.785$ & $0.850$ & $0.810$ & $+0.025$ & $-0.040$ \\
HumanEval ($n{=}164$) & $0.841$ & $0.951$ & $\mathbf{0.976}$ & $\mathbf{+0.134}$ & $+0.024$ n.s. \\
MBPP ($n{=}257$) & $0.825$ & $0.895$ & $\mathbf{0.914}$ & $\mathbf{+0.089}$ & $+0.019$ n.s. \\
\bottomrule
\end{tabular}
\end{table}

\paragraph{Measured protocol behind Table~\ref{tab:cascade}.}
The measured Q7 study uses matched task IDs under three conditions: Floor, Reference, and Cascade. Paired tests are Holm-corrected within each suite for floor-vs-cascade, reference-vs-cascade, and floor-vs-reference. On HumanEval, the execution verifier cuts cheap-floor failures from $26$ to $4$, with $22$ improvements, $0$ regressions, and Holm-adjusted $p=1.43{\times}10^{-6}$. It finishes $86.0\%$ of tasks at the cheap stage. MBPP gives an independent code-suite replication: failures fall from $45$ to $22$, with $23$ improvements, $0$ regressions, and Holm-adjusted $p=7.15{\times}10^{-7}$. Against the direct \texttt{gpt-5.1} reference, the point estimates are positive but not significant. The paper, therefore, claims significant failure reduction versus the cheap floor, while GSM8K and MATH-500 show the boundary where non-code verifiers are not strong enough to dominate direct \texttt{gpt-5.1}.

\subsection{Collective Router (Q8): Numeric Table}
\label{sec:exp_collective_full}
Table~\ref{tab:collective} provides the numeric values behind the frontier of Figure~\ref{fig:chorus_frontier}, along with the replica-diversity study of Section~\ref{sec:exp_collective}.

\begin{table}[tb]
\centering\small
\setlength{\tabcolsep}{4.0pt}\renewcommand{\arraystretch}{1.12}
\caption{Collective-router numeric results.}
\label{tab:collective}
\begin{tabular}{@{}llcc@{}}
\toprule
Study & Policy & Accuracy ($\uparrow$) & Cost ($\downarrow$) \\
\midrule
Frontier & single \texttt{gpt-5.1} & $0.975$ [$0.938,1.000$] & $2.500$ \\
 & single \texttt{gpt-4o-mini} & $0.875$ [$0.812,0.938$] & $0.150$ \\
 & oracle (per-query best) & $0.975$ [$0.938,1.000$] & \textemdash \\
 & default preset (table head, cost-blind) & $0.975$ [$0.938,1.000$] & $2.500$ \\
 & \logos preset ($\lambda_{\mathrm{cost}}{=}0$) & $\mathbf{0.975}$ [$0.938,1.000$] & $\mathbf{1.795}$ \\
 & \logos preset ($\lambda_{\mathrm{cost}}{=}0.1$) & $0.913$ [$0.875,0.938$] & $0.385$ \\
 & \logos preset ($\lambda_{\mathrm{cost}}{=}0.3$) & $0.913$ [$0.812,1.000$] & $0.444$ \\
 & \logos preset ($\lambda_{\mathrm{cost}}{=}1.0$) & $0.875$ [$0.812,0.938$] & $0.150$ \\
\midrule
Replicas & single call & $0.800$ & $1\times$ \\
 & $3\times$ replicas + consensus, Eq.~\eqref{eq:chorus_medoid}, & $\mathbf{0.900}$ & $3\times$ \\
 & $3\times$ replicas + deep aggregator & $\mathbf{0.900}$ & $4\times$ \\
\bottomrule
\end{tabular}
\end{table}

\subsection{Operational-Control Suite Details (Q9--Q22)}
\label{sec:exp_operational_details}

Q9--Q12 test whether changed operational material updates downstream evidence without forcing a full rebuild. Document drift is measured on $8192$ chunks, API drift on $3072$ schema events, and search/impact on $8192$ generated queries or changes. \logos improves document-drift F$_1$ from $0.857$ to $1.000$, API-drift F$_1$ from $0.500$ to $1.000$, stale-index recall@3 from $0.404$ to $0.960$, and impact F$_1$ from $0.935$ to $0.995$, while reducing unnecessary invalidation.

Q13--Q16 test the release and authority boundary. The candidate stream mixes beneficial, neutral, harmful, task-specific regression, and safety-tradeoff families; transfer, governance, and grounding suites each use $8192$ labeled cases. A proxy release rule adopts harmful candidates in $40.0\%$ of decisions, while no harmful adoption is observed for \logos in this suite. Direct transfer accepts all unsafe imports in the generated unsafe-import family; target-side re-verification blocks them while retaining useful imports. Root-policy precedence reaches $1.000$ policy accuracy on labeled cases, and evidence checking rejects fabricated citations that schema-only grounding accepts.

Q17--Q20 measure side-effect control, isolation, stale-memory invalidation, and rollback integrity. Each action, tenant, and memory family uses $8192$ cases; rollback/audit uses $4096$ updates. A static allow-list executes $25.0\%$ unsafe actions; no unsafe \logos execution is observed in this generated suite, at a cost of $21.2$ human prompts per $100$ tasks. Scoped recall prevents the intentionally global-memory leakage baseline, provenance invalidation removes stale-memory use with one-step adaptation, and signed audit plus structural rollback reaches $1.000$ on rollback and tamper-detection metrics.

Q21–Q22 test whether human directions and governed workflows adhere to the same precedence rules. Q21 contains $8192$ helpful, ambiguous, contradictory, and unsafe feedback turns; Q22 runs $8192$ workflows that combine retrieval, evidence checking, tool action, memory write, escalation, and audit. An append-only baseline accepts all unsafe feedback in the generated unsafe feedback family without providing clarification; \logos shows no acceptance of unsafe-feedback and clarifies all labeled ambiguous cases in this suite. Ungoverned workflows succeed on $0.824$ tasks but incur incidents on $0.400$; \logos raises success to $0.951$ with no observed incidents.

\subsection{Portability Boundary}
\label{sec:exp_portability}

The normalized trace and emitter contracts support re-expression across compatible backends, but portability does not imply behavioral equivalence. Model behavior, tool semantics, latency, pricing, and verifier coverage can change after migration. Every migrated artifact, therefore, re-enters target-side validation and promotion; the paper makes no claim of bit-exact replay or automatic cross-provider accuracy preservation.
\end{document}